\pdfoutput=1  
\documentclass{article} 
\usepackage{iclr2027_conference,times}

\iclrfinalcopy

\usepackage[utf8]{inputenc} 
\usepackage[T1]{fontenc}    
\usepackage{hyperref}       
\usepackage{url}            
\usepackage{booktabs}       
\usepackage{amsfonts}       
\usepackage{nicefrac}       
\usepackage{microtype}      
\usepackage{xcolor}         

\usepackage{subcaption}
\usepackage{epstopdf,fancyhdr,amsmath,amsthm,xspace,enumerate,paralist, mathtools,tabularx,hhline,tabu,forest,array,stmaryrd}
\usepackage{amssymb}
\usepackage{graphicx}
\usepackage{wrapfig}
\usepackage{dsfont}
\usepackage{algorithm}
\usepackage{algpseudocode}
\usepackage{multirow}
\usepackage{comment}

\usepackage[capitalise,nameinlink]{cleveref}
\theoremstyle{plain}
\newtheorem{theorem}{Theorem}[section]

\newtheorem{lemma}[theorem]{Lemma}
\newtheorem{corollary}[theorem]{Corollary}
\theoremstyle{definition}
\newtheorem{definition}[theorem]{Definition}

\theoremstyle{remark}

\theoremstyle{definition}

\crefname{section}{Sec.}{Secs.}
\Crefname{section}{Sec.}{Secs.}
\crefname{subsection}{Sec.}{Secs.}
\Crefname{subsection}{Sec.}{Secs.}
\crefname{appendix}{App.}{Apps.}
\Crefname{appendix}{App.}{Apps.}
\crefname{equation}{Eq.}{Eqs.}
\Crefname{equation}{Eq.}{Eqs.}
\crefname{figure}{Fig.}{Figs.}
\Crefname{figure}{Fig.}{Figs.}
\crefname{table}{Table}{Tables}
\Crefname{table}{Table}{Tables}
\crefname{theorem}{Thm.}{Thms.}
\Crefname{theorem}{Thm.}{Thms.}
\crefname{lemma}{Lem.}{Lems.}
\Crefname{lemma}{Lem.}{Lems.}
\crefname{proposition}{Prop.}{Props.}
\Crefname{proposition}{Prop.}{Props.}
\crefname{corollary}{Cor.}{Cors.}
\Crefname{corollary}{Cor.}{Cors.}
\crefname{definition}{Def.}{Defs.}
\Crefname{definition}{Def.}{Defs.}
\crefname{assumption}{Assumption}{Assumptions}
\Crefname{assumption}{Assumption}{Assumptions}
\crefname{remark}{Remark}{Remarks}
\Crefname{remark}{Remark}{Remarks}
\crefname{example}{Example}{Examples}
\Crefname{example}{Example}{Examples}
\crefname{algorithm}{Alg.}{Algs.}
\Crefname{algorithm}{Alg.}{Algs.}
\makeatletter
\renewcommand{\theHALG@line}{\thealgorithm.\arabic{ALG@line}}
\makeatother

\usepackage{pgfplots}
\DeclareUnicodeCharacter{2212}{−}
\usepgfplotslibrary{groupplots,dateplot}
\usetikzlibrary{patterns,shapes.arrows,fit,backgrounds}
\pgfplotsset{compat=newest}

\usetikzlibrary{positioning,arrows,shapes,shapes.arrows,arrows.meta,trees,shapes.misc,shapes.geometric,decorations.pathreplacing,automata}
\tikzset{%
  >={Latex[width=1.5mm,length=2mm]},
  vertex/.style={draw,circle,inner sep=0mm,semithick,minimum width=4mm},
  point/.style = {circle, draw, inner sep=0.04cm,fill,node contents={}},
  bidir/.style={<->,dashed, line width=0.25mm},
uvertex/.style={outer sep=0},
  dir/.style={->, line width=0.25mm},
  regime/.style={shape=rectangle,fill=black,inner sep=0pt,minimum size=3pt,draw},
  node distance=1cm,
  font=\scriptsize\sffamily
}

\definecolor{betterred}{RGB}{228,26,28}
\definecolor{betterblue}{RGB}{55,126,184}

\pgfdeclarelayer{back}
\pgfsetlayers{back,main}

\makeatletter
\pgfkeys{%
  /tikz/on layer/.code={
    \def\tikz@path@do@at@end{\endpgfonlayer\endgroup\tikz@path@do@at@end}%
    \pgfonlayer{#1}\begingroup%
  }%
}
\makeatother

\makeatletter
\newcommand{\xdashleftrightarrow}[2][]{\ext@arrow 3359\leftrightarrowfill@@{#1}{#2}}
\def\rightarrowfill@@{\arrowfill@@\relax\relbar\rightarrow}
\def\leftarrowfill@@{\arrowfill@@\leftarrow\relbar\relax}
\def\leftrightarrowfill@@{\arrowfill@@\leftarrow\relbar\rightarrow}
\def\arrowfill@@#1#2#3#4{%
  $\m@th\thickmuskip0mu\medmuskip\thickmuskip\thinmuskip\thickmuskip
   \relax#4#1
   \xleaders\hbox{$#4#2$}\hfill
   #3$%
}
\makeatother

\newcommand{\mli}[1]{\mathit{#1}}
\newcolumntype{C}[1]{>{\centering\arraybackslash}m{#1}}

\makeatletter
\def\rightarrowfill@@{\arrowfill@@\relax\relbar\rightarrow}
\def\leftarrowfill@@{\arrowfill@@\leftarrow\relbar\relax}
\def\leftrightarrowfill@@{\arrowfill@@\leftarrow\relbar\rightarrow}
\def\arrowfill@@#1#2#3#4{%
  $\m@th\thickmuskip0mu\medmuskip\thickmuskip\thinmuskip\thickmuskip
   \relax#4#1
   \xleaders\hbox{$#4#2$}\hfill
   #3$%
}
\makeatother

\newcommand{\inv}[2]{P_{#2}\left ( #1\right)}

\newcommand{\PA}{\mli{PA}}

\newcommand{\D}{\Omega}

\newcommand{\doo}{\mathrm{do}}

\newcommand{\cauvade}{\textsc{CauVaDE}\xspace}

\newcommand{\dox}{\mathrm{do}(X{=}x)}        
\newcommand{\doX}[1]{\mathrm{do}(X{=}#1)}    

\def\*#1{\boldsymbol{#1}}
\def\1#1{\mathcal{#1}}
\def\2#1{\mathscr{#1}}
\def\3#1{\mathbb{#1}}
\def\4#1{\mathds{#1}}
\def\5#1{\bar{\*#1}}

\title{A Causal Bootstrap for Diverse Image Generation under Unobserved Confounding}

\author{%
  Jingyuan Chen\\
  Department of EECS\\
  Syracuse University\\
  \texttt{jchen357@syr.edu}
  \And
  Kangrui Ruan\\
  AWS AI Lab\\
  \texttt{kr2910@columbia.edu}
  \And
  Junzhe Zhang\\
  Department of EECS\\
  Syracuse University\\
  \texttt{jzhan403@syr.edu}
}

\begin{document}

\maketitle
\lhead{Preprint}

\begin{abstract}
Deep generative models reproduce the observational distribution underlying the training data, which, in turn, inherit any spurious associations it contains. A common source is an \emph{unobserved confounder} that simultaneously affects both a treatment attribute the user wants to control/edit (e.g., age, sex) and an observed outcome (e.g., an image) expected to vary in response. Existing causal inference methods resolve the resulting ambiguity by imposing structural assumptions strong enough to single out one interventional distribution. However, in practical domains like image generation, such assumptions are rarely warranted, and the interventional distribution is generally \emph{non-identifiable}: distinct causal models can agree on the observed data yet imply different causal predictions. This paper proposes \cauvade{} (Causal Variational Deep Embedding), a novel family of causal generative models in which the unobserved confounder collapses into a discrete latent cluster with bounded support without loss of generality. We prove that this family can approximate both observational and interventional image distributions with arbitrary accuracy in a canonical latent causal diagram. Building on this model representation, we introduce a novel causal bootstrap that varies the strength of confounder entropy to generate a diverse family of causal explanations, each defining a distinct interventional image distribution. Such a procedure produces distinct interventional images while maintaining a good fit to the observed data. 
\end{abstract}

\section{Introduction}\label{sec:intro}
Deep generative models, including variational autoencoders (VAEs), generative adversarial networks (GANs), and diffusion models, synthesize images by learning the observational distribution of a training corpus~\citep{kingma2022vae,goodfellow2014gan,ho2020ddpm}. Such models generalize well when the training corpus faithfully represents the target environment, but they can also reproduce and amplify its biases. Confounding bias is one such failure: when an unobserved confounder simultaneously affects both an attribute $X$ (e.g., age, smile) and image $I$, the learned conditional $P(I \mid X{=}x)$ can differ from the interventional distribution $P(I \mid \doo(X = x))$ obtained by externally setting $X$ to $x$. Generating potential images under such interventions from confounded observational data is the task we call \emph{deconfounded image generation}.

Causal inference provides a framework for identifying interventional distributions from confounded observational data under explicit assumptions. Causal image models such as CausalGAN, CausalVAE, and DeepSCM encode relationships among semantic factors using structural causal models (SCMs)~\citep{kocaoglu2018causalgan,yang2021causalvae,pawlowski2020deepscm}; CausalVAE, for instance, assumes causal sufficiency, a linear SCM with independent Gaussian noises, and semantic labels. Such assumptions are difficult to satisfy or even specify in image domains, where the relevant factors may be unobserved and their causal relationships uncertain. Absent sufficient identifying assumptions or data, the interventional distribution is generally non-identifiable: multiple SCMs can produce the same observations yet predict different interventional responses~\citep{pearl:2k}.

\begin{figure}[t]
\hfill
        \begin{subfigure}{0.16\linewidth}\centering
		\includegraphics[width=\linewidth]{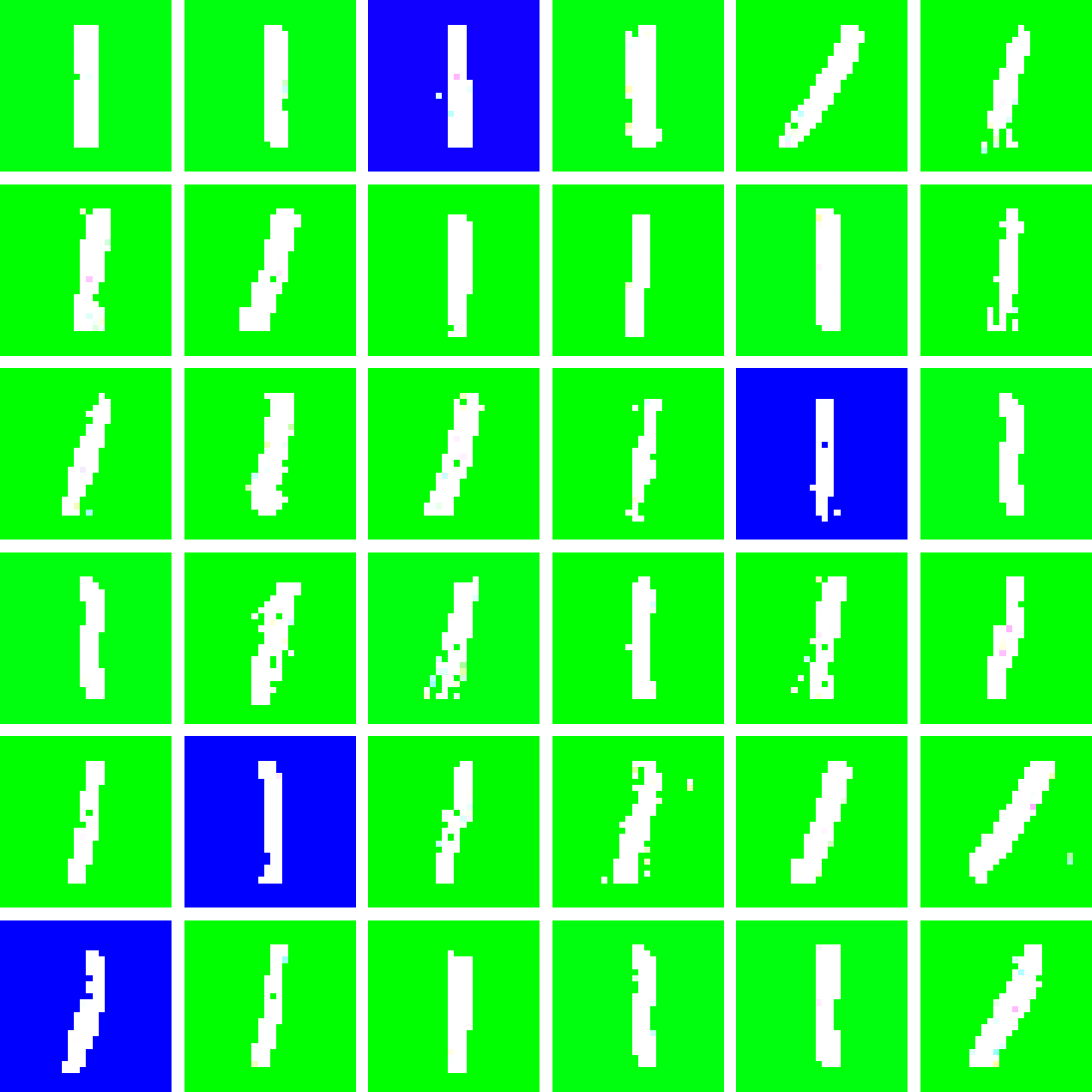}
		\caption{Observational}
		\label{fig:_cmnist_a}
	\end{subfigure}\hfill
        \begin{subfigure}{0.16\linewidth}\centering
		\includegraphics[width=\linewidth]{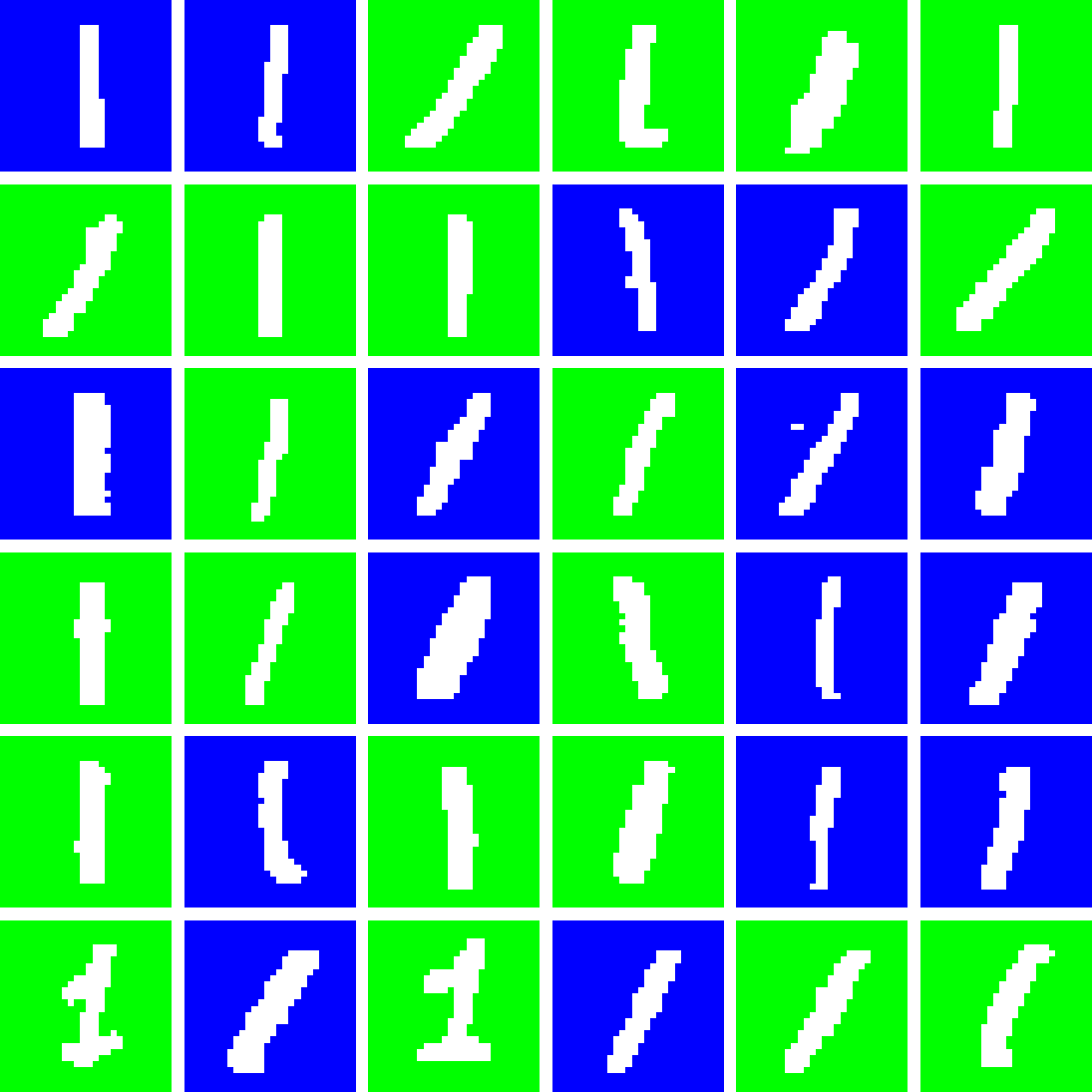}
		\caption{Interventional}
		\label{fig:_cmnist_b}
	\end{subfigure}\hfill
        \begin{subfigure}{0.16\linewidth}\centering
		\includegraphics[width=\linewidth]{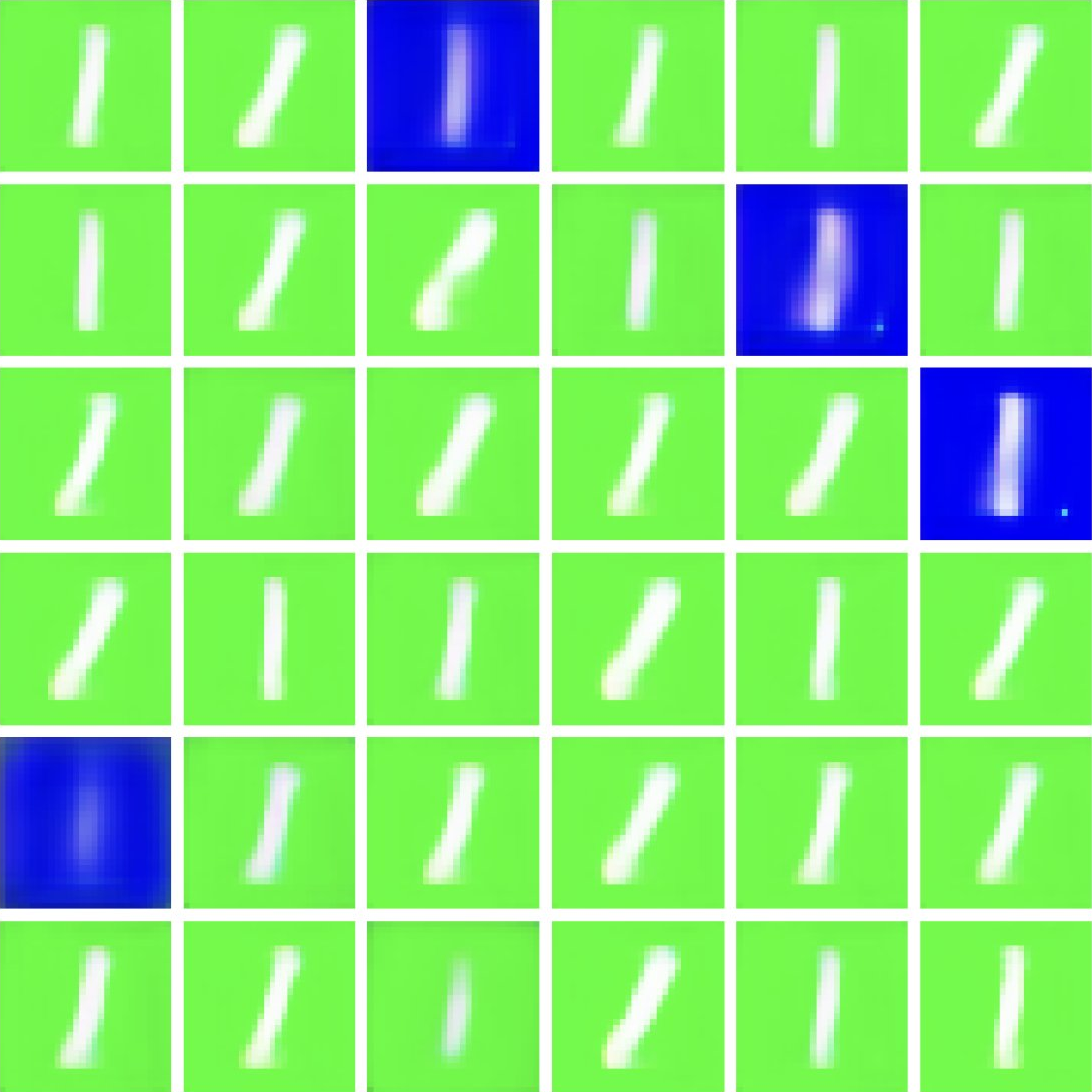}
		\caption{VAE}
		\label{fig:_cmnist_c}
	\end{subfigure}\hfill
    \begin{subfigure}{0.16\linewidth}\centering
        \includegraphics[width=\linewidth]{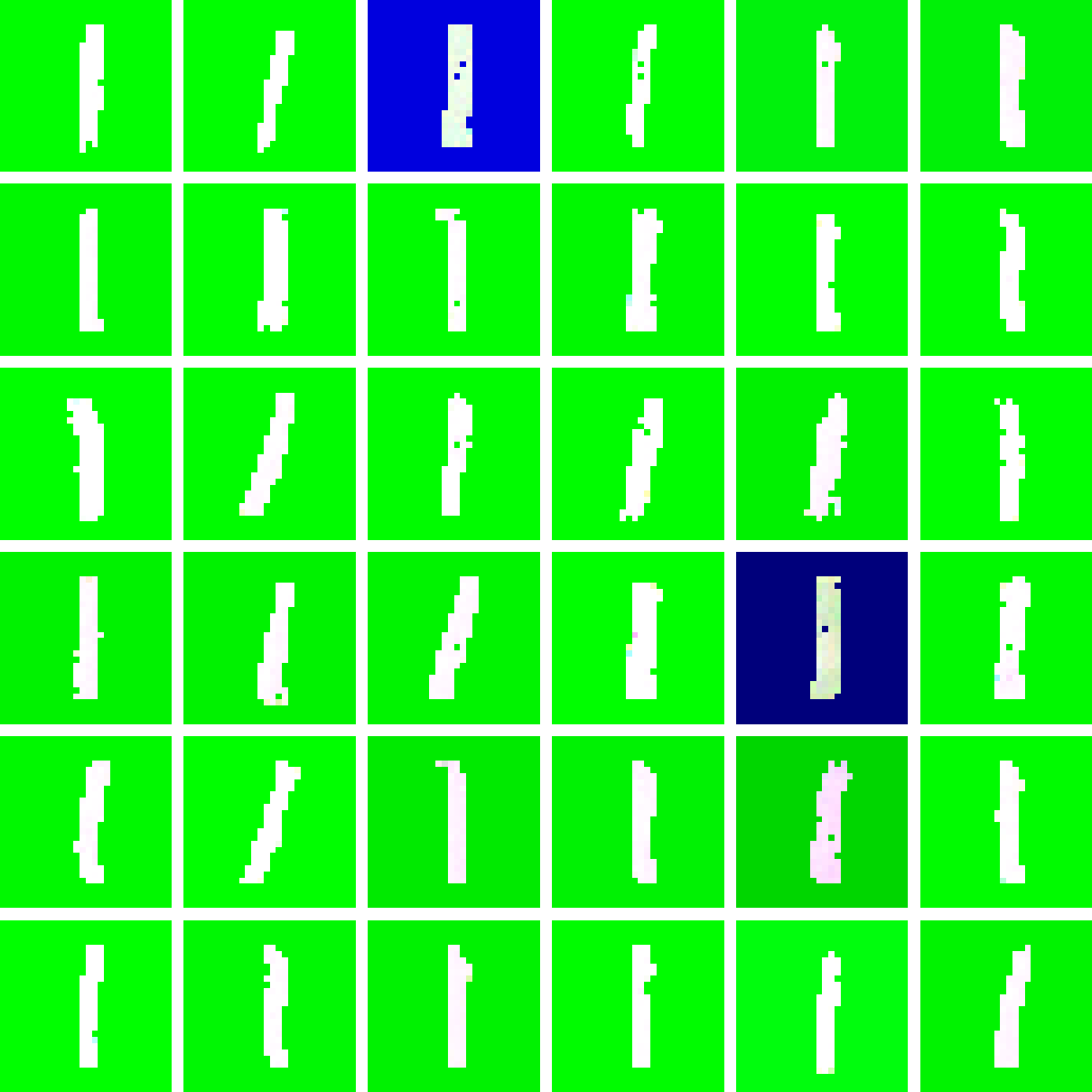}
        \caption{DDPM}\label{fig:_cmnist_d}
    \end{subfigure}\hfill
    \begin{subfigure}{0.16\linewidth}\centering
        \includegraphics[width=\linewidth]{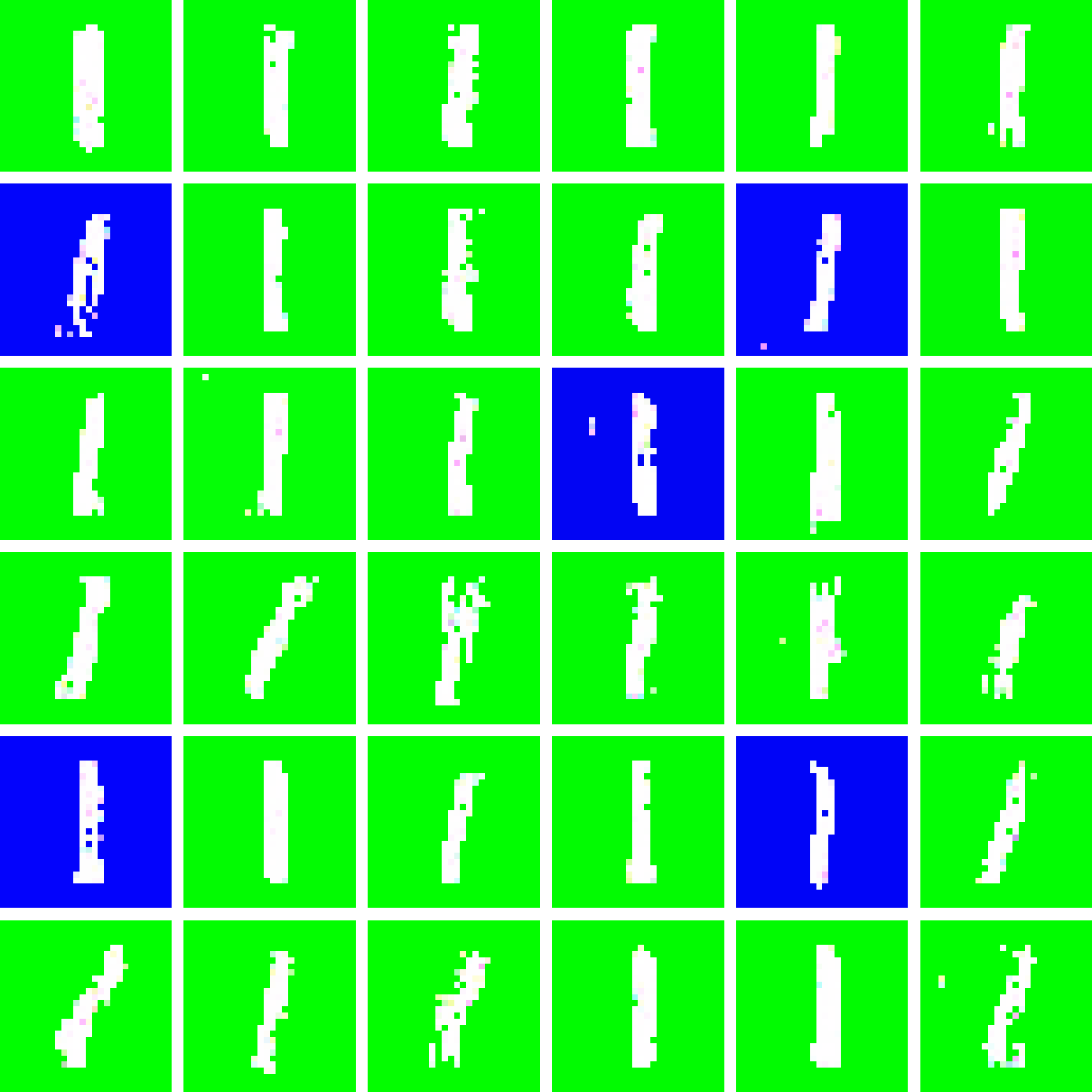}
        \caption{GAN}
        \label{fig:_cmnist_e}
    \end{subfigure}\hfill
    \begin{subfigure}{0.16\linewidth}\centering
    \includegraphics[width=\linewidth]{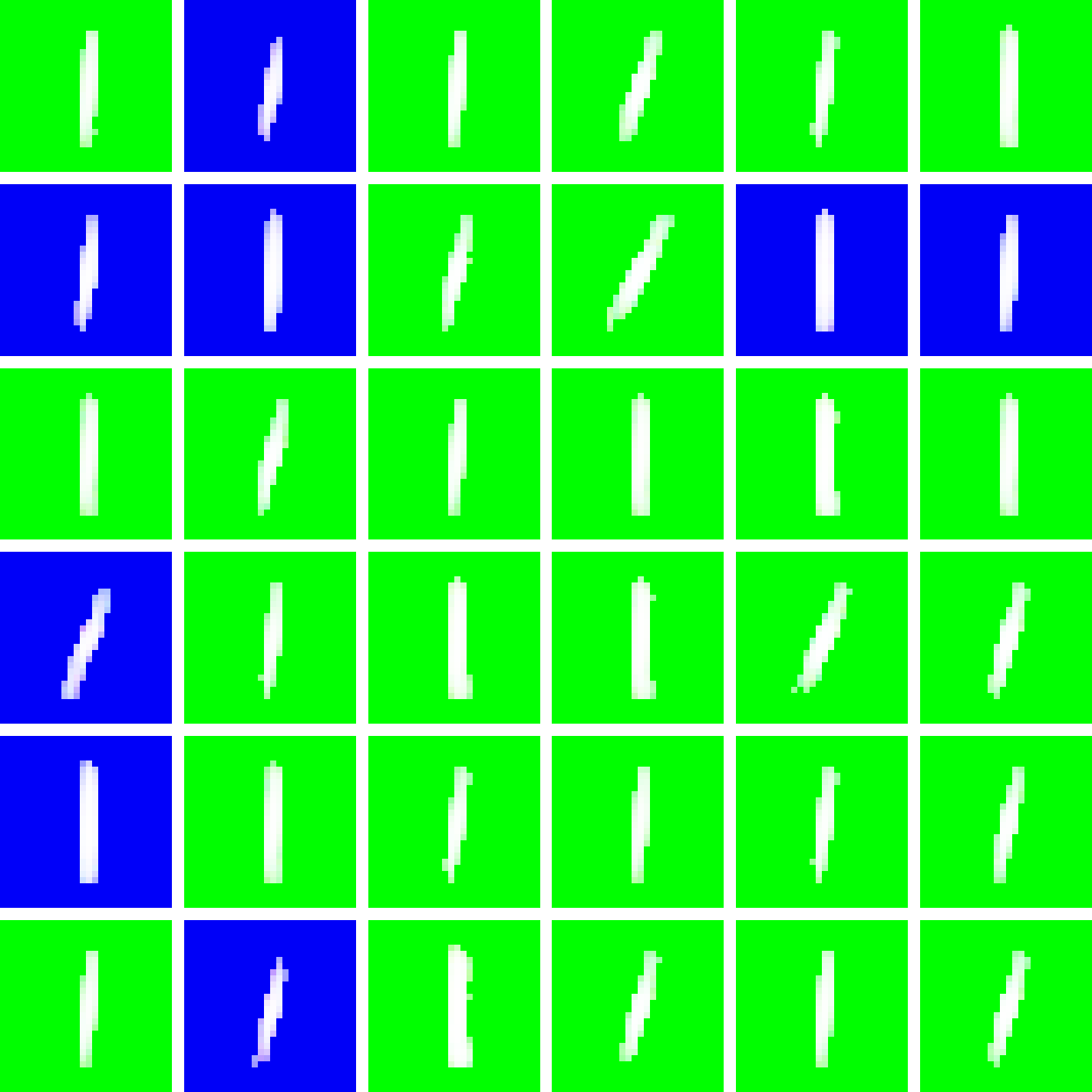}
    \caption{ANCM}\label{fig:_cmnist_f}
    \end{subfigure}
    \hfill\null
	\caption{Samples from (\subref{fig:_cmnist_a}) confounded and (\subref{fig:_cmnist_b}) unconfounded Color-MNIST; and generated by conditional (\subref{fig:_cmnist_c}) VAE, (\subref{fig:_cmnist_d}) DDPM, (\subref{fig:_cmnist_e}) GAN, and (\subref{fig:_cmnist_f}) ANCM.}
	\label{fig:_cmnist}
\end{figure}

The problem of inferring non-identifiable causal effects has been studied within the framework of \emph{partial identification}, which characterizes the range of causal predictions compatible with observations and maintained assumptions. Early work established nonparametric treatment-effect bounds and sharp bounds under instrumental variables~\citep{manski1990nonparametric,balke:pea97}. For finite observed domains, discretizing exogenous variables into finitely many response types yields canonical SCMs that preserve causal distributions and reduce sharp-bound computation to polynomial optimization~\citep{zhang2022partial}; neural approaches instead optimize causal queries under observational constraints~\citep{xia2021causal,xia2022neural}. These discrete-domain methods do not directly provide a practical solution for images: enumerating response types over image-valued variables is prohibitive, and bounding a causal statistic does not itself produce diverse image generators.

More recently, the Augmented Neural Causal Model (ANCM)~\citep{pan2024counterfactualimageediting} models images with hidden confounding while acknowledging non-identifiability, requiring fully labeled discrete semantic factors and a supplied causal diagram. Fitting an ANCM selects one causal explanation consistent with its constraints rather than enumerating the compatible ones, and repeated optimization from different random initializations yields very similar interventional predictions (\cref{fig:_cmnist_py}). Random restarts thus do not adequately explore causal ambiguity to learn diverse interventional image distributions compatible with the observed data, particularly when some image attributes are unlabeled and their causal relationships are not fully specified. The following example illustrates these challenges.

\begin{wrapfigure}[18]{r}{0.26\linewidth}
\vspace{-0.2in}
\centering
    \begin{subfigure}{\linewidth}\centering
        \includegraphics[width=\linewidth]{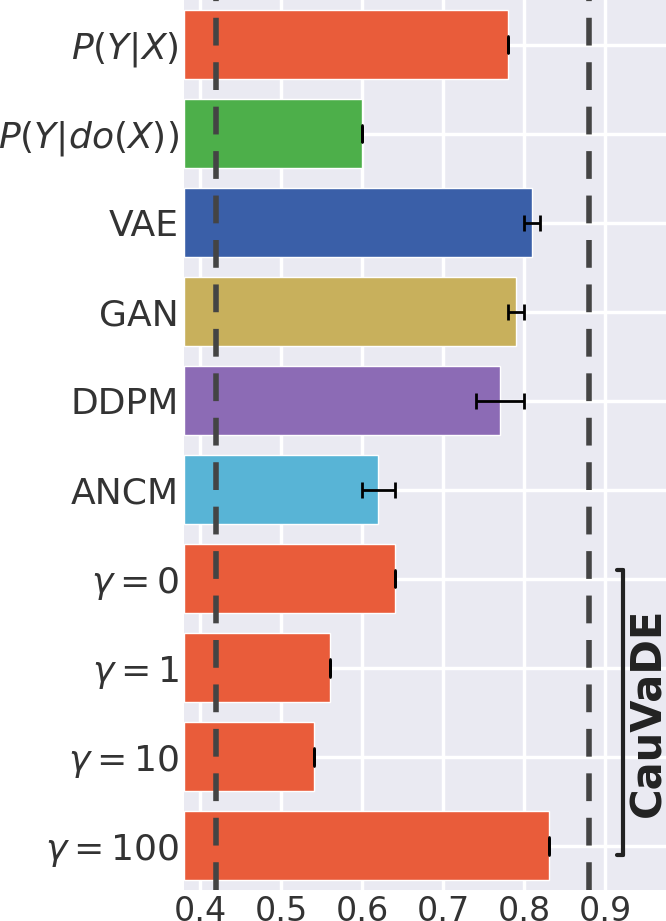}
    \end{subfigure}\hfill
    \vspace{-0.05in}
    \caption{Estimated $P(Y{=}1 \mid \doX{1})$. Dashed lines: Manski bounds from observations~\citep{manski1990nonparametric}.}
    \label{fig:_cmnist_py}
\end{wrapfigure}
\textbf{Example 1 (Confounded Color-MNIST).} Consider images with binary digit $X \in \{0,1\}$ and background color $Y \in \{0,1\}$, encoding green and blue, respectively. An unobserved binary variable $U$ influences both the digit and the color, creating a confounded training distribution (\cref{fig:_cmnist_a}). The interventional reference breaks the dependence of $X$ on $U$ while preserving the color mechanism (\cref{fig:_cmnist_b}); it need not make digit and color independent. We train conditional VAE, DDPM, GAN, and ANCM baselines on the confounded data. Their samples appear in \crefrange{fig:_cmnist_c}{fig:_cmnist_f}, and \cref{fig:_cmnist_py} compares their implied color probabilities with the reference $P(Y{=}1 \mid \doX{1})$. The conventional generators retain observational bias. ANCM is closer to the reference, but its estimates across random starts concentrate near one point inside the Manski bounds~\citep{manski1990nonparametric}. These bounds admit a wider range of causal predictions compatible with observations: fitting the observed data alone does not determine which causal prediction is correct. $\hfill \blacksquare$

This paper addresses these challenges with a causal bootstrap that generates diverse candidate interventional image distributions under unobserved confounding. We propose a canonical causal diagram for deconfounded image generation, comprising a treatment attribute, latent pre-treatment covariates, and post-treatment image attributes; only the treatment attribute requires full labels, while the remaining attributes may be partially labeled or unlabeled. Our contributions are twofold. \emph{(1)~Causal Variational Deep Embedding (\cauvade{}).} We introduce a causal generative framework whose canonical representation collapses exogenous confounding into a finite-support discrete variable, with independent exogenous noises carrying continuous variation; under standard regularity conditions, this model class approximates the observational and interventional image distributions of compatible augmented SCMs. \emph{(2)~Causal bootstrap.} We then gradually vary a cluster-posterior entropy weight to produce diverse candidates, and prove that under local conditions varying the confounder entropy yields distinct interventional image distributions while preserving observational admissibility within a fixed tolerance, though without guaranteeing exhaustive coverage or sharp bounds. We evaluate on Color-MNIST, CelebA, and MIMIC-CXR-JPG~\citep{liu2015celeba,johnson2024mimiccxr}: the first two show diverse predictions within the evaluated Manski bounds, and the last shows improved FID against an unconfounded reference. Proofs and implementation details appear in Appendices~\ref{app:related}--\ref{app:setup}.

\subsection{Preliminaries}\label{sec:prelim}
We use capital letters to denote variables ($X$), small letters for their values ($x$), and $\D_X$ for their domains. For a set $\*X$, $|\*X|$ denotes its cardinality. The probability distribution over variables $\*X$ is denoted by $P(\*X)$, and we write $P(\*x)$ as shorthand for $P(\*X = \*x)$. 
The basic semantical framework of our analysis rests on \textit{structural causal models} (SCMs)~\citep{pearl:2k,bareinboim2020pearl}. A SCM is a triple $\1M = \langle \*V, \*U, \1F \rangle$ describing how a set of observed variables $\*V = \{X, Y, Z\}$ is generated. Throughout, $X$ denotes a \emph{treatment}, $Y$ an \emph{outcome}, and $Z$ a \emph{covariate}. The set $\*U$ collects unobserved exogenous variables, drawn from an exogenous distribution $P(\*U)$. The structural equations $\1F = \{f_X, f_Y, f_Z\}$ specify how each observed variable is generated: $Z \gets f_Z(\*U)$, $X \gets f_X(Z, \*U)$, and $Y \gets f_Y(X, Z, \*U)$. Sampling $\*U \sim P(\*U)$ and applying these equations induces the \emph{observational distribution} $P(X, Y, Z) \triangleq P(X, Y, Z; \1M)$.

Each SCM $\1M$ is also associated with a causal diagram $\1G$ (e.g., \Cref{fig:ascm}), where nodes represent endogenous $\*V$, and arrows represent the arguments $\PA_V, U_V$ of each structural function $f_{V}$. Exogenous variables $\*U$ are not explicitly shown; a bi-directed arrow $V_i \leftrightarrow V_j$ indicates an unobserved confounder (UC), $U_{i, j} \in \*U$, affecting $ V_i$ and $ V_j$. An \emph{intervention} $\dox$ models externally setting $X$ to a constant $x$, overriding whatever value $f_X$ would have produced. Formally, $\dox$ replaces $f_X$ in $\1M$ with the constant assignment $X \gets x$, yielding a submodel $\1M_x$. Sampling $\*U \sim P(\*U)$ in $\1M_x$ induces the \emph{interventional distribution} $P(Y \mid \doo(x)) \triangleq P(Y_x; \1M)$, where $Y_x(\*u) \triangleq Y_{\1M_x}(\*u)$ is the potential outcome of $Y$ when $X$ is set to $x$. For convenience, we also use the subscript notation $\inv{Y}{x} = P(Y \mid \doo(x))$ interchangeably for interventional distributions.

\section{Canonical Causal Models for Image Generation}\label{sec:model}
The augmented structural causal model (ASCM) of \citet{pan2024counterfactualimageediting} provides a framework for causal image generation with a pre-specified causal diagram and fully labeled, discrete semantic factors. In practical domains, detailed causal relationships may be unavailable, labels incomplete, and attributes continuous. We address these challenges in two steps: a canonical diagram obtained by grouping semantic variables relative to the treatment, followed by a canonical representation that places discrete structure on exogenous confounding.

\begin{wrapfigure}[13]{r}{0.45\linewidth}
\vspace{-0.35in}
\centering
\hfill
\begin{subfigure}{0.48\linewidth}\centering
  \begin{tikzpicture}
      \def\innerr{2.7}
      \node[vertex] (X) at (0, 0) {X};
      \node[vertex] (Z) at (1, 1) {Z};
      \node[vertex] (Y) at (2, 0) {Y};
      \node[vertex] (I) at (1,-1) {I};
      \draw[dir] (Z) to (Y);
      \draw[dir] (X) -- (Y);
      \draw[dir] (X) to (I);
      \draw[dir] (Z) to (I);
      \draw[dir] (Y) to (I);
      \draw[dir] (Z) to (X);
      \draw[bidir, dashed] (Z) to [bend left = 45] (Y);
      \draw[bidir, dashed] (Z) to [bend right = 45] (X);
      \draw[bidir, dashed] (X) to [bend right = 20] (Y);
  \begin{pgfonlayer}{back}
      \node[circle,fill=betterred!65,draw=none,minimum size=2*\innerr mm] at (X) {};
      \node[circle,fill=betterblue!65,draw=none,minimum size=2*\innerr mm] at (Y) {};
      \node[draw=gray, dashed, rounded corners, fit=(X)(Y)(Z), inner sep=6pt] {};
  \end{pgfonlayer}
  \end{tikzpicture}
  \caption{$\1M$}
  \label{fig:ascm_obs}
  \end{subfigure}\hfill
\begin{subfigure}{0.48\linewidth}\centering
\begin{tikzpicture}
      \def\innerr{2.7}
      \node[vertex] (X) at (0, 0) {X};
      \node[vertex] (Z) at (1, 1) {Z};
      \node[vertex] (Y) at (2, 0) {Y};
      \node[vertex] (I) at (1,-1) {I};
      \draw[dir] (Z) to (Y);
      \draw[dir] (X) -- (Y);
      \draw[dir] (X) to (I);
      \draw[dir] (Z) to (I);
      \draw[dir] (Y) to (I);
      \draw[bidir, dashed] (Z) to [bend left = 45] (Y);
  \begin{pgfonlayer}{back}
      \node[circle,fill=betterred!65,draw=none,minimum size=2*\innerr mm] at (X) {};
      \node[circle,fill=betterblue!65,draw=none,minimum size=2*\innerr mm] at (Y) {};
      \node[draw=gray, dashed, rounded corners, fit=(X)(Y)(Z), inner sep=6pt] {};
  \end{pgfonlayer}
  \end{tikzpicture}
    \caption{$\1M_{x}$}
   \label{fig:ascm_exp}
  \end{subfigure}\hfill \null
  \caption{(a) Canonical diagram with treatment $X$, pre-treatment covariates $Z$, post-treatment attributes $Y$, and image $I$. (b) Submodel induced by an intervention $\dox$.}
  \label{fig:ascm}
\end{wrapfigure}
Rather than specify every causal relationship among individual image attributes, we organize the semantic variables of an acyclic ASCM around the attribute $X$ to be manipulated, grouping its descendants into a post-treatment block $Y$ and all other semantic variables into a pre-treatment block $Z$. Here, pre-treatment means unaffected by interventions on $X$, so $Z$ may include both causes of $X$ and other non-descendants. Both blocks may be multivariate and latent. For example, when $X$ is a digit, $Z$ can encode the writer's identity and $Y$ a background color generated in response to the digit. For any acyclic causal semantic model with treatment $X$, composing causal mechanisms within each block yields $Z\gets f_Z(\*U)$, $X\gets f_X(Z,\*U)$, and $Y\gets f_Y(X,Z,\*U)$, where $\*U$ collects the original exogenous variables. This reduction gives the \emph{canonical diagram} in \cref{fig:ascm_obs} while preserving the observational image distribution and interventions on $X$ (\cref{app:diagram-reduction}).

Following the ASCM framework of \citet{pan2024counterfactualimageediting}, we augment the grouped causal semantic model with an image variable $I$ and mechanism $I\gets f_I(X,Y,Z,\1E_I)$, where the rendering noise $\1E_I$ is independent of $\*U$ and captures visual features not encoded in the causal semantics. The covariates $Z$ are endogenous causal attributes, distinct from the exogenous variables $\*U$ that can confound the blocks. Only $X$ requires full labels; the components of $Y$ and $Z$ may be partially labeled or unlabeled. For the representation result below, we assume:
\setdefaultleftmargin{1.2em}{}{}{}{}{}
\begin{compactitem}
\item \textbf{(A1) Compact observed domain.} The image domain $\D_I$ and treatment domain $\D_X$ are compact.
\item \textbf{(A2) Continuous counterfactual outcomes.} For every exogenous realization $u$, the map $x\mapsto Y_x(u)=f_Y(x,f_Z(u),u)$ is continuous on $\D_X$.
\item \textbf{(A3) Continuous mixing function.} $f_I$ is continuous in its endogenous inputs $(X,Y,Z)$.
\end{compactitem}
(A1) models finite-resolution images with bounded pixel values and treatments in closed, bounded ranges; (A2) requires outcomes to vary continuously with treatment for each fixed exogenous realization; (A3) ensures that small changes in semantic inputs produce small changes in images for each fixed rendering-noise realization. Stronger continuity matters for partial identification, as causal bounds on continuous domains can be vacuous without Lipschitz regularization \citep{tan2024consistency}.

Sampling the exogenous variables and propagating through this model gives the observational image distribution $P(I)\triangleq P(I;\1M)$. Applying $\dox$ replaces the treatment mechanism by $X\gets x$, giving the submodel in \cref{fig:ascm_exp} and the interventional image distribution $P_x(I) \triangleq P(I_x; \1M) = P\bigl(f_I(x, Y_x, Z_x, \1E_I)\bigr)$, where $Z_x=Z$. Under confounding, this distribution can differ from the observational $P(I)$ and conditional $P(I\mid X{=}x)$. Thus, even with a common diagram, fitting the observed pairs $(X,I)$ need not determine the intervention response. Our target is the set of candidate ASCMs compatible with $P(X,I)$ and the canonical diagram, each inducing a possibly distinct $P_x(I)$.

\subsection{Canonical Representation with Discrete Confounding}\label{sec:cascm}
The second challenge is to represent this model class when causal attributes are incompletely labeled and need not have finite discrete domains. We leave unlabeled $Y$ and $Z$ latent, without claiming to identify their values from images alone. Their domains and causal mechanisms can remain continuous. The unknown exogenous structure still makes the compatible class infinite-dimensional: $\*U$ can have unknown dimension, support, and joint distribution. We therefore introduce a representation with a finite-support shared confounder and independent noises carrying continuous variation.

\begin{definition}[Canonical ASCM]\label{def:cascm}
A \emph{canonical ASCM} (CASCM) is an ASCM satisfying (A1)--(A3) in which: (i)~the unobserved confounder is a single discrete variable $U \in \{1, \dots, d\}$ drawn from $P(U)$; (ii)~the semantic variables are generated by continuous structural functions $Z \gets f_Z(U, \1E_Z)$, $X \gets f_X(Z, U, \1E_X)$, $Y \gets f_Y(X, Z, U, \1E_Y)$ with noises $\1E_Z, \1E_X, \1E_Y, \1E_I, U$ mutually independent; (iii)~images are generated by the mixing function $I \gets f_I(X, Y, Z, \1E_I)$ satisfying (A3).
\end{definition}
The canonical diagram groups causal mechanisms; the canonical representation additionally replaces their exogenous domain by a discrete variable $U$. The following theorem justifies the latter step by approximation for continuous counterfactual outcomes and mixing functions.
\begin{theorem}[Augmented Causal Approximation]\label{thm:approx}
Fix the causal diagram in \cref{fig:ascm_obs} and assume (A1)--(A3), with the domain and measurability conventions above. Let $\1M$ be a compatible ASCM with exogenous variables $\*U$ taking values in a Borel subset of $\mathbb R^m$ for some finite $m$. All structural assignments, including interventions in $\D_X$, are assumed well-defined in these domains. For every $\epsilon>0$, there exists a CASCM $\1N$ such that:
\begin{align}
    &W_1\bigl(P(I, X; \1M),\, P(I, X; \1N)\bigr) < \epsilon, &&\int_{\D_X} W_1\bigl(P_x(I; \1M),\, P_x(I; \1N)\bigr)\, dx < \epsilon, \label{eq:approx-int}
\end{align}
where $W_1$ is the $1$-Wasserstein distance and $dx$ denotes Lebesgue measure for continuous treatments or counting measure for categorical treatments. That is, the class of CASCMs is \emph{dense} in this regular class of ASCMs compatible with the canonical diagram \cref{fig:ascm_obs} under joint $W_1$ approximation of the observational and interventional distributions (\cref{eq:approx-int}).
\end{theorem}
The finite-domain causal representation result of \citet{zhang2022partial} motivates this construction, but does not itself establish the image-domain guarantee with latent causal models. Our proof (\cref{app:proof}) quantizes the joint distribution of natural treatment assignments, covariates, and entire response functions, preserving their dependence, and couples the rendering noise between models. Continuity of the responses and mixing function, together with compact image and treatment domains, yields observational and integrated interventional convergence. 

    \section{Causal Variational Deep Embedding}\label{sec:cauvade}
    Under its continuity and exogenous-domain assumptions, \cref{thm:approx} motivates CASCMs as a representation of observational and interventional distributions, but this representation alone does not provide a tractable procedure for learning and inference. We therefore introduce Causal Variational Deep Embedding (\cauvade{}), a tractable parametric instantiation based on Variational Deep Embedding (VaDE)~\citep{jiang2017vade} that incorporates the canonical causal relationships in \cref{fig:ascm_obs}. We then develop a causal bootstrap to explore diverse causal explanations compatible with the observed data and their associated interventional predictions.
    
    \begin{wrapfigure}[15]{r}{0.45\linewidth}
    \vspace{-0.3in}
    \centering
    \hfill
    \begin{subfigure}{0.6\linewidth}\centering
    \begin{tikzpicture}[node distance=0.85cm]
        \def\innerr{2.7}
        \node[vertex] (C) at (0, 3) {$C$};
        \node[vertex] (Z) at (0, 2) {$\mathbf{Z}$};
        \node[vertex] (X) at (-1.3, 1.5) {X};
        \node[vertex] (Y) at (1.3, 1.5) {Y};
        \node[vertex] (I) at (0, 0.5) {I};
        \draw[dir] (C) to (Z);
        \draw[dir] (C) to (X);
        \draw[dir] (Z) to (X);
        \draw[dir] (C) to (Y);
        \draw[dir] (Z) to (Y);
        \draw[dir] (X) to (Y);
        \draw[dir] (C) to [bend right=30] (I);
        \draw[dir] (Z) to (I);
        \draw[dir] (X) to (I);
        \draw[dir] (Y) to (I);
    \begin{pgfonlayer}{back}
        \node[circle,fill=betterred!65,draw=none,minimum size=2*\innerr mm] at (X) {};
        \node[circle,fill=betterblue!65,draw=none,minimum size=2*\innerr mm] at (Y) {};
        \node[draw=gray, dashed, rounded corners, fit=(C)(Z)(X)(Y), inner sep=4pt] {};
    \end{pgfonlayer}
    \end{tikzpicture}
    \caption{Decoder}
    \label{fig:cauvade_gen}
    \end{subfigure}\hfill
    \begin{subfigure}{0.3\linewidth}\centering
    \begin{tikzpicture}[node distance=0.85cm]
        \def\innerr{2.7}
        \node[vertex] (C) at (0, 3) {$C$};
        \node[vertex] (Z) at (0, 1.75) {$\mathbf{Z}$};
        \node[vertex] (I) at (0, 0.5) {I};
        \draw[dir, dashed] (I) to (Z);
        \draw[dir, dashed] (Z) to (C);
    \begin{pgfonlayer}{back}
        \node[draw=gray, dashed, rounded corners, fit=(C)(Z), inner sep=4pt] {};
    \end{pgfonlayer}
    \end{tikzpicture}
    \caption{Encoder}
    \label{fig:cauvade_inf}
    \end{subfigure}\hfill\null
    \caption{\cauvade{}: (a) decoder of \cref{eq:cauvade-gen}, where the discrete $C$ instantiates the canonical confounder of \cref{def:cascm}; (b) encoder network $q_\phi(\mathbf{Z}, C \mid I) = q_\phi(\mathbf{Z} \mid I)\,q_\phi(C \mid \mathbf{Z})$.}
    \label{fig:cauvade_pgm}
    \end{wrapfigure}
    VaDE models the latent space of a variational autoencoder as a Gaussian mixture indexed by a discrete cluster $C \in \{1, \dots, K\}$. In \cauvade{}, we identify $C$ with the canonical confounder $U$ in \cref{def:cascm}, using its finite mixture structure to parameterize the confounder's influence on the generative mechanisms. The continuous latent collects the remaining exogenous structure of \cref{def:cascm} into a single bundle,
    $\mathbf{Z} \;=\; (Z,\, \1E_Z,\, \1E_X,\, \1E_Y,\, \1E_I)$,
    where $Z$ is the pre-treatment attribute and $\1E_Z, \1E_X, \1E_Y, \1E_I$ are the independent noises driving the four CASCM mechanisms; structural functions $f_Z, f_X, f_Y, f_I$ are then realized by neural decoders. \cauvade{} specifies the generative process described as follows:
    \begin{equation}\label{eq:cauvade-gen}
    \begin{aligned}
        C &\sim P(C), \qquad
        \mathbf{Z} \mid C \sim \1N(\mu_c, \sigma_c^2 \mathbf{I}), \qquad
        X \mid \mathbf{Z}, C \sim p_\theta(X \mid \mathbf{Z}, C),\\
        Y \mid X, \mathbf{Z}, C &\sim p_\theta(Y \mid X, \mathbf{Z}, C), \qquad
        I \mid X, Y, \mathbf{Z}, C \sim p_\theta(I \mid X, Y, \mathbf{Z}, C),
    \end{aligned}
    \end{equation}
    where $P(C)$ is a learned categorical prior, $(\mu_c, \sigma_c^2)$ are cluster-specific Gaussian parameters, and the conditionals over $X$, $Y$, and $I$ are amortized by neural decoders $\theta$. \cref{fig:cauvade_pgm}(a) shows the corresponding graph: $C$ and $\mathbf{Z}$ jointly play the role of the unobserved exogenous structure of the CASCM, $X$ and $Y$ are the user-controlled and post-treatment semantic attributes, and $I$ is the rendered image. In \cauvade{}, $Z$ and the noise variables have no separate latent representation; they are folded into $\mathbf{Z}$ and, along with $Y$, recovered as unlabeled semantic factors by the decoder.
    
    Under intervention, the \cauvade{} submodel induced by $\dox$ proceeds by drawing unobserved exogenous variable $(C, \mathbf{Z})$ from their joint prior $p(C)\, p_\theta(\mathbf{Z} \mid C)$, fixing $X = x$, and propagating through the decoders for $Y$ and $I$. This yields the parametric interventional image distribution:
    \begin{equation}\label{eq:Px_param}
    P_x^\theta(I) \;=\; \sum_{c=1}^{K} p(c) \int p_\theta(\mathbf{z} \mid c) \int p_\theta(y \mid x, \mathbf{z}, c)\, p_\theta(I \mid x, y, \mathbf{z}, c)\, dy\, d\mathbf{z}.
    \end{equation}
    With the generative and interventional distributions specified, we turn to fitting \cauvade{} from a confounded dataset $\1D = \{(x_i, I_i)\}_{i=1}^N$ in which only $X$ is labeled. Following VaDE, we approximate the interventional posterior $p_\theta(\mathbf{Z}, C \mid I,\mathrm{do}(X=x))$ by a factorized variational family $q_\phi(\mathbf{z}, c \mid I) = q_\phi(\mathbf{z} \mid I)\, q_\phi(c \mid \mathbf{z})$ (\cref{fig:cauvade_pgm}(b)). The encoder uses only $I$ as an architectural choice; interventional generation uses the latent prior independently of this choice. A variational lower bound on $\log P_x^\theta(I)$, evaluated at the observed $x=X$ (\cref{app:elbo}), yields the classic ELBO objective:
    \begin{equation}\label{eq:elbo}
    \1L_{\mathrm{ELBO}}(\theta, \phi) \;=\; \underbrace{\3E_{q_\phi}\bigl[\log p_\theta(I \mid X, \mathbf{Z}, C)\bigr]}_{\1L_{\mathrm{rec}}} \;-\; \underbrace{D_{\mathrm{KL}}\bigl(q_\phi(\mathbf{Z}, C \mid I)\,\|\,p_\theta(\mathbf{Z} \mid C)\,P(C)\bigr)}_{\1L_{\mathrm{KL}}}.
    \end{equation}
    To incorporate labeled latent components, we anchor treatment $X$ in the latent representation with a discriminative head, yielding the \emph{treatment-consistency term}:\footnote{Labeled subsets of $Y$ or components of $\mathbf{Z}$, when available, are readily incorporated by adding analogous discriminative heads with no change to the rest of the objective.} $\1L_{X}(\phi, \psi) \;=\; \3E_{q_\phi}\bigl[-\log p_\psi(X \mid \mathbf{Z}, C)\bigr]$, where $p_\psi$ is an auxiliary classifier on top of the encoder. The term encourages $(C, \mathbf{Z})$ to predict $X$, but does not by itself force the image decoder to use $X$, or identify the generative treatment mechanism $p_\theta(X\mid\mathbf{Z},C)$ with the auxiliary head $p_\psi$. Combining reconstruction, latent regularization, and treatment supervision with fixed weights $\alpha,\beta,\lambda>0$ gives the \emph{ELBO surrogate}
    \begin{equation}\label{eq:D}
    D(\theta, \phi, \psi) \;\coloneqq\; -\,\alpha\, \1L_{\mathrm{rec}} \;+\; \beta\, \1L_{\mathrm{KL}} \;+\; \lambda\, \1L_{X},
    \end{equation}
    where the reconstruction and treatment terms measure fitness with the observed images and labeled components, while the KL term regularizes the marginalized latent representation. Optimizing \cref{eq:D} adapts ANCM's reconstruction and semantic supervision~\citep{pan2024counterfactualimageediting} to partially labeled data. In our Color-MNIST experiments (\cref{fig:_cmnist_py}), repeated ANCM training from different random starts yields similar interventional predictions. This empirical concentration motivates a more directed exploration of candidate models; it does not identify the underlying optimization mechanism, which calls for an alternative training and inference approach.
    
    \subsection{A Causal Bootstrap for Diverse Interventional Image Generation}\label{sec:elbo-aug}
    Our goal is to obtain a set of diverse predictions on interventional distributions among causal models compatible with the observed data. To achieve this goal, we introduce a training and inference procedure that varies the fitted model by regularizing its cluster-posterior entropy. 
    
    To make this argument more concrete, we introduce a cluster-entropy regularizer that controls uncertainty in the inferred canonical confounder. Formally, define the marginal posterior distribution over the cluster confounder by
    $q_\phi(c\mid I)=\int q_\phi(c\mid\mathbf z)q_\phi(\mathbf z\mid I)\,d\mathbf z$ and its mean entropy by $H(\phi) \coloneqq -\,\3E_{P(I)}\sum_{c=1}^K q_\phi(c\mid I)\log q_\phi(c\mid I)\in[0,\log K]$, which measures uncertainty in the inferred cluster assignment for each image. Let $\eta=(\theta,\phi,\psi)$ collect the model and inference parameters. Adding this regularizer to the surrogate objective in \Cref{eq:D} gives an augmented ELBO regularized by the entropy of the categorical confounder, defined as follows:
    \begin{equation}\label{eq:objective}
    \1L(\eta;\gamma)
    =-\alpha\1L_{\mathrm{rec}}+\beta\1L_{\mathrm{KL}}+\lambda\1L_X-\gamma H
    =D(\eta)-\gamma H(\phi),\qquad \eta=(\theta,\phi,\psi),
    \end{equation}
    with coefficient $\gamma\ge0$; the negative signs reward reconstruction log-likelihood and posterior entropy. At $\gamma=0$, this reduces to the baseline objective $D$. To generate interventional image distributions, we fix hyperparameters $(\alpha,\beta,\lambda)$ and fit the model as the entropy weight $\gamma$ varies gradually. This yields a sequence of fitted parameters $(\theta, \phi)$; each pair defines an interventional image distribution $P_x^{\theta}(I)$.
    
    Our analysis provides intuition for how this training procedure could generate a diverse family of interventional images while being consistent with the observational data. Specifically, we consider three separate questions: how much varying the entropy weight increases the ELBO surrogate, whether it changes interventional predictions, and when the resulting models remain observationally admissible. We start at an unregularized baseline $\eta_0$ and impose two local assumptions. (A4) specifies the local smoothness and baseline curvature needed to follow fitted parameters as $\gamma$ varies.
    \begin{compactitem}
    \item \textbf{(A4) Stable baseline.} $D,H$ are twice differentiable near $\eta_0$, with $\nabla_\eta D(\eta_0)=0$, $\nabla_\eta^2D(\eta_0)\succ0$.
    \end{compactitem}
    To describe the branch's initial direction and entropy response, set
    \begin{equation}\label{eq:sweep-direction}
    A=\nabla_\eta^2D(\eta_0),\qquad b=\nabla_\eta H(\eta_0),\qquad
    v=A^{-1}b,\qquad \kappa_0=b^\top A^{-1}b.
    \end{equation}
    Here $v$ is the initial parameter direction induced by entropy regularization: the local curvature $A$ converts the entropy gradient $b$ into a displacement of the fitted optimum. The quantity $\kappa_0$ measures the corresponding entropy response. (A4) supplies a stable local branch, but does not require $b\ne0$; if $b=0$, its initial direction vanishes. Even when parameters move, the generated distribution need not change. Write $\mathsf P_x(\eta)=P_x^\theta(I)$ and let $\|\cdot\|_{\mathrm{KR}}$ denote the Kantorovich--Rubinstein norm on the bounded image domain, with $\|P-Q\|_{\mathrm{KR}}=W_1(P,Q)$. (A5) requires a smooth, nondegenerate response of the interventional image distribution along the induced parameter direction.
    \begin{compactitem}
    \item \textbf{(A5) Nondegenerate interventional response.} The map $\mathsf P_x$ is differentiable at $\eta_0$ in the KR norm, with $\dot P_x=\mathrm d\mathsf P_x(\eta_0)[A^{-1}b]\ne0$ and $\rho_x=\|\dot P_x\|_{\mathrm{KR}}>0$.
    \end{compactitem}
    
    \begin{wrapfigure}[16]{r}{0.46\textwidth}
    \centering
    \vspace{-0.22in}
\begin{subfigure}[t]{0.49\linewidth}
\centering
\captionsetup{justification=centering}
\begin{tikzpicture}
\begin{axis}[
    width=0.78\linewidth, height=1.02in, scale only axis,
    xmin=-1.05, xmax=1.05, ymin=-1.05, ymax=1.05,
    axis lines=middle, axis line style={black!45},
    ticks=none, xlabel={$\theta$}, ylabel={$\phi$},
    label style={font=\scriptsize},
    xlabel style={at={(axis description cs:1,0.5)},anchor=west},
    ylabel style={at={(axis description cs:0.5,1)},anchor=south},
    clip=true,
]
\foreach \radius in {0.22,0.45,0.68,0.90,1.15} {
    \addplot[black!25, thin, domain=0:360, samples=81, variable=\t]
        ({\radius*(cos(\t)+sin(\t))},{\radius*sin(\t)});
}
\addplot[black!65, thick, -{Stealth[length=1.4mm]}]
    coordinates {(0,0) (0.79,0.79)};
\addplot[only marks, mark=*, mark size=1.7pt, black!65]
    coordinates {(0,0)};
\addplot[only marks, mark=*, mark size=1.7pt, betterblue]
    coordinates {(0.3,0.3)};
\addplot[only marks, mark=*, mark size=1.7pt, betterred]
    coordinates {(0.6,0.6)};
\node[font=\scriptsize, anchor=north west, inner sep=2pt]
    at (axis cs:0,0) {$\eta_0$};
\node[font=\scriptsize, anchor=north west, inner sep=2pt]
    at (axis cs:0.6,0.6) {$\eta_\gamma$};
\node[font=\scriptsize, anchor=south east, inner sep=1pt,
    fill=white, fill opacity=0.9, text opacity=1]
    at (axis cs:0.34,0.48) {$v=A^{-1}b$};
\node[font=\scriptsize, anchor=south west, inner sep=1pt, align=left,
    fill=white, fill opacity=0.9, text opacity=1]
    at (rel axis cs:0.03,0.03) {$D,H\in C^2$\\$A\succ0$};
\end{axis}
\end{tikzpicture}
\caption{A4: smooth\\local geometry.}
\end{subfigure}\hfill
\begin{subfigure}[t]{0.49\linewidth}
\centering
\captionsetup{justification=centering}
\begin{tikzpicture}
\begin{axis}[
    width=0.78\linewidth, height=1.02in, scale only axis,
    xmin=0, xmax=1, ymin=0, ymax=2.4,
    axis lines=left, axis line style={black!45},
    xtick={0,1}, ytick=\empty,
    tick label style={font=\scriptsize},
    xlabel={$i$}, ylabel={$p_x^{\theta_\gamma}(i)$},
    label style={font=\scriptsize},
    xlabel style={at={(axis description cs:1,0)},anchor=north west},
    ylabel style={rotate=-90,at={(axis description cs:0,1)},anchor= west},
    domain=0:1, samples=81,
]
\addplot[black!65, densely dashed, thick] {6*x*(1-x)};
\addplot[betterblue, dashdotted, thick] {6*x*(1-x)*(1+0.3*(2*x-1))};
\addplot[betterred, solid, thick] {6*x*(1-x)*(1+0.6*(2*x-1))};
\draw[black!65, thick, -{Stealth[length=1.4mm]}]
    (axis cs:0.3,1.86) -- (axis cs:0.78,1.86)
    node[midway, above, font=\scriptsize] {$\dot P_x\ne0$};
\end{axis}
\end{tikzpicture}
\caption{A5: nonzero\\distributional response.}
\end{subfigure}
\par\smallskip
{\scriptsize
\tikz[baseline=-0.5ex]{\draw[black!65,densely dashed,thick] (0,0)--(0.25,0);}
$\gamma=0$\; (baseline)\quad
\tikz[baseline=-0.5ex]{\draw[betterblue,dashdotted,thick] (0,0)--(0.25,0);}
$\gamma_1$\quad
\tikz[baseline=-0.5ex]{\draw[betterred,thick] (0,0)--(0.25,0);}
$\gamma_2>\gamma_1$}
    \caption{(A4)--(A5): a stationary baseline with positive curvature gives a smooth local branch (left), along which the interventional distribution has a nonzero derivative (right). Colors match parameters to distributions.}
    \label{fig:assumptions-sweep}
    \end{wrapfigure}
    \Cref{fig:assumptions-sweep}(a) illustrates the local geometry required by (A4): the gradients and Hessians of both $D$ and $H$ exist and vary continuously near $\eta_0$. The center of the loss contours marks the stationary baseline, $\nabla D(\eta_0)=0$; their positive curvature represents $A\succ0$. Together, these conditions ensure that a small change in $\gamma$ admits a nearby local optimum of $D-\gamma H$, illustrated by the colored points. Each point locally minimizes its own regularized objective. Panel (b) illustrates the requirement in (A5): along the induced parameter direction, the interventional distribution is differentiable and has a nonzero derivative. The matching density curves depict probability shifting toward larger values as $\gamma$ increases.
    
    We first quantify the cost of the observation surrogate along the local branch of the fitted model parameters as the $\gamma$ weight varies.
    \begin{lemma}\label{lem:price}
    Under (A4), there exist $\bar\gamma>0$ and a unique nearby $C^1$ branch $\eta_\gamma$ of strict local minimizers of $D-\gamma H$ through $\eta_0$, for $0\le\gamma<\bar\gamma$. Along this branch, $\eta_\gamma=\eta_0+\gamma v+o(\gamma)$, and
    \begin{equation}\label{eq:cost}
    D(\eta_\gamma)-D(\eta_0)=\tfrac12\kappa_0\gamma^2+o(\gamma^2), \qquad H(\phi_\gamma)-H(\phi_0)=\kappa_0\gamma+o(\gamma).
    \end{equation}
    \end{lemma}
    \Cref{lem:price} shows that, when the entropy gradient is nonzero, a small step in $\gamma$ buys a first-order increase in entropy at only a second-order increase in the surrogate cost. The local relation between surrogate cost and achieved entropy characterizes the cost of moving through parameter space. This movement provides a potential route to exploring different interventional image distributions, with (A5) specifying a local condition under which such variation occurs.
    
    \begin{theorem}\label{thm:move}
    Under (A4)--(A5), there exists $\rho_x>0$ such that $W_1(P_x^{\theta_\gamma},P_x^{\theta_0})=\gamma\rho_x+o(\gamma)$. Thus the interventional image distribution changes at first order while $D$ changes at second order, i.e.,
    \begin{align}\label{eq:law-distance}
    \lim_{\gamma\downarrow0}\frac{W_1(P_x^{\theta_\gamma},P_x^{\theta_0})}{\gamma}=\rho_x>0, \qquad \lim_{\gamma\downarrow0}\frac{D(\eta_\gamma)-D(\eta_0)}{W_1(P_x^{\theta_\gamma},P_x^{\theta_0})}=0.
    \end{align}
    \end{theorem}
    
    \Cref{thm:move} establishes that varying the entropy weight $\gamma$ can produce distinct interventional image distributions at a comparatively small observational surrogate cost: near the baseline, the observational cost increase divided by the interventional movement tends to zero. Both changes vanish as $\gamma$ decreases, but the surrogate cost vanishes at a faster rate. The remaining question is observational admissibility. A small increase in the surrogate $D$ does not ensure that a model is compatible with the observational data, because $D$ is a training surrogate rather than a discrepancy between the full observational distributions. Let $\Delta_j(\theta)$, $j=1,\ldots,J$, be a fixed finite collection of discrepancy measurements between $P^\theta(X,I)$ and the real data distribution $P(X, I)$, with observational constraints $\Delta_j(\theta)\le\varepsilon_j$. Acceptance refers to these finite checks, not equality of the full observational distributions. The following result establishes that there exists a sequence of entropy-weighted
    movements satisfying the finite observational checks while yielding distinct interventional image predictions.
    \begin{corollary}[Observationally admissible movement]\label{cor:admissible-move}
    Under the conditions of \cref{thm:move}, suppose each $\Delta_j$ is continuous at $\theta_0$ and the baseline satisfies $\Delta_j(\theta_0)<\varepsilon_j$. Then there exists $\delta>0$ such that every $0\le\gamma<\delta$ satisfies all observational constraints, and every $0<\gamma<\delta$ yields $P_x^{\theta_\gamma}\ne P_x^{\theta_0}$.
    \end{corollary}
    \Cref{cor:admissible-move} shows that interventional image predictions can change while the same observational constraints remain satisfied. These analyses motivate training candidates across a sequence of entropy weights $\Gamma=\{\gamma_1,\dots,\gamma_M\}$ and retaining those that satisfy common held-out constraints. \Cref{app:algorithm}  provides details of the training procedure. For each $\gamma\in \Gamma$, we minimize \cref{eq:objective} with gradient descent, reparameterizing $\mathbf Z$ and analytically summing over $C$. To estimate $H$, average the cluster posterior over reparameterized draws of $\mathbf Z$ before taking its entropy. After unregularized-ELBO pretraining initializes $(P(C),\mu_c,\sigma_c^2)$, jointly optimize the encoder, decoders, and classifier. Each model generates interventional samples according to \cref{eq:Px_param}. Observational admissibility uses a joint image-treatment held-out measurement, such as maximum mean discrepancy (MMD).

\section{Experiments}
\label{sec:experiments}
We evaluate \cauvade{} on confounded variants of Color-MNIST, CelebA \citep{liu2015celeba}, and MIMIC-CXR-JPG \citep{johnson2024mimiccxr,johnson2019mimic,PhysioNet}. Each construction introduces an unobserved confounder $U$ affecting both the treatment $X$ and outcome $Y$. We train models on confounded images; unconfounded data serve as a ground truth for evaluation. We compare \cauvade{} with a vanilla VAE and ANCM \citep{pan2024counterfactualimageediting}; \Cref{tab:cauvade-results} also reports conditional GAN (CGAN) and diffusion baselines on Color-MNIST and CelebA.

The experiments assess two questions: whether varying the entropy weight $\gamma$ produces diverse interventional predictions, and whether constructed causal explanations are compatible with observed data. For Color-MNIST and CelebA, we use the $(X,Y)$ labels both to check interventional estimates of $P_x(Y=1)$ against scalar Manski bounds \citep{manski1990nonparametric} and to assess observational fit through Label RMSE between model-implied joint semantic probabilities $P(X,Y)$ and their empirical counterparts. These complementary checks assess compatibility at the causal-semantic variable level, while reconstruction negative log-likelihood (NLL) and mean squared error (MSE) assess image-level observational fit. We also examine generated samples and report FID against the MIMIC-CXR-JPG reference. \Cref{app:setup} provides details of the evaluation metrics and training procedures.

\begin{table}[t]
    \centering
    \caption{Reference observational and interventional probabilities, baseline estimates, and \cauvade{} predictions on Confounded Color-MNIST and CelebA. Varying $\gamma$ yields different estimates of $P_x(Y=1)$ within the reported scalar Manski bounds \citep{manski1990nonparametric}.}
    \setlength{\tabcolsep}{1pt}
    \renewcommand{\arraystretch}{1}
    \resizebox{\textwidth}{!}{%
\begin{tabular}{ll cc cccc cc cccc}
    \toprule
    & & $P(Y=1\mid X)$ & $P_x(Y=1)$ & VAE & CGAN & Diffusion & ANCM & \multicolumn{2}{c}{Causal Bounds} & \multicolumn{4}{c}{\cauvade{}} \\
    \cmidrule(lr){9-10} \cmidrule(lr){11-14}
    & & & & & & & & Lower & Upper & $\gamma=0$ & $\gamma=1$ & $\gamma=10$ & $\gamma=100$ \\
    \midrule
    \multirow{2}{*}{C-MNIST} 
        & $X=0$ & 0.61 & 0.40 & $0.61 \pm 0.01$ & $0.61 \pm 0.01$ & $0.58 \pm 0.03$ & $0.58 \pm 0.01$ & 0.28 & 0.82 & 0.57 & 0.54 & 0.49 & 0.72 \\
        \cmidrule(lr){2-14}
        & $X=1$ & 0.78 & 0.60 & $0.81 \pm 0.01$ & $0.79 \pm 0.01$ & $0.77 \pm 0.03$ & $0.62 \pm 0.02$ & 0.42 & 0.88 & 0.61 & 0.56 & 0.54 & 0.83 \\
    \midrule
    \multirow{2}{*}{CelebA} 
        & $X=0$ & 0.46 & 0.33 & $0.46 \pm 0.01$ & $0.41 \pm 0.01$ & $0.51 \pm 0.03$ & $0.49 \pm 0.02$ & 0.26 & 0.73 & 0.42 & 0.57 & 0.48 & 0.40 \\
        \cmidrule(lr){2-14}
        & $X=1$ & 0.31 & 0.46 & $0.30 \pm 0.02$ & $0.36 \pm 0.01$ & $0.35 \pm 0.01$ & $0.48 \pm 0.02$ & 0.14 & 0.67 & 0.44 & 0.51 & 0.48 & 0.39 \\
    \bottomrule
\end{tabular}
    }
    \label{tab:cauvade-results}
\end{table}

\begin{figure*}[t]
\hfill
    \begin{subfigure}{0.16\linewidth}\centering
        \setlength{\abovecaptionskip}{0pt}
        \includegraphics[width=\linewidth]{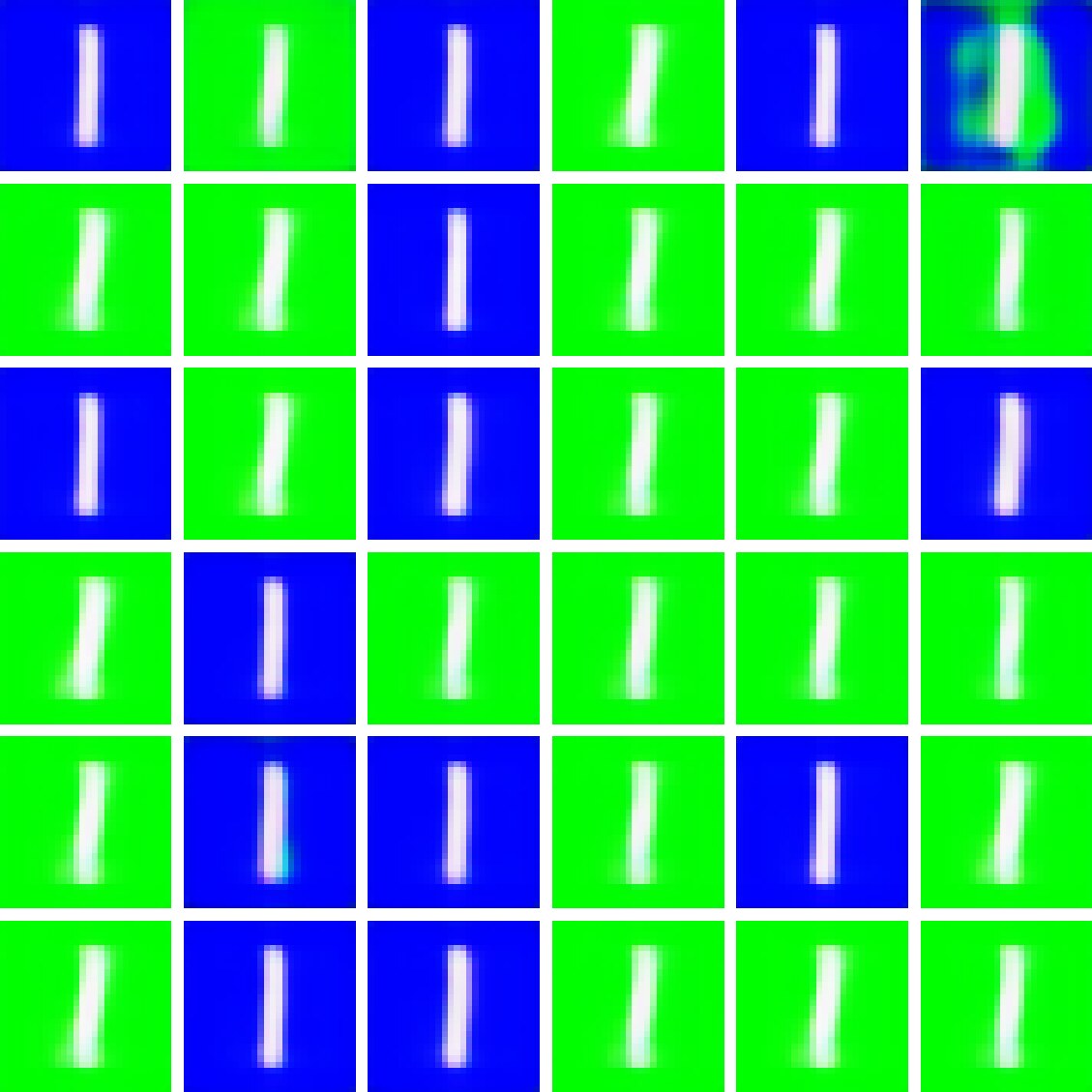}
        \caption{}\label{fig:_cmnist_sample_a}
    \end{subfigure}\hfill
    \begin{subfigure}{0.16\linewidth}\centering
        \setlength{\abovecaptionskip}{0pt}
        \includegraphics[width=\linewidth]{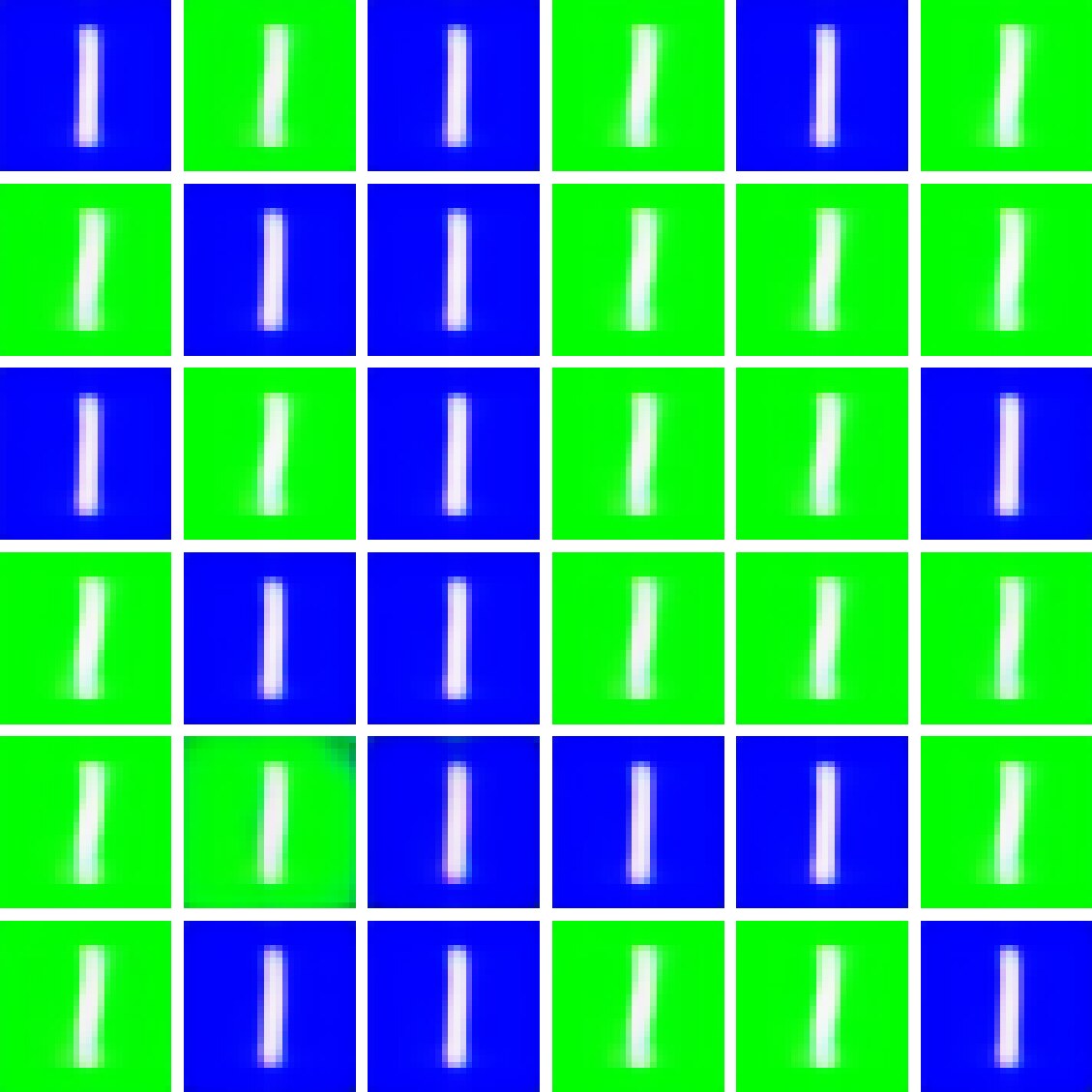}
        \caption{}\label{fig:_cmnist_sample_b}
    \end{subfigure}\hfill
    \begin{subfigure}{0.16\linewidth}\centering
        \setlength{\abovecaptionskip}{0pt}
        \includegraphics[width=\linewidth]{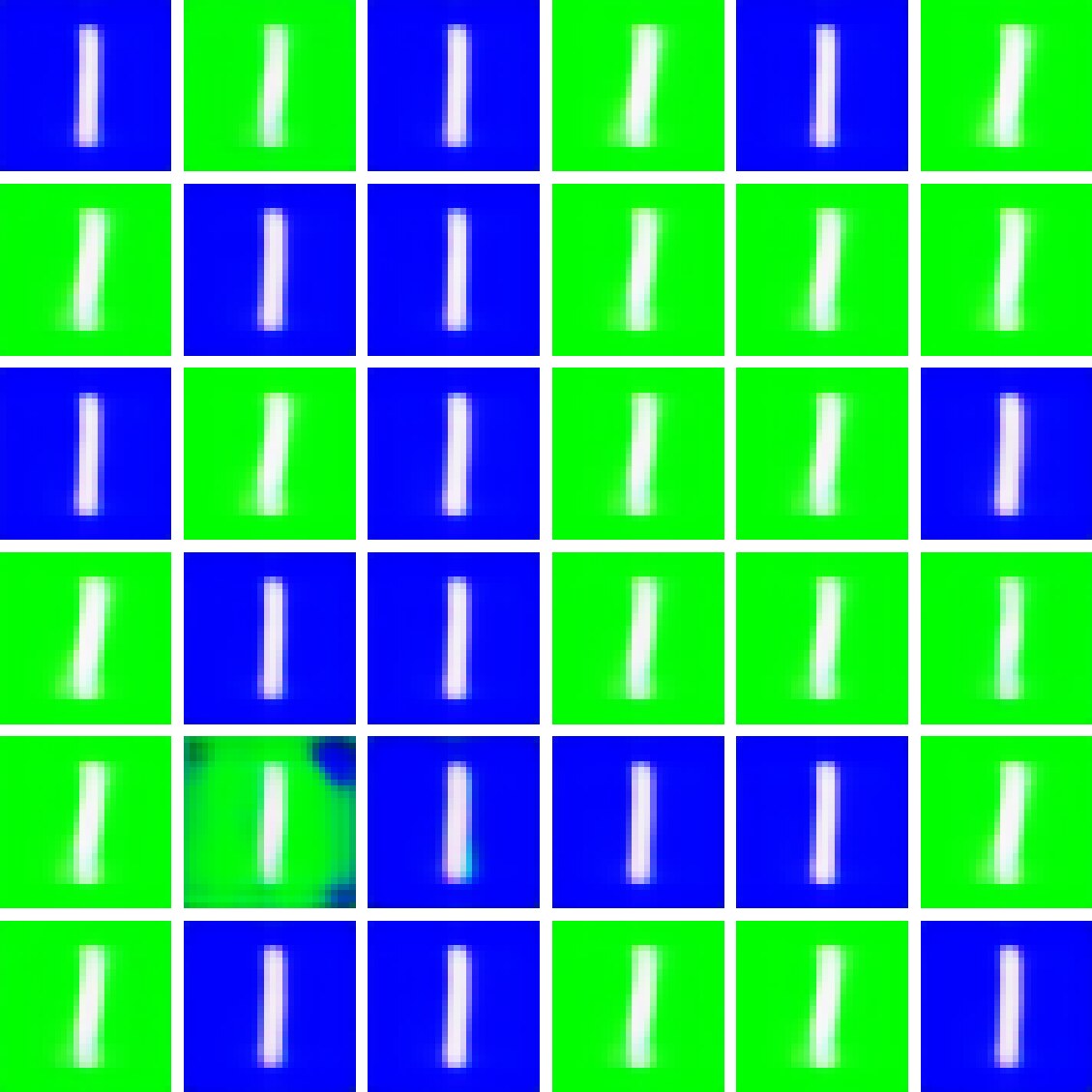}
        \caption{}\label{fig:_cmnist_sample_c}
    \end{subfigure}\hfill
    \begin{subfigure}{0.16\linewidth}\centering
        \setlength{\abovecaptionskip}{0pt}
        \includegraphics[width=\linewidth]{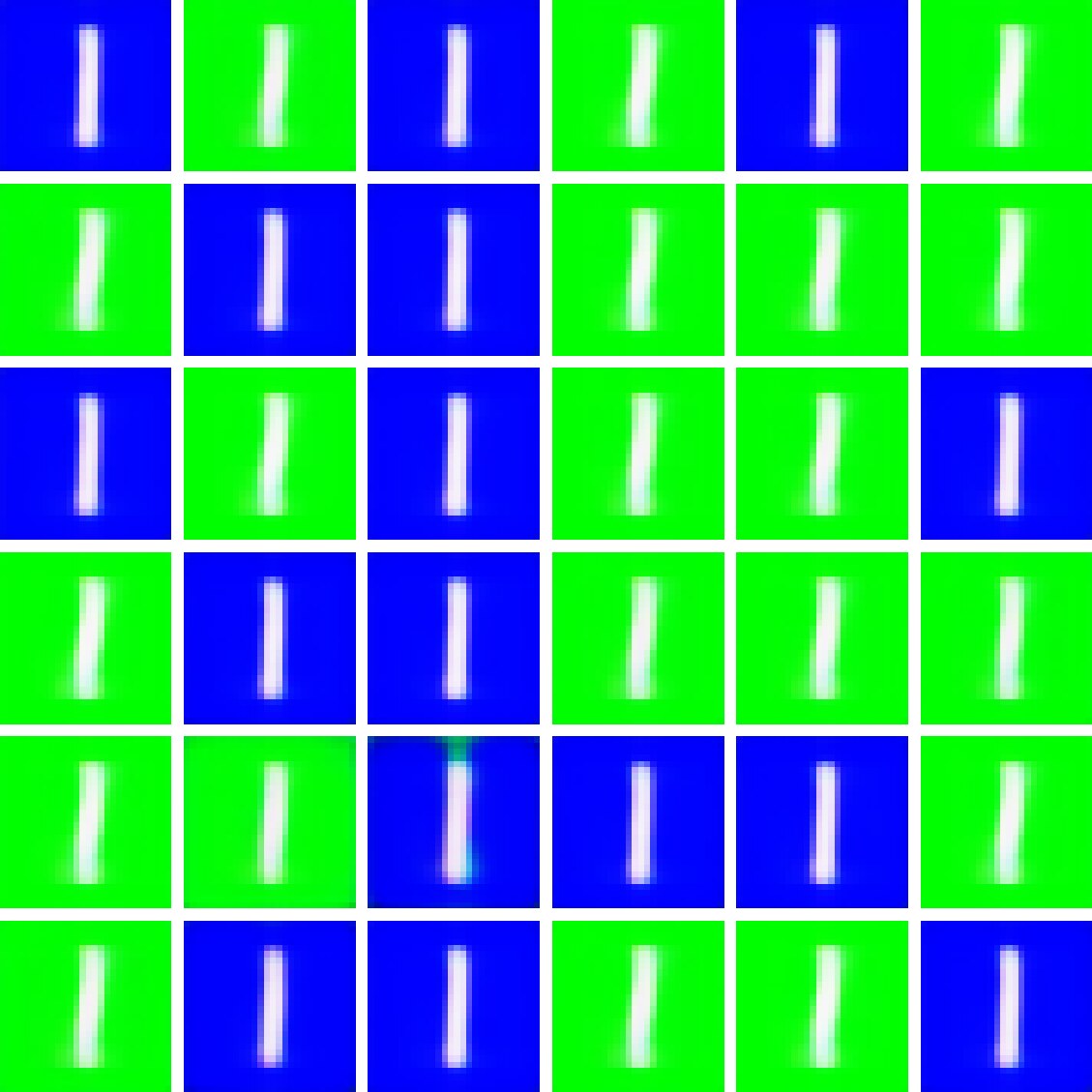}
        \caption{}\label{fig:_cmnist_sample_d}
    \end{subfigure}\hfill
    \begin{subfigure}{0.16\linewidth}\centering
        \setlength{\abovecaptionskip}{0pt}
        \includegraphics[width=\linewidth]{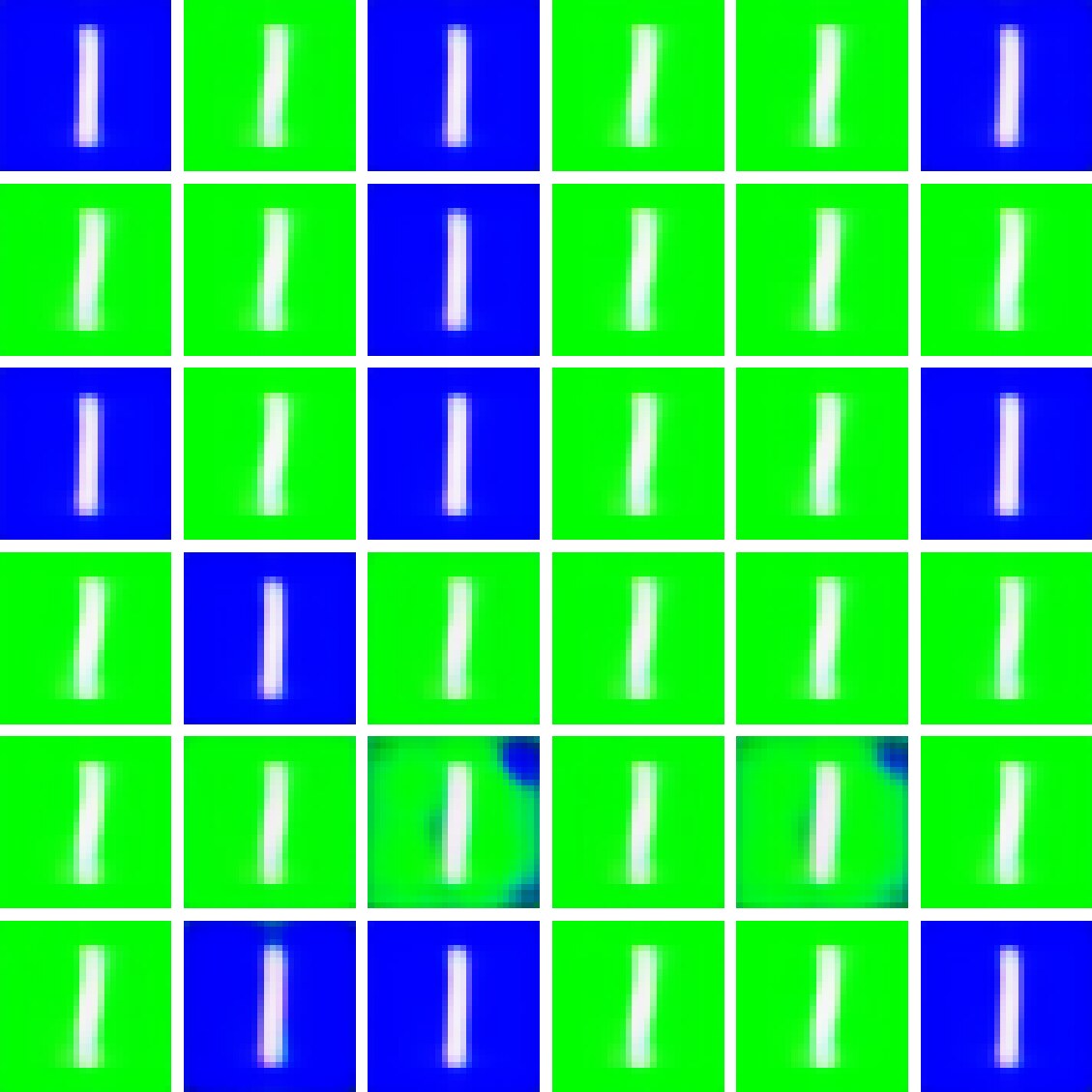}
        \caption{}\label{fig:_cmnist_sample_e}
    \end{subfigure}\hfill
    \begin{subfigure}{0.16\linewidth}\centering
        \setlength{\abovecaptionskip}{0pt}
        \includegraphics[width=\linewidth]{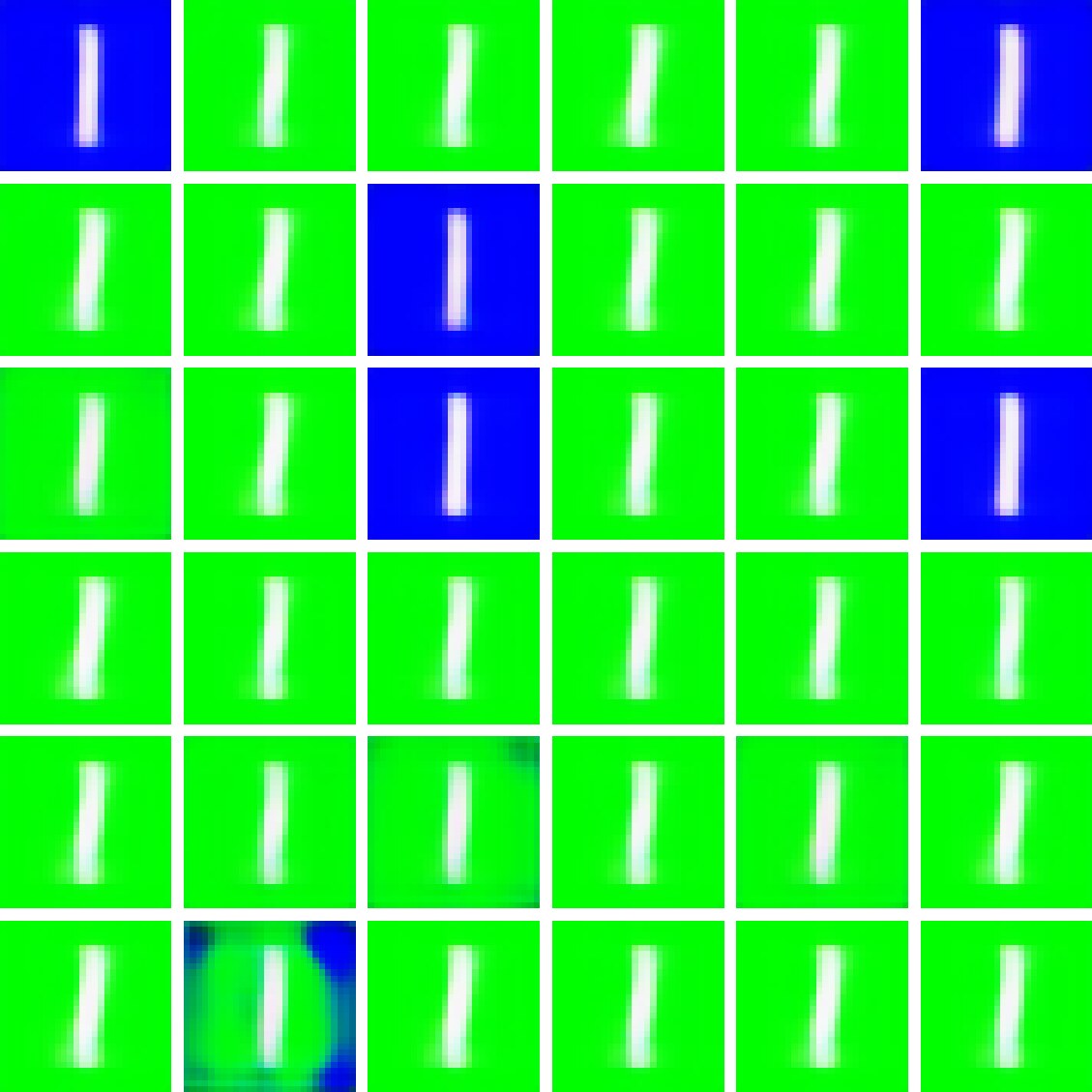}
        \caption{}\label{fig:_cmnist_sample_f}
    \end{subfigure}\hfill
    \caption{\cauvade{} samples on Confounded Color-MNIST under $X=1$ at
    (\subref{fig:_cmnist_sample_a}) $\gamma=1$,
    (\subref{fig:_cmnist_sample_b}) $\gamma=2$,
    (\subref{fig:_cmnist_sample_c}) $\gamma=10$,
    (\subref{fig:_cmnist_sample_d}) $\gamma=20$,
    (\subref{fig:_cmnist_sample_e}) $\gamma=50$, and
    (\subref{fig:_cmnist_sample_f}) $\gamma=100$. The digit--color association varies across entropy weights, illustrating changes in the fitted interventional predictions.}
    \label{fig:_cmnist_sample}
\end{figure*}

\begin{wrapfigure}[17]{r}{0.55\linewidth}
\hfill
    \begin{subfigure}{0.49\linewidth}\centering
        \includegraphics[width=\linewidth]{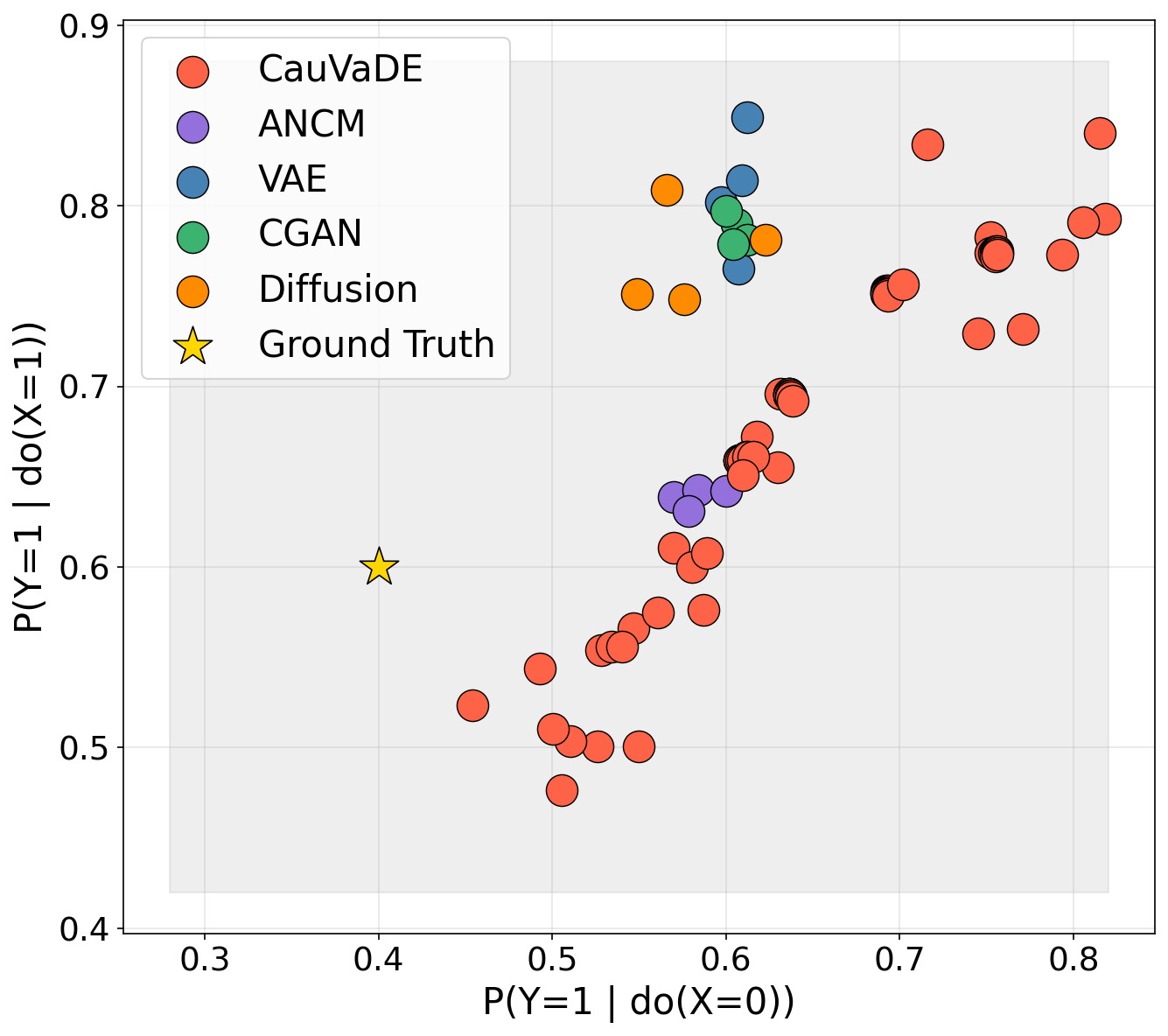}
        \caption{Color-MNIST}
        \label{fig:_vade_gamma_cmnist}
    \end{subfigure}\hfill
    \begin{subfigure}{0.49\linewidth}\centering
        \includegraphics[width=\linewidth]{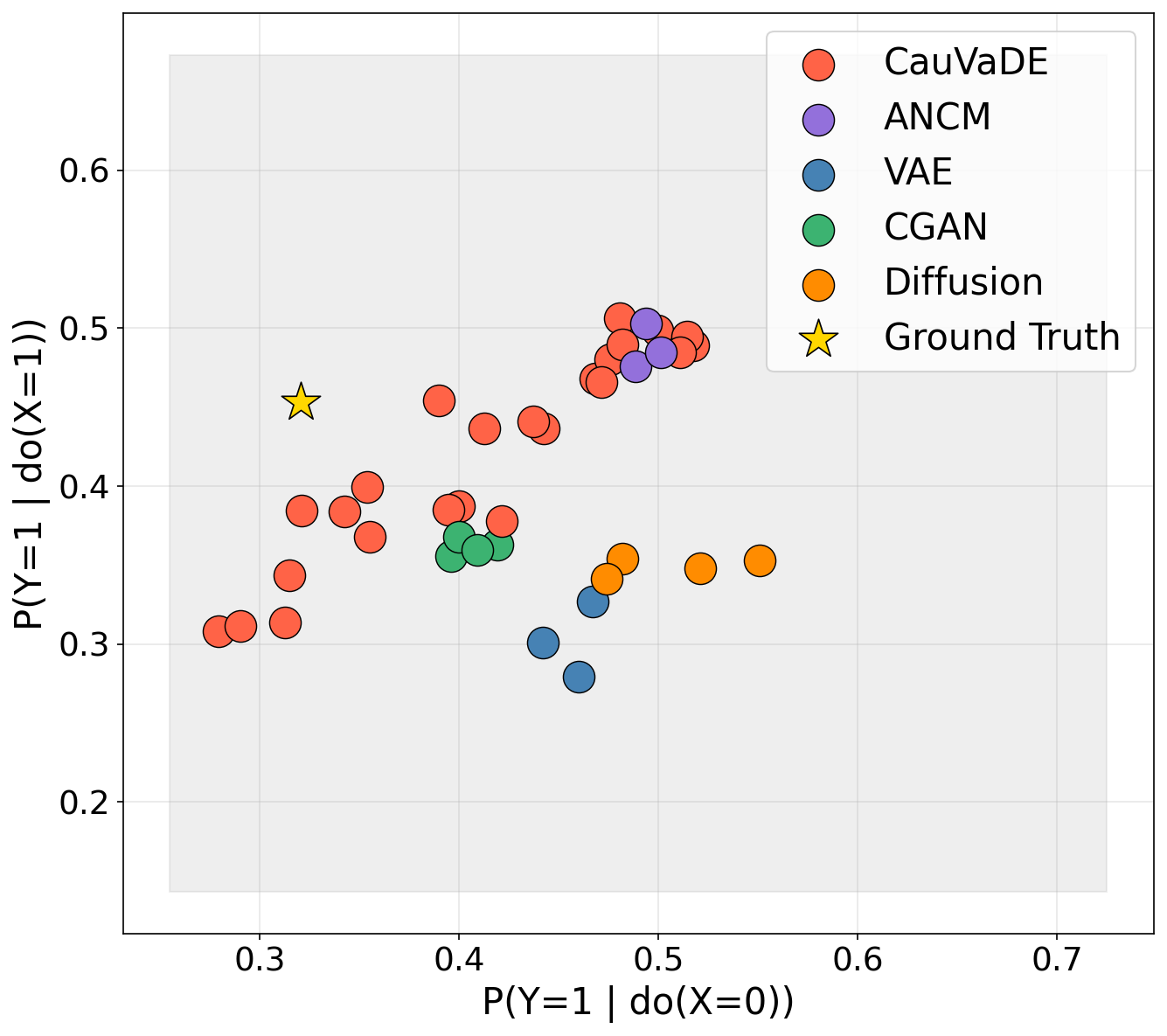}
        \caption{Confounded CelebA}
        \label{fig:_vade_gamma_celeba}
    \end{subfigure}\hfill
    \caption{Estimated $P_x(Y=1)$ on Confounded Color-MNIST and CelebA. VAE and ANCM estimates concentrate across seeds, whereas \cauvade{} predictions vary with $\gamma$. The shaded feasible region denotes the evaluated Manski's causal bounds \citep{manski1990nonparametric}.}
    \label{fig:_vade_gamma}
\end{wrapfigure}
\textbf{Confounded Color-MNIST.}
We construct Color-MNIST with digit labels $X\in\{0,1\}$, background colors $Y\in\{\text{green},\text{blue}\}$ (encoded as $0,1$), and a binary confounder $U$ affecting both. \Crefrange{fig:_cmnist_a}{fig:_cmnist_b} contrast observational samples from $P(I\mid X=1)$ with the interventional reference $P_{X=1}(I)$; construction details and additional samples appear in \cref{app:setup}. Across five seeds, VAE and ANCM yield concentrated estimates of $P_x(Y=1)$ (\cref{tab:cauvade-results}). VAE produces nearly identical estimates for the two treatment values, whereas ANCM distinguishes them. Varying $\gamma$ in \cauvade{} yields a broader range of predictions than seed variation in either baseline, with all tabulated estimates remaining within the Manski bounds. \Cref{fig:_vade_gamma_cmnist,fig:_cmnist_sample} show the corresponding changes in predicted color frequencies and generated images. Across the reconstruction's entropy weights, per-pixel MSE ranges from $0.0303$ to $0.0366$, NLL from $0.1153$ to $0.1362$, and Label RMSE for $P(X,Y)$ from $0.0067$ to $0.0519$ (\cref{tab:recon_gamma,app:reconstruction}). The small errors in the labeled distribution, together with stable reconstruction metrics, provide support for observational admissibility as interventional predictions vary.

\begin{figure}[t]
\hfill
    \begin{subfigure}{0.245\linewidth}\centering
        \setlength{\abovecaptionskip}{0pt}
        \includegraphics[width=\linewidth]{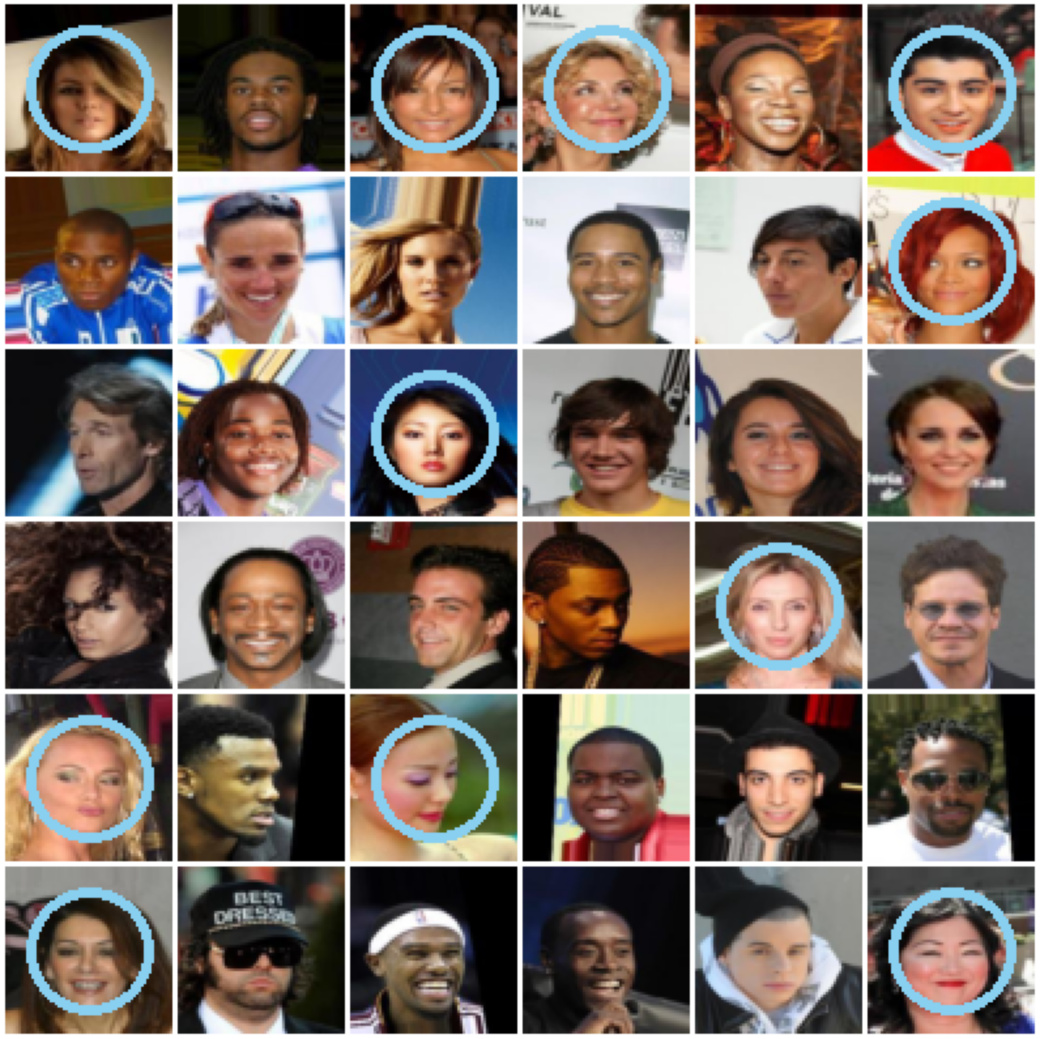}
        \caption{}\label{fig:_celeb_confound}
    \end{subfigure}\hfill
    \begin{subfigure}{0.245\linewidth}\centering
        \setlength{\abovecaptionskip}{0pt}
        \includegraphics[width=\linewidth]{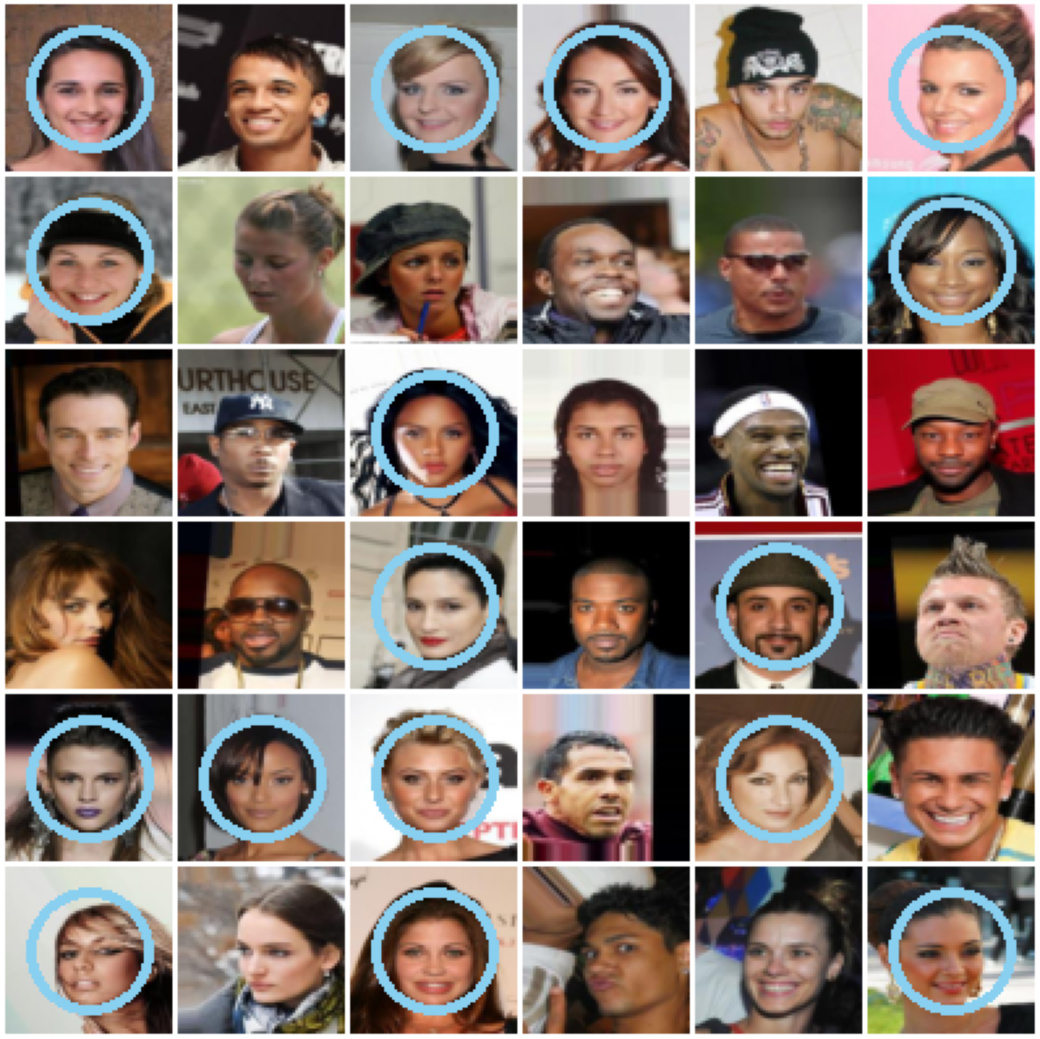}
        \caption{}\label{fig:_celeb_unconfound}
    \end{subfigure}\hfill
    \begin{subfigure}{0.245\linewidth}\centering
        \setlength{\abovecaptionskip}{0pt}
        \includegraphics[width=\linewidth]{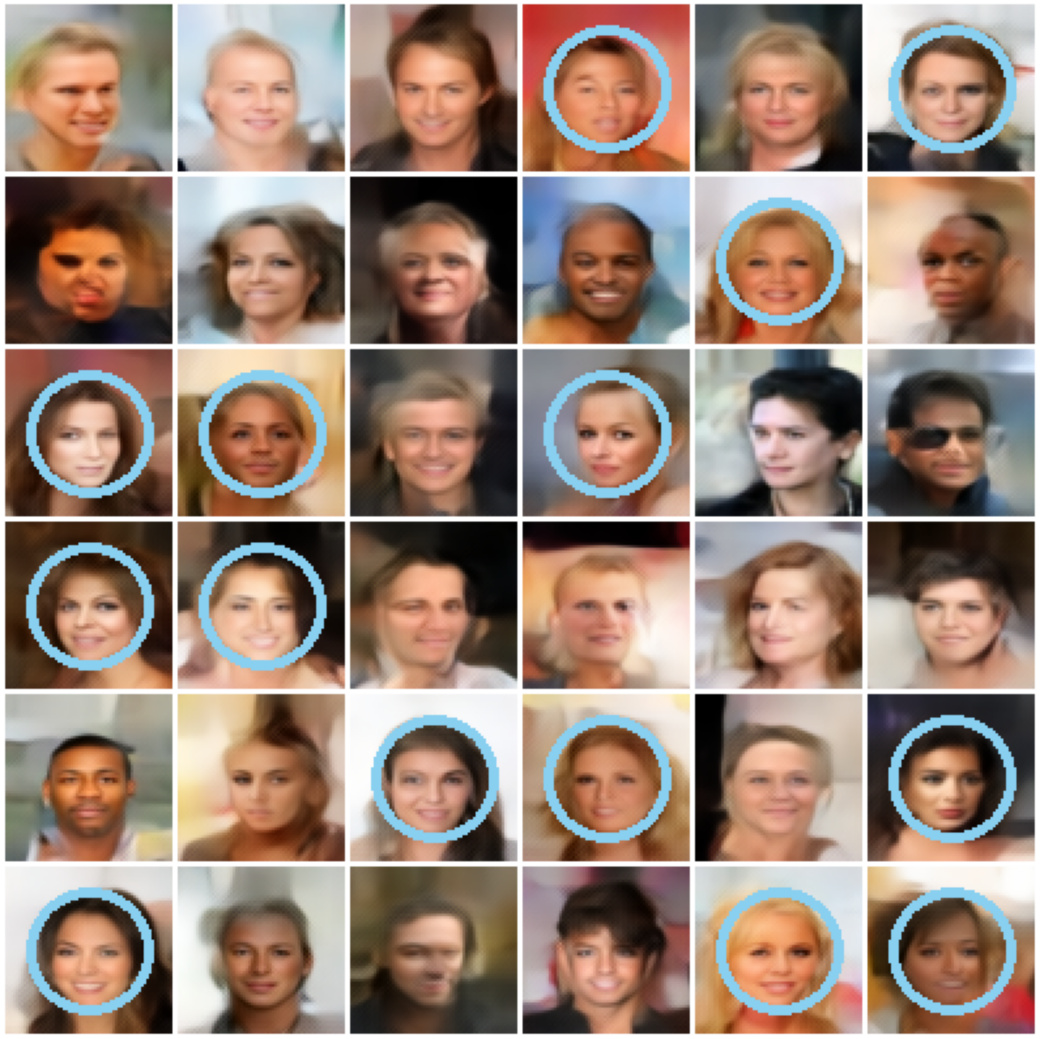}
        \caption{}\label{fig:_celeb_VAE}
    \end{subfigure}
    \begin{subfigure}{0.245\linewidth}\centering
        \setlength{\abovecaptionskip}{0pt}
        \includegraphics[width=\linewidth]{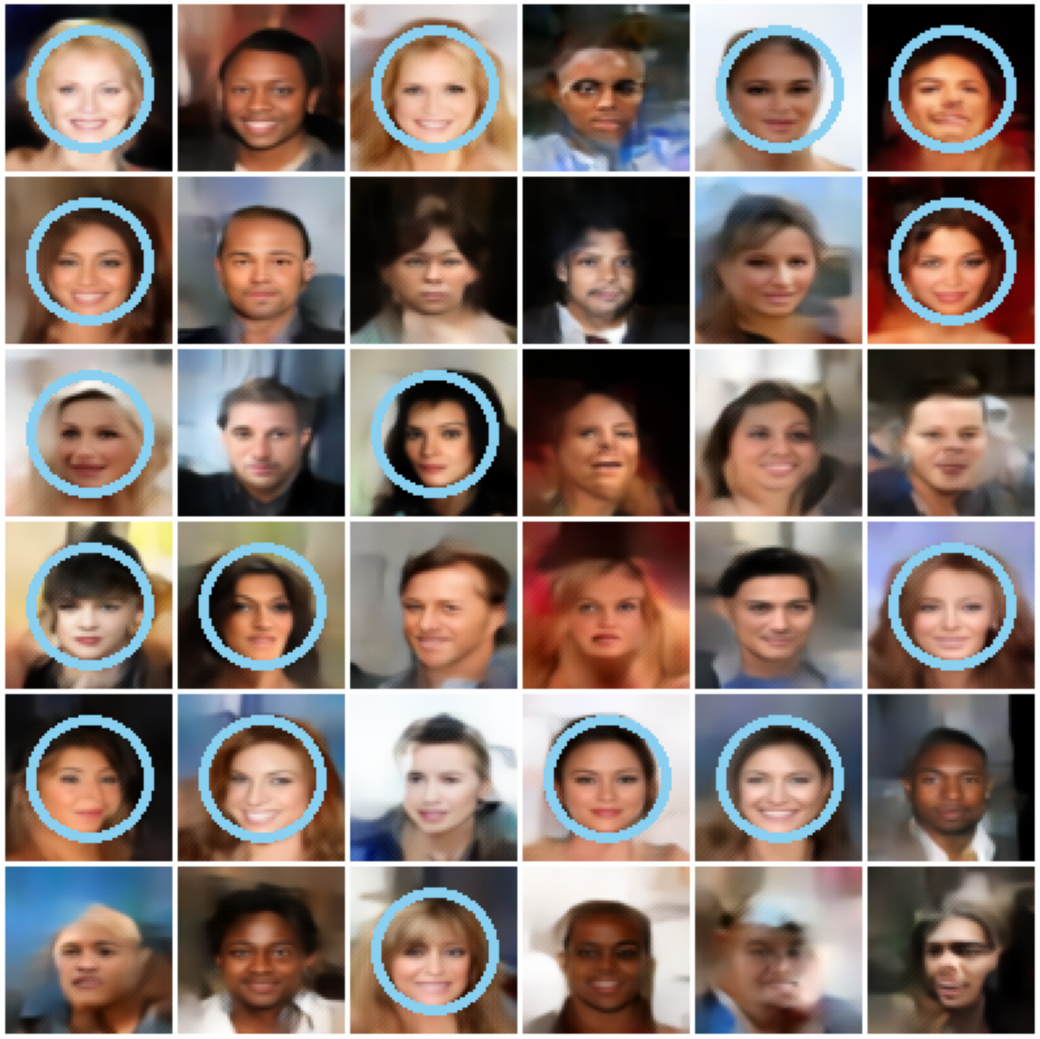}
        \caption{}\label{fig:_celeb_ANCM}
    \end{subfigure}
    \begin{subfigure}{0.245\linewidth}\centering
        \setlength{\abovecaptionskip}{0pt}
        \includegraphics[width=\linewidth]{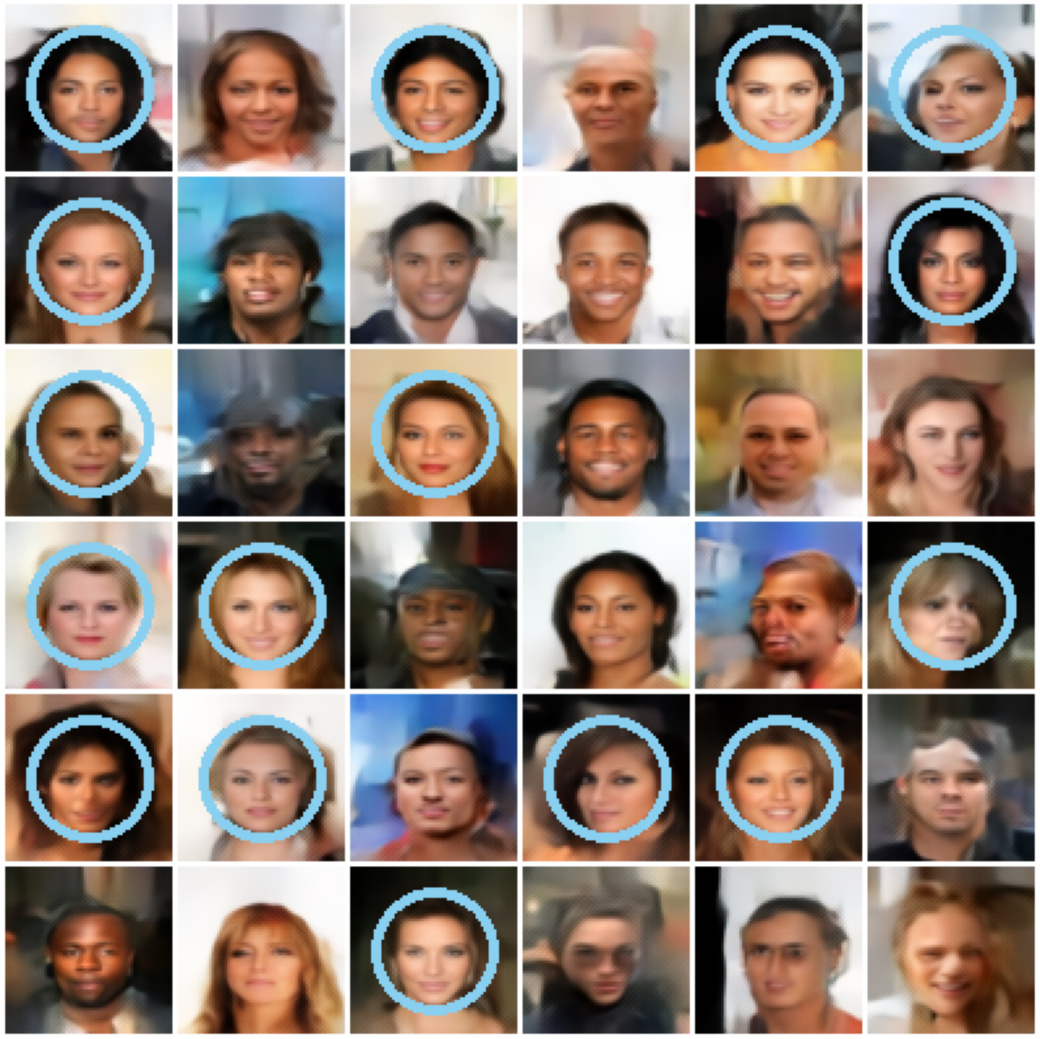}
        \caption{}\label{fig:_celeb_gamma0}
    \end{subfigure}\hfill
    \begin{subfigure}{0.245\linewidth}\centering
        \setlength{\abovecaptionskip}{0pt}
        \includegraphics[width=\linewidth]{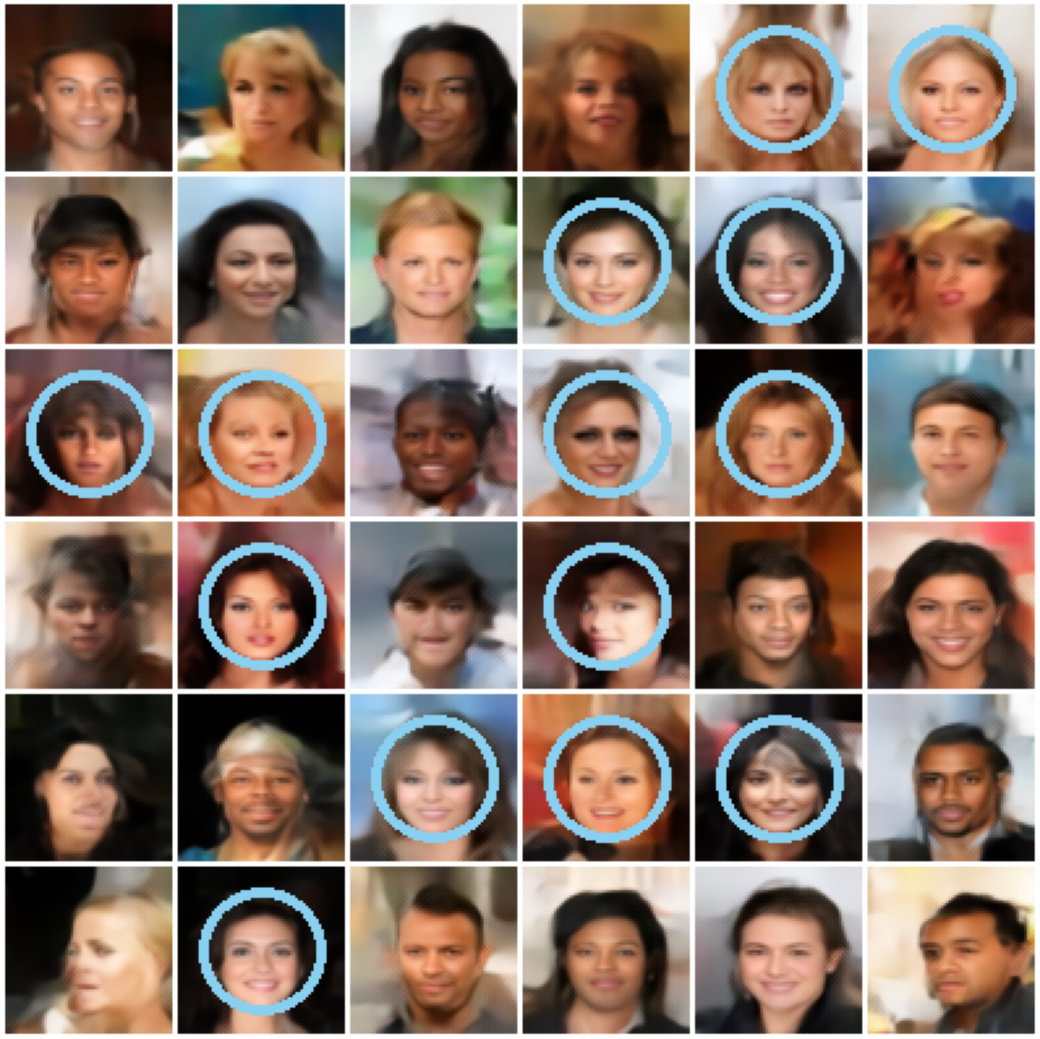}
        \caption{}\label{fig:_celeb_gamma1}
    \end{subfigure}\hfill
    \begin{subfigure}{0.245\linewidth}\centering
        \setlength{\abovecaptionskip}{0pt}
        \includegraphics[width=\linewidth]{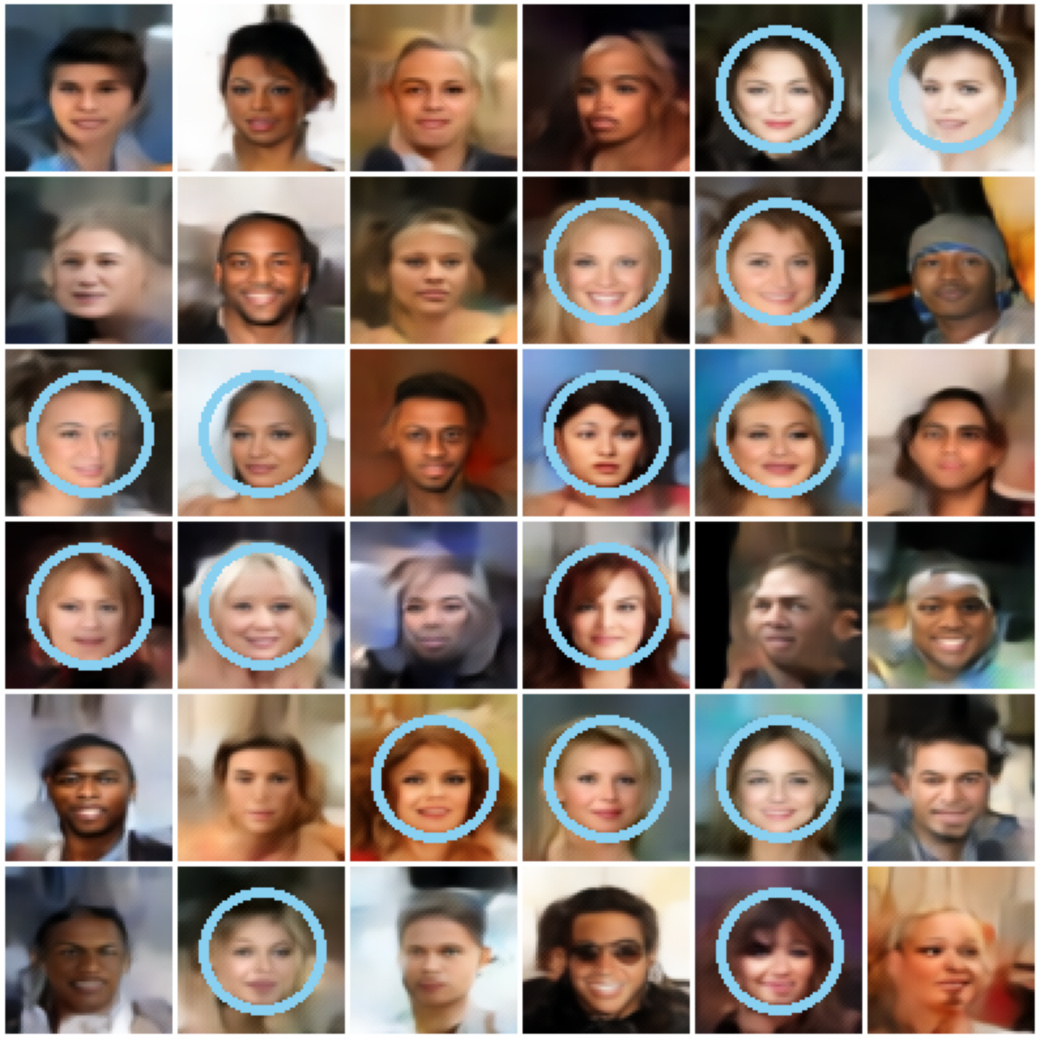}
        \caption{}\label{fig:_celeb_gamma10}
    \end{subfigure}\hfill    
    \begin{subfigure}{0.245\linewidth}\centering
        \setlength{\abovecaptionskip}{0pt}
        \includegraphics[width=\linewidth]{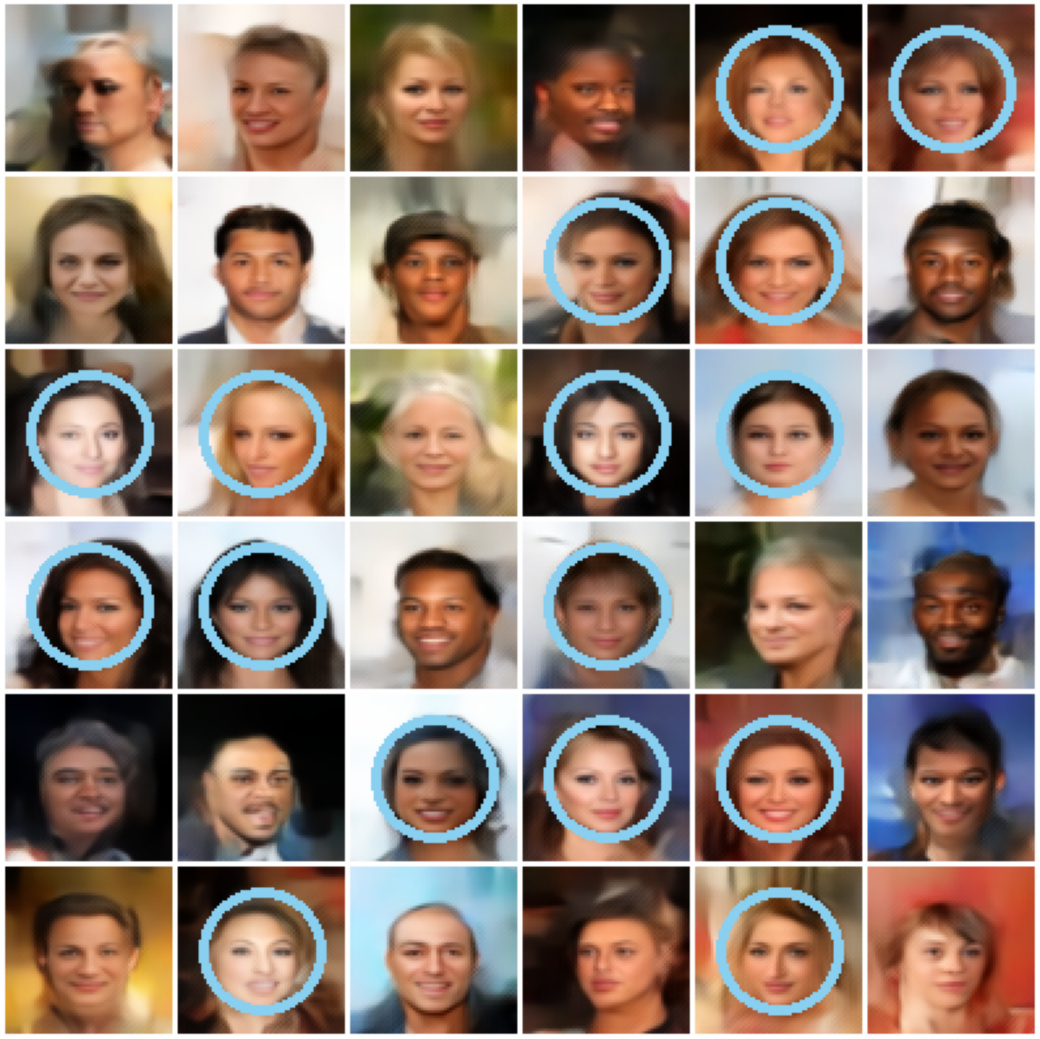}
        \caption{}\label{fig:_celeb_gamma100}
    \end{subfigure}\hfill\null
    \caption{
    (\subref{fig:_celeb_confound}) Confounded CelebA samples $P(I \mid X=1)$;
    (\subref{fig:_celeb_unconfound}) unconfounded reference $P(I \mid \doX{1})$;
    (\subref{fig:_celeb_VAE}) VAE;
    (\subref{fig:_celeb_ANCM}) ANCM;
    (\subref{fig:_celeb_gamma0}) \cauvade{} at $\gamma=0$;
    (\subref{fig:_celeb_gamma1}) \cauvade{} at $\gamma=1$;
    (\subref{fig:_celeb_gamma10}) \cauvade{} at $\gamma=10$;
    (\subref{fig:_celeb_gamma100}) \cauvade{} at $\gamma=100$.
    Blue circles mark images with $Y=1$ (Heavy Makeup).}
    \label{fig:_celeba}
\end{figure}

\textbf{Confounded CelebA.}
We resample images from CelebA (202,599 source images), assigning ``Young'' to $X$, ``Heavy Makeup'' to $Y$, and ``Attractive'' to the hidden confounder $U$. The resampling protocol is given in \cref{app:celeba}. \Crefrange{fig:_celeb_confound}{fig:_celeb_unconfound} compare the confounded data and unconfounded reference at $X=1$. VAE and ANCM show limited variation across three seeds, while changing $\gamma$ yields a range of \cauvade{} predictions within the evaluated Manski bounds (\cref{fig:_vade_gamma_celeba,tab:cauvade-results}). The tabulated predictions vary nonmonotonically with $\gamma$ over the evaluated range. For $X=1$, the tabulated estimates span $0.385$--$0.506$, compared with the reference probability of $0.463$, which lies within this range. Thus, increasing $\gamma$ does not necessarily bring predictions closer to the reference. Samples at $\gamma\in\{1,10\}$ visually resemble the unconfounded reference (\cref{fig:_celeba}), although this resemblance does not identify the causal mechanism. Reconstruction MSE and NLL remain stable across the entropy weights, while Label RMSE for $P(X,Y)$ ranges from $0.0055$ to $0.0387$ (\cref{tab:recon_gamma_celeba}). Together, these diagnostics support the observational admissibility as interventional predictions vary.

\begin{wraptable}[15]{r}{0.45\textwidth}
    \caption{FID against the unconfounded MIMIC-CXR-JPG reference. The \cauvade{} entry averages over $\gamma\in\{1,10,100\}$. The lower the better.}
    \label{tab:fid}
    \centering
    \small
\vspace{-0.1in}
\begin{tabular}{lccc}
    \toprule
    & VAE & ANCM & \cauvade{} \\
    \midrule
    FID & 354.97 & 407.73 & $\mathbf{283.57}$ \\
    \bottomrule
\end{tabular}
\captionsetup{name=Figure}
    \centering
    \begin{subfigure}[t]{0.99\linewidth}\centering
        \includegraphics[width=\linewidth]{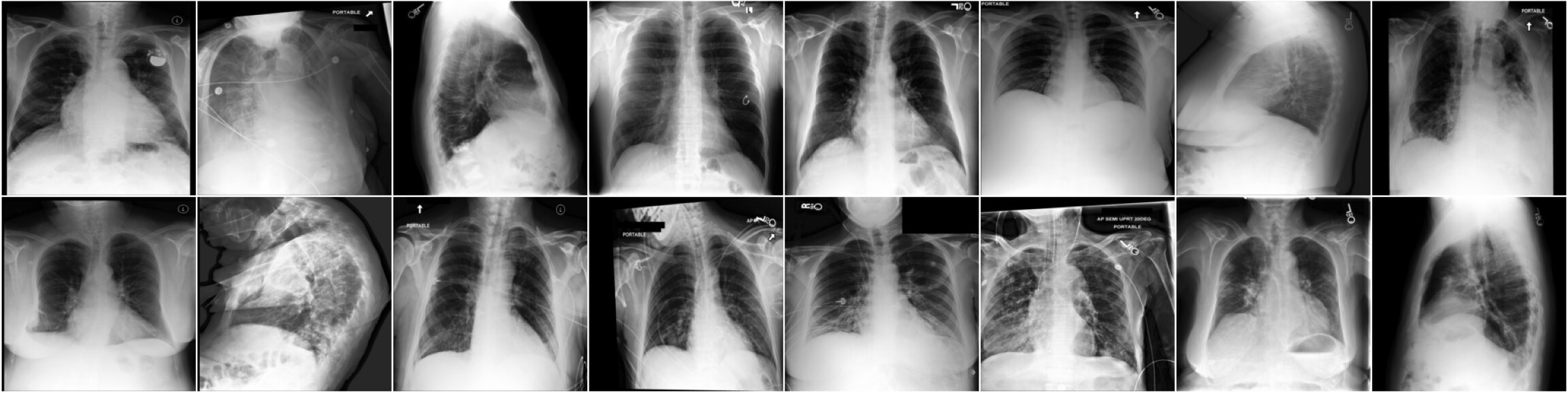}
    \end{subfigure}\hfill
    \renewcommand{\thetable}{\thefigure}
    \refstepcounter{figure}
    \addtocounter{table}{-1}
    \label{fig:_mimic_example}
    \vspace{-0.8\baselineskip}
    \caption{Confounded MIMIC-CXR-JPG samples: $X=0$ (top row, no Pneumonia) and $X=1$ (bottom row, Pneumonia).}
\end{wraptable}
\textbf{Confounded MIMIC-CXR-JPG.}
We construct a resampling experiment from MIMIC-CXR-JPG \citep{johnson2024mimiccxr,johnson2019mimic,PhysioNet} (377,110 source radiographs), using Pneumonia as $X$ and subsequent outcomes as $Y$ (e.g., Lung Opacity, unlabeled). The hidden confounder $U$ is the patient's age and the radiograph's anteroposterior view. \Cref{fig:_mimic_example} shows confounded samples. This comparison tests image generation when treatment assignment no longer depends on the age--view grouping. FID summarizes image-distribution agreement rather than the accuracy of individual clinical attributes.
For $\gamma\in\{1,10,100\}$, \cauvade{} achieves FIDs of $276.64$, $284.57$, and $289.51$, respectively.
Averaged across these entropy weights, its FID is $283.57$, compared with $354.97$ for VAE and $407.73$ for ANCM (\cref{tab:fid}), reductions of $71.40$ and $124.16$, respectively.
All three settings yield lower FID than both baselines, indicating closer image-distribution agreement with the unconfounded reference.
Alongside these changes, we compare observational fit with the $\gamma=0$ baseline to check whether diverse causal predictions violate admissibility.

\section{Conclusion}\label{sec:conclusion}

Unobserved confounding means that the same observed images can support different predictions about the effects of an intervention. Our augmented causal approximation result (\cref{thm:approx}) provides a foundation for representing both observed and intervention-generated images using a discrete confounder. \cauvade{} instantiates this canonical augmented causal representation as a Gaussian-mixture VAE. A causal bootstrap then explores different causal explanations and their resulting interventional image predictions by varying a training weight. Our analysis explains how the resulting predictions can change while the fit to the observed data remains nearly unchanged (\cref{lem:price,thm:move}). Experiments on Color-MNIST and CelebA show distinct predictions within the causal bounds derived from labeled semantics. On MIMIC-CXR-JPG, generated images more closely match the unconfounded ground truth images on average than those from the compared baselines. Determining the full range of effects consistent with the observed data remains an open question.

\subsection*{Acknowledgments}
This work was supported in part by the Build on Trainium Amazon Research Award. The authors gratefully acknowledge the provider for the computing infrastructure that enabled the experiments in this paper.

\bibliographystyle{iclr2027_conference}
\bibliography{add_ref}

@inproceedings{kingma2022vae,
  author    = {Kingma, Diederik P. and Welling, Max},
  title     = {Auto-Encoding Variational {B}ayes},
  booktitle = {2nd International Conference on Learning Representations (ICLR)},
  year      = {2014},
  url       = {https://arxiv.org/abs/1312.6114}
}

@article{goodfellow2014gan,
  author    = {Ian J. Goodfellow and Jean Pouget-Abadie and Mehdi Mirza and Bing Xu and David Warde-Farley and Sherjil Ozair and Aaron Courville and Yoshua Bengio},
  title     = {Generative Adversarial Networks},
  journal   = {arXiv preprint arXiv:1406.2661},
  year      = {2014},
  url       = {https://arxiv.org/abs/1406.2661}
}

@article{ho2020ddpm,
  author    = {Jonathan Ho and Ajay Jain and Pieter Abbeel},
  title     = {Denoising Diffusion Probabilistic Models},
  journal   = {arXiv preprint arXiv:2006.11239},
  year      = {2020},
  url       = {https://arxiv.org/abs/2006.11239}
}

@inproceedings{jiang2017vade,
  author    = {Zhuxi Jiang and Yin Zheng and Huachun Tan and Bangsheng Tang and Hanning Zhou},
  title     = {Variational Deep Embedding: An Unsupervised and Generative Approach to Clustering},
  booktitle = {Proceedings of the 26th International Joint Conference on Artificial Intelligence (IJCAI)},
  pages     = {1965--1972},
  year      = {2017},
  doi       = {10.24963/ijcai.2017/273}
}

@inproceedings{khemakhem2020vae,
  author    = {Khemakhem, Ilyes and Kingma, Diederik P. and Monti, Ricardo Pio and Hyv{\"a}rinen, Aapo},
  title     = {Variational Autoencoders and Nonlinear {ICA}: A Unifying Framework},
  booktitle = {Proceedings of the Twenty Third International Conference on Artificial Intelligence and Statistics (AISTATS)},
  series    = {Proceedings of Machine Learning Research},
  volume    = {108},
  pages     = {2207--2217},
  publisher = {PMLR},
  year      = {2020},
  url       = {https://proceedings.mlr.press/v108/khemakhem20a.html}
}

@inproceedings{locatello2019disentanglement,
  author    = {Locatello, Francesco and Bauer, Stefan and Lucic, Mario and R{\"a}tsch, Gunnar and Gelly, Sylvain and Sch{\"o}lkopf, Bernhard and Bachem, Olivier},
  title     = {Challenging Common Assumptions in the Unsupervised Learning of Disentangled Representations},
  booktitle = {Proceedings of the 36th International Conference on Machine Learning (ICML)},
  series    = {Proceedings of Machine Learning Research},
  volume    = {97},
  pages     = {4114--4124},
  publisher = {PMLR},
  year      = {2019},
  url       = {https://proceedings.mlr.press/v97/locatello19a.html}
}

@book{pearl:2k,
  author    = {Pearl, Judea},
  title     = {Causality: Models, Reasoning, and Inference},
  publisher = {Cambridge University Press},
  address   = {New York},
  year      = {2000},
  isbn      = {978-0-521-77362-1}
}

@incollection{bareinboim2020pearl,
  author    = {Bareinboim, Elias and Correa, Juan D. and Ibeling, Duligur and Icard, Thomas},
  title     = {On {P}earl's Hierarchy and the Foundations of Causal Inference},
  booktitle = {Probabilistic and Causal Inference: The Works of Judea Pearl},
  editor    = {Geffner, Hector and Dechter, Rina and Halpern, Joseph Y.},
  series    = {ACM Books},
  publisher = {Association for Computing Machinery},
  address   = {New York, NY, USA},
  pages     = {507--556},
  year      = {2022},
  doi       = {10.1145/3501714.3501743}
}

@article{balke:pea97,
  author    = {Balke, Alexander and Pearl, Judea},
  title     = {Bounds on Treatment Effects from Studies with Imperfect Compliance},
  journal   = {Journal of the American Statistical Association},
  volume    = {92},
  number    = {439},
  pages     = {1171--1176},
  year      = {1997},
  doi       = {10.1080/01621459.1997.10474074}
}

@article{manski1990nonparametric,
  title     = {Nonparametric Bounds on Treatment Effects},
  author    = {Manski, Charles F.},
  journal   = {The American Economic Review},
  volume    = {80},
  number    = {2},
  pages     = {319--323},
  year      = {1990},
  publisher = {American Economic Association},
  url       = {https://www.jstor.org/stable/2006592}
}

@article{manski1998monotone,
  author  = {Manski, Charles F.},
  title   = {Monotone Treatment Response},
  journal = {Econometrica},
  volume  = {65},
  number  = {6},
  pages   = {1311--1334},
  year    = {1997},
  doi     = {10.2307/2171738}
}

@article{frangakis:rub02,
  author  = {Frangakis, Constantine E. and Rubin, Donald B.},
  title   = {Principal Stratification in Causal Inference},
  journal = {Biometrics},
  volume  = {58},
  number  = {1},
  pages   = {21--29},
  year    = {2002},
  doi     = {10.1111/j.0006-341X.2002.00021.x}
}

@inproceedings{zhang2022partial,
  author    = {Zhang, Junzhe and Tian, Jin and Bareinboim, Elias},
  title     = {Partial Counterfactual Identification from Observational and Experimental Data},
  booktitle = {Proceedings of the 39th International Conference on Machine Learning (ICML)},
  series    = {Proceedings of Machine Learning Research},
  volume    = {162},
  pages     = {26548--26558},
  publisher = {PMLR},
  year      = {2022},
  url       = {https://proceedings.mlr.press/v162/zhang22ab.html}
}

@article{duarte2024automated,
  author    = {Duarte, Guilherme and Finkelstein, Noam and Knox, Dean and Mummolo, Jonathan and Shpitser, Ilya},
  title     = {An Automated Approach to Causal Inference in Discrete Settings},
  journal   = {Journal of the American Statistical Association},
  volume    = {119},
  number    = {547},
  pages     = {1778--1793},
  year      = {2024},
  publisher = {Taylor \& Francis},
  doi       = {10.1080/01621459.2023.2216909}
}

@article{scholkopf2021crl,
  author  = {Sch{\"o}lkopf, Bernhard and Locatello, Francesco and Bauer, Stefan and Ke, Nan Rosemary and Kalchbrenner, Nal and Goyal, Anirudh and Bengio, Yoshua},
  title   = {Toward Causal Representation Learning},
  journal = {Proceedings of the IEEE},
  volume  = {109},
  number  = {5},
  pages   = {612--634},
  year    = {2021},
  doi     = {10.1109/JPROC.2021.3058954}
}

@inproceedings{xia2021causal,
  author    = {Xia, Kevin and Lee, Kai-Zhan and Bengio, Yoshua and Bareinboim, Elias},
  title     = {The Causal-Neural Connection: Expressiveness, Learnability, and Inference},
  booktitle = {Advances in Neural Information Processing Systems 34 (NeurIPS 2021)},
  year      = {2021},
  url       = {https://proceedings.neurips.cc/paper/2021/hash/5989add1703e4b0480f75e2390739f34-Abstract.html}
}

@inproceedings{xia2022neural,
  author    = {Xia, Kevin and Pan, Yushu and Bareinboim, Elias},
  title     = {Neural Causal Models for Counterfactual Identification and Estimation},
  booktitle = {The Eleventh International Conference on Learning Representations (ICLR)},
  year      = {2023},
  url       = {https://openreview.net/forum?id=vouQcZS8KfW}
}

@inproceedings{nasr2023counterfactual,
  author    = {Nasr-Esfahany, Arash and Alizadeh, Mohammad and Shah, Devavrat},
  title     = {Counterfactual Identifiability of Bijective Causal Models},
  booktitle = {Proceedings of the 40th International Conference on Machine Learning (ICML)},
  series    = {Proceedings of Machine Learning Research},
  volume    = {202},
  pages     = {25733--25754},
  publisher = {PMLR},
  year      = {2023},
  url       = {https://proceedings.mlr.press/v202/nasr-esfahany23a.html}
}

@inproceedings{padh2023stochastic,
  author    = {Padh, Kirtan and Zeitler, Jakob and Watson, David and Kusner, Matt and Silva, Ricardo and Kilbertus, Niki},
  title     = {Stochastic Causal Programming for Bounding Treatment Effects},
  booktitle = {Proceedings of the Second Conference on Causal Learning and Reasoning (CLeaR)},
  series    = {Proceedings of Machine Learning Research},
  volume    = {213},
  pages     = {142--176},
  publisher = {PMLR},
  year      = {2023},
  url       = {https://proceedings.mlr.press/v213/padh23a.html}
}

@inproceedings{pan2024counterfactualimageediting,
  author    = {Pan, Yushu and Bareinboim, Elias},
  title     = {Counterfactual Image Editing},
  booktitle = {Proceedings of the 41st International Conference on Machine Learning (ICML)},
  series    = {Proceedings of Machine Learning Research},
  volume    = {235},
  pages     = {39087--39101},
  publisher = {PMLR},
  year      = {2024},
  url       = {https://proceedings.mlr.press/v235/pan24a.html}
}

@inproceedings{kocaoglu2018causalgan,
  author    = {Kocaoglu, Murat and Snyder, Christopher and Dimakis, Alexandros G. and Vishwanath, Sriram},
  title     = {{CausalGAN}: Learning Causal Implicit Generative Models with Adversarial Training},
  booktitle = {International Conference on Learning Representations (ICLR)},
  year      = {2018},
  url       = {https://openreview.net/forum?id=BJE-4xW0W}
}

@inproceedings{yang2021causalvae,
  author    = {Yang, Mengyue and Liu, Furui and Chen, Zhitang and Shen, Xinwei and Hao, Jianye and Wang, Jun},
  title     = {{CausalVAE}: Disentangled Representation Learning via Neural Structural Causal Models},
  booktitle = {Proceedings of the IEEE/CVF Conference on Computer Vision and Pattern Recognition (CVPR)},
  pages     = {9593--9602},
  year      = {2021},
  doi       = {10.1109/CVPR46437.2021.00947}
}

@inproceedings{sanchez2022diffscm,
  author    = {Sanchez, Pedro and Tsaftaris, Sotirios A.},
  title     = {Diffusion Causal Models for Counterfactual Estimation},
  booktitle = {Proceedings of the First Conference on Causal Learning and Reasoning (CLeaR)},
  series    = {Proceedings of Machine Learning Research},
  volume    = {177},
  pages     = {647--668},
  publisher = {PMLR},
  year      = {2022},
  url       = {https://proceedings.mlr.press/v177/sanchez22a.html}
}

@inproceedings{monteiro2023deepcounterfactuals,
  author    = {Monteiro, Miguel and De Sousa Ribeiro, Fabio and Pawlowski, Nick and Castro, Daniel C. and Glocker, Ben},
  title     = {Measuring Axiomatic Soundness of Counterfactual Image Models},
  booktitle = {The Eleventh International Conference on Learning Representations (ICLR)},
  year      = {2023},
  url       = {https://openreview.net/forum?id=lZOUQQvwI3q}
}

@inproceedings{javaloy2023normalizing,
  author    = {Javaloy, Adri{\'a}n and S{\'a}nchez-Mart{\'\i}n, Pablo and Valera, Isabel},
  title     = {Causal Normalizing Flows: From Theory to Practice},
  booktitle = {Advances in Neural Information Processing Systems 36 (NeurIPS 2023)},
  pages     = {58833--58864},
  year      = {2023},
  doi       = {10.52202/075280-2566}
}

@inproceedings{khemakhem2021causalautoregressive,
  author    = {Khemakhem, Ilyes and Monti, Ricardo and Leech, Robert and Hyv{\"a}rinen, Aapo},
  title     = {Causal Autoregressive Flows},
  booktitle = {Proceedings of the 24th International Conference on Artificial Intelligence and Statistics (AISTATS)},
  series    = {Proceedings of Machine Learning Research},
  volume    = {130},
  pages     = {3520--3528},
  publisher = {PMLR},
  year      = {2021},
  url       = {https://proceedings.mlr.press/v130/khemakhem21a.html}
}

@inproceedings{pawlowski2020deepscm,
  author    = {Pawlowski, Nick and Castro, Daniel C. and Glocker, Ben},
  title     = {Deep Structural Causal Models for Tractable Counterfactual Inference},
  booktitle = {Advances in Neural Information Processing Systems (NeurIPS)},
  volume    = {33},
  pages     = {857--869},
  year      = {2020},
  url       = {https://proceedings.neurips.cc/paper/2020/hash/0987b8b338d6c90bbedd8631bc499221-Abstract.html}
}

@book{rosenbaum2002sensitivity,
  author    = {Rosenbaum, Paul R.},
  title     = {Observational Studies},
  edition   = {2},
  publisher = {Springer},
  address   = {New York},
  year      = {2002},
  doi       = {10.1007/978-1-4757-3692-2}
}

@article{vanderweele2017evalue,
  author    = {VanderWeele, Tyler J. and Ding, Peng},
  title     = {Sensitivity Analysis in Observational Research: Introducing the {E}-value},
  journal   = {Annals of Internal Medicine},
  volume    = {167},
  number    = {4},
  pages     = {268--274},
  year      = {2017},
  doi       = {10.7326/M16-2607}
}

@inproceedings{hartford2017deepiv,
  author    = {Hartford, Jason and Lewis, Greg and Leyton-Brown, Kevin and Taddy, Matt},
  title     = {Deep {IV}: A Flexible Approach for Counterfactual Prediction},
  booktitle = {Proceedings of the 34th International Conference on Machine Learning (ICML)},
  series    = {Proceedings of Machine Learning Research},
  volume    = {70},
  pages     = {1414--1423},
  publisher = {PMLR},
  year      = {2017},
  url       = {https://proceedings.mlr.press/v70/hartford17a.html}
}

@inproceedings{xu2021dfiv,
  author    = {Xu, Liyuan and Chen, Yutian and Srinivasan, Siddarth and de Freitas, Nando and Doucet, Arnaud and Gretton, Arthur},
  title     = {Learning Deep Features in Instrumental Variable Regression},
  booktitle = {International Conference on Learning Representations (ICLR)},
  year      = {2021},
  url       = {https://openreview.net/forum?id=sy4Kg_ZQmS7}
}

@article{tchetgen2024proximal,
  author    = {Tchetgen Tchetgen, Eric J. and Ying, Andrew and Cui, Yifan and Shi, Xu and Miao, Wang},
  title     = {An Introduction to Proximal Causal Inference},
  journal   = {Statistical Science},
  volume    = {39},
  number    = {3},
  pages     = {375--390},
  year      = {2024},
  doi       = {10.1214/23-STS911}
}

@article{miao2018proxy,
  author    = {Miao, Wang and Geng, Zhi and Tchetgen Tchetgen, Eric J.},
  title     = {Identifying Causal Effects with Proxy Variables of an Unmeasured Confounder},
  journal   = {Biometrika},
  volume    = {105},
  number    = {4},
  pages     = {987--993},
  year      = {2018},
  publisher = {Oxford University Press},
  doi       = {10.1093/biomet/asy038}
}

@inproceedings{liu2015celeba,
  author    = {Liu, Ziwei and Luo, Ping and Wang, Xiaogang and Tang, Xiaoou},
  title     = {Deep Learning Face Attributes in the Wild},
  booktitle = {Proceedings of the IEEE International Conference on Computer Vision (ICCV)},
  month     = {December},
  year      = {2015}
}

@article{johnson2024mimiccxr,
  author  = {Johnson, Alistair and Lungren, Matthew and Peng, Yifan and Lu, Zhiyong and Mark, Roger and Berkowitz, Seth and Horng, Steven},
  title   = {{MIMIC-CXR-JPG} --- Chest Radiographs with Structured Labels},
  journal = {PhysioNet},
  year    = {2024},
  month   = {March},
  note    = {Version 2.1.0},
  doi     = {10.13026/jsn5-t979},
  url     = {https://doi.org/10.13026/jsn5-t979}
}

@misc{johnson2019mimic,
      title         = {{MIMIC-CXR-JPG}, a Large Publicly Available Database of Labeled Chest Radiographs}, 
      author        = {Alistair E. W. Johnson and Tom J. Pollard and Nathaniel R. Greenbaum and Matthew P. Lungren and Chih-ying Deng and Yifan Peng and Zhiyong Lu and Roger G. Mark and Seth J. Berkowitz and Steven Horng},
      year          = {2019},
      eprint        = {1901.07042},
      archivePrefix = {arXiv},
      primaryClass  = {cs.CV},
      url           = {https://arxiv.org/abs/1901.07042}
}

@article{PhysioNet,
  author  = {Goldberger, A. L. and Amaral, L. A. N. and Glass, L. and
             Hausdorff, J. M. and Ivanov, P. Ch. and Mark, R. G. and
             Mietus, J. E. and Moody, G. B. and Peng, C.-K. and Stanley, H. E.},
  title   = {{PhysioBank, PhysioToolkit, and PhysioNet}: Components of a New
             Research Resource for Complex Physiologic Signals},
  journal = {Circulation},
  year    = {2000},
  month   = {June},
  volume  = {101},
  number  = {23},
  pages   = {e215--e220},
  doi     = {10.1161/01.CIR.101.23.e215},
  note    = {PMID: 10851218}
}

@book{villani2009optimal,
  author    = {Villani, C{\'e}dric},
  title     = {Optimal Transport: Old and New},
  series    = {Grundlehren der mathematischen Wissenschaften},
  volume    = {338},
  publisher = {Springer},
  address   = {Berlin, Heidelberg},
  year      = {2009},
  doi       = {10.1007/978-3-540-71050-9},
  isbn      = {978-3-540-71049-3}
}

@misc{loshchilov2019decoupledweightdecayregularization,
      title         = {Decoupled Weight Decay Regularization}, 
      author        = {Ilya Loshchilov and Frank Hutter},
      year          = {2019},
      eprint        = {1711.05101},
      archivePrefix = {arXiv},
      primaryClass  = {cs.LG},
      url           = {https://arxiv.org/abs/1711.05101}
}

@inproceedings{tan2024consistency,
  author    = {Tan, Jiyuan and Blanchet, Jose and Syrgkanis, Vasilis},
  title     = {Consistency of Neural Causal Partial Identification},
  booktitle = {Advances in Neural Information Processing Systems},
  volume    = {37},
  year      = {2024},
  url       = {https://proceedings.neurips.cc/paper_files/paper/2024/hash/7f9220f90cc85b0da693643add6618e6-Abstract-Conference.html}
}

\appendix

\section{Related Work}\label{app:related}

\textbf{Causal generative models and counterfactual evaluation.}
Causal generative models incorporate structural assumptions into image synthesis and editing.
CausalGAN combines a generator over attributes, structured by a supplied causal graph, with conditional image generation~\citep{kocaoglu2018causalgan}.
CausalVAE learns causal relations among semantic factors with label supervision and a linear causal layer driven by independent Gaussian exogenous variables~\citep{yang2021causalvae}.
DeepSCM and Diff-SCM enable counterfactual image generation through inference of exogenous noise and diffusion-based mechanisms, respectively~\citep{pawlowski2020deepscm,sanchez2022diffscm}.
Causal autoregressive and normalizing flows exploit invertible mechanisms for causal inference~\citep{khemakhem2021causalautoregressive,javaloy2023normalizing}, while bijective causal models establish counterfactual identifiability for specific structures that include hidden confounding~\citep{nasr2023counterfactual}.
Complementing these modeling approaches, \citet{monteiro2023deepcounterfactuals} evaluate counterfactual image models through composition, reversibility, and effectiveness.
Together, these works develop mechanisms and criteria for causal generation under distinct structural and identification assumptions.
Our focus is on sampling a family of candidate interventional image distributions $P_x(I)$ under hidden confounding; this distributional target does not require identifying individual-level counterfactual images.

\textbf{Causal representation learning.}
Causal representation learning seeks to recover high-level causal variables from observations~\citep{scholkopf2021crl}.
Auxiliary observations can make latent-variable models identifiable under suitable conditional-prior and model assumptions~\citep{khemakhem2020vae}, whereas unsupervised disentanglement is generally impossible without inductive biases on the model and data~\citep{locatello2019disentanglement}.
\cauvade{} uses latent representations to model incompletely labeled semantic attributes and generate candidate intervention responses.
Its approximation result concerns observational and interventional image distributions and does not assert recovery of the true semantic coordinates.

\textbf{Augmented neural causal models.}
Augmented neural causal models (ANCMs) extend neural causal modeling to images~\citep{pan2024counterfactualimageediting}, building on methods for causal identification and estimation with neural SCMs~\citep{xia2021causal,xia2022neural}.
ANCMs explicitly accommodate hidden confounding and non-identifiability, targeting counterfactual consistency with respect to selected semantic features.
Their training procedure uses a supplied causal graph and labels for the modeled semantic factors, which have finite discrete domains; additional image factors may remain unobserved.
This detailed semantic graph specifies how an edit to one attribute can affect others, supporting counterfactual image editing conditioned on a given source image.
\cauvade{} builds on this augmented SCM framework, allowing semantic variables with continuous domains and incomplete labels for non-treatment factors.
Our focus is the simpler task of deconfounded interventional generation: setting $X=x$ breaks its dependence on hidden confounders while preserving their effects on other variables.
We target the population-level distribution $P_x(I)$, rather than an individual counterfactual edit conditioned on a source image.
We group semantic variables relative to the treatment and approximate exogenous confounding with a finite-support latent variable under the regularity conditions of \cref{thm:approx}.
The causal bootstrap then varies the cluster-posterior entropy weight to obtain multiple candidate interventional image distributions.
Empirically, ANCM behaves similarly to \cauvade{} with $\gamma=0$, with both producing interventional estimates that concentrate tightly across random seeds (see the ablation study in \cref{sec:ablation-gamma-zero}).

\textbf{Causal inference under hidden confounding.}
Sensitivity analysis characterizes how causal conclusions depend on assumptions about unmeasured confounding~\citep{rosenbaum2002sensitivity,vanderweele2017evalue}.
Instrumental-variable methods exploit external sources of treatment variation~\citep{hartford2017deepiv,xu2021dfiv}, while proxy-variable methods use measurements related to the hidden confounder~\citep{tchetgen2024proximal,miao2018proxy}.
Identification in these settings depends on assumptions about the auxiliary variables and causal mechanisms, including rank or completeness conditions where required.
Our entropy sweep explores variation in fitted causal predictions without requiring instruments or proxies.
The weight $\gamma$ controls cluster-posterior uncertainty and has no calibrated interpretation as the strength of hidden confounding.

\textbf{Partial identification.}
Partial identification characterizes the causal conclusions compatible with available data and assumptions.
Foundational results include nonparametric treatment-effect bounds~\citep{manski1990nonparametric}, sharp bounds under instrumental-variable restrictions~\citep{balke:pea97}, and bounds under monotone treatment response~\citep{manski1998monotone}.
Principal stratification provides a complementary framework for defining effects within strata determined by potential post-treatment values~\citep{frangakis:rub02}.
For finite observed domains, \citet{zhang2022partial} show that canonical SCMs with finite exogenous support preserve counterfactual distributions and reduce sharp-bound computation to polynomial optimization; \citet{duarte2024automated} develop automated bounding procedures for discrete settings.
Our contribution combines a representation guarantee for observational and interventional image distributions with a procedure for generating diverse candidate models.
The representation theorem allows continuous semantic mechanisms under its stated regularity conditions, and the entropy sweep explores intervention responses while retaining candidates that satisfy observational checks.
This procedure does not guarantee sharp bounds or exhaustive coverage of the compatible distributions.

\textbf{Optimization-based causal bounds and image generation.}
Neural causal models minimize and maximize specified causal queries over models constrained by the data and causal graph~\citep{xia2021causal,xia2022neural}.
Stochastic causal programming similarly optimizes causal effects over distributions of response functions subject to data-fitting constraints, including settings with multivariate continuous treatments~\citep{padh2023stochastic}.
These approaches do not directly solve our interventional image-generation task for two reasons.
First, their optimization requires a specified target functional, whereas our target is a family of image distributions $P_x(I)$.
There is no canonical scalar objective whose minimum and maximum characterize this entire family: choosing an image statistic, such as an attribute score, permits bounding its expectation but does not recover the diversity of compatible image distributions.
Second, neural parameterizations and response-function models can yield nonconvex objectives and constraints, so local optimization can miss feasible extrema even in structured-data settings and return overly narrow estimated interventional ranges.
\citet{padh2023stochastic} explicitly acknowledge this risk, although they do not report consistent empirical evidence of becoming trapped in poor local optima.
This limitation of local solvers should be distinguished from discrete polynomial approaches~\citep{zhang2022partial,duarte2024automated}: global bounding procedures such as that of \citet{duarte2024automated} provide certified bounds, but enumerating discrete response types for image-valued variables is prohibitive.
\cauvade{} instead varies a cluster-posterior entropy weight to explore candidate interventional image distributions without prescribing a scalar image query to extremize.
Our goal is to generate diverse interventional image samples compatible with the observed data, rather than to recover exhaustive causal bounds; entropy regularization supports this exploration by varying the fitted models while retaining those that pass observational checks.

\section{Derivation of the ELBO}\label{app:elbo}

We derive \cref{eq:elbo} for the generative model in \cref{eq:cauvade-gen}. As in \cref{sec:cauvade}, the data consists of observed pairs $\mathcal D=\{(x_i,I_i)\}_{i=1}^N$, with only $X$ labeled. The latent vector $\mathbf Z=(Z,\mathcal E_Z,\mathcal E_X,\mathcal E_Y,\mathcal E_I)$ collects the pre-treatment attribute and the four independent exogenous noises, and the discrete cluster $C$ represents the canonical confounder $U$. The encoder has the factorization
\[
q_\phi(\mathbf z,c\mid I)
=q_\phi(\mathbf z\mid I)\,q_\phi(c\mid\mathbf z).
\]

We first specify the likelihood to which the variational bound is applied. The target is the interventional log-evidence $\log P_x^\theta(I)$, evaluated at the observed treatment value $x=X$. Since the post-treatment attribute $Y$ is unobserved, we integrate it out and write
\begin{equation}\label{eq:Y-marginal}
p_\theta(I\mid X,\mathbf Z,C)
\;\triangleq\;\int p_\theta(Y\mid X,\mathbf Z,C)\,
p_\theta(I\mid X,Y,\mathbf Z,C)\,dY.
\end{equation}
Here the integral is part of the definition of the image likelihood. If instead one samples $Y$ and averages the conditional log-likelihood, Jensen's inequality gives a further lower bound on the logarithm of this integral.

\begin{proof}[Derivation of \cref{eq:elbo}]
Let
\[
r_\theta(\mathbf z,c)=p_\theta(\mathbf z\mid c)\,P(c)
\]
denote the joint latent prior. Under $\mathrm{do}(X=x)$, the latent variables are drawn from this prior and the treatment is fixed to $x$. Thus, at $x=X$,
\[
P_x^\theta(I)
=\int\sum_{c=1}^K
p_\theta(I\mid X,\mathbf z,c)\,r_\theta(\mathbf z,c)\,d\mathbf z.
\]
Following the variational argument of \citet{kingma2022vae}, we introduce an expectation under the encoder. With the density ratio taken on its support, this gives
\[
\log P_x^\theta(I)
\ge\log\mathbb E_{q_\phi(\mathbf z,c\mid I)}\!\left[
\frac{p_\theta(I\mid X,\mathbf z,c)\,r_\theta(\mathbf z,c)}
{q_\phi(\mathbf z,c\mid I)}
\right].
\]
This step is an equality when the encoder covers the support of the integrand. Since $\log$ is concave, Jensen's inequality yields
\[
\log P_x^\theta(I)
\ge\mathbb E_{q_\phi(\mathbf z,c\mid I)}\!\left[
\log p_\theta(I\mid X,\mathbf z,c)
+\log\frac{r_\theta(\mathbf z,c)}{q_\phi(\mathbf z,c\mid I)}
\right].
\]
By definition of the KL divergence,
\[
\mathbb E_{q_\phi}\!\left[
\log\frac{r_\theta(\mathbf Z,C)}{q_\phi(\mathbf Z,C\mid I)}
\right]
=-D_{\mathrm{KL}}\bigl(q_\phi(\mathbf Z,C\mid I)\,\|\,r_\theta(\mathbf Z,C)\bigr).
\]
Therefore,
\[
\log P_x^\theta(I)
\ge\mathbb E_{q_\phi}\bigl[\log p_\theta(I\mid X,\mathbf Z,C)\bigr]
-D_{\mathrm{KL}}\bigl(q_\phi(\mathbf Z,C\mid I)\,\|\,r_\theta(\mathbf Z,C)\bigr).
\]
Substituting $r_\theta(\mathbf z,c)=p_\theta(\mathbf z\mid c)\,P(c)$ gives \cref{eq:elbo}. The bound is attained when the encoder equals the posterior under the intervened model,
\[
q_\phi(\mathbf z,c\mid I)
=p_\theta(\mathbf z,c\mid I,\mathrm{do}(X=x)).
\]
This proves the stated variational bound.
\end{proof}

We next expand the KL term using the encoder and prior factorizations. By definition,
\[
\mathcal L_{\mathrm{KL}}
=\mathbb E_{q_\phi(\mathbf z\mid I)}\!\left[
\sum_{c=1}^K q_\phi(c\mid\mathbf z)
\log\frac{q_\phi(\mathbf z\mid I)\,q_\phi(c\mid\mathbf z)}
{p_\theta(\mathbf z\mid c)\,P(c)}
\right].
\]
For each $(\mathbf z,c)$, the logarithm separates as
\[
\log\frac{q_\phi(\mathbf z\mid I)\,q_\phi(c\mid\mathbf z)}
{p_\theta(\mathbf z\mid c)\,P(c)}
=\log q_\phi(\mathbf z\mid I)-\log p_\theta(\mathbf z\mid c)
+\log\frac{q_\phi(c\mid\mathbf z)}{P(c)}.
\]
The sum of the last term over $c$ is the KL divergence between the cluster posterior and the categorical prior. Consequently,
\begin{align}\label{eq:kl-decomposed}
\mathcal L_{\mathrm{KL}}
&=\mathbb E_{q_\phi(\mathbf z\mid I)}\!\left[\sum_{c=1}^{K}q_\phi(c\mid\mathbf z)\bigl(\log q_\phi(\mathbf z\mid I)-\log p_\theta(\mathbf z\mid c)\bigr)\right]\\
&\quad+\mathbb E_{q_\phi(\mathbf z\mid I)}\!\bigl[D_{\mathrm{KL}}\bigl(q_\phi(C\mid\mathbf z)\,\|\,P(C)\bigr)\bigr].\nonumber
\end{align}
The Gaussian encoder density $q_\phi(\mathbf z\mid I)=\mathcal N(\mu_\phi(I),\sigma_\phi^2(I))$ and the Gaussian prior density $p_\theta(\mathbf z\mid c)=\mathcal N(\mathbf z;\mu_c,\sigma_c^2)$ can both be evaluated in closed form. For a given $\mathbf z$, the categorical KL term is also a finite sum. Thus the sum over $C$ can be evaluated analytically, while the remaining expectation over $\mathbf Z$ is estimated as described in \cref{app:algorithm}.

The encoder takes only $I$ as input, whereas the posterior that makes the bound tight is conditioned on both $I$ and the intervention. This architectural restriction can therefore loosen the bound. It does not change the intervention sampling procedure: generation draws from the latent prior, so allowing the encoder to depend on $X$ would also be compatible with that procedure.

To distinguish this derivation from observational conditioning, consider the observational joint distribution. In that model, $X$ is generated from its mechanism, giving
\begin{equation}\label{eq:obs-joint}
p_\theta(x,I)=\sum_c\int r_\theta(\mathbf z,c)\,
p_\theta(x\mid\mathbf z,c)\,p_\theta(I\mid x,\mathbf z,c)\,d\mathbf z.
\end{equation}
Conditioning on the observed treatment therefore weights the latent variables by $p_\theta(\mathbf z,c\mid X=x)$, whereas intervention leaves their prior $r_\theta(\mathbf z,c)$ unchanged. A variational bound on the observational joint log-likelihood would also contain the generative treatment log-likelihood. The auxiliary classifier loss $\mathcal L_X$ in \cref{sec:cauvade} does not establish such an interpretation without a specified relationship to the treatment mechanism and objective weights. Accordingly, \cref{eq:objective} remains a training surrogate, and observational consistency is assessed separately using the constraints associated with \cref{cor:admissible-move}.

The two auxiliary terms have distinct roles. The treatment-consistency term encourages the latent variables to predict $X$, although this alone does not ensure that the image decoder uses $X$. The entropy term encourages uncertainty in the cluster posterior. To conclude that changing its weight changes interventional predictions, the analysis additionally uses the distributional response in (A5).

Finally, suppose $X$ and $Y$ are absent and the dataset consists only of images $\{I_i\}_{i=1}^N$. Then the reconstruction term reduces to $\mathbb E_{q_\phi}\bigl[\log p_\theta(I\mid\mathbf Z,C)\bigr]$,
and the KL term retains \cref{eq:kl-decomposed} with the same Gaussian-mixture prior $r_\theta(\mathbf z,c)=P(c)\,p_\theta(\mathbf z\mid c)$. Hence \cref{eq:elbo} reduces to the VaDE objective~\citep{jiang2017vade}. In \cauvade{}, the treatment additionally conditions the decoder, the post-treatment attribute is marginalized as in \cref{eq:Y-marginal}, and the auxiliary terms $\mathcal L_X$ and $H$ encourage treatment information and generate candidates for observational acceptance.

\section{Training Algorithm}\label{app:algorithm}

This appendix gives pseudocode for training \cauvade{} at a single value of the entropy weight \(\gamma\) (\cref{alg:cauvade-train}) and for training across entropy weights to explore candidate interventional distributions (\cref{alg:cauvade-sweep}). The notation matches \cref{sec:cauvade}: \(\mathcal D=\{(x_i,I_i)\}_{i=1}^N\) is the confounded training set, \(K\) is the number of mixture components, and \((\theta,\phi,\psi)\) collect the decoder, encoder, and classifier-head parameters. As in \cref{sec:cauvade}, \(\mathbf{Z}\) is the augmented continuous latent and \(C\) is the discrete cluster. The GMM parameters \(\{P(c),\mu_c,\sigma_c^2\}_{c=1}^K\) are part of \(\theta\) but are initialized separately by the VaDE-style pretraining phase \citep{jiang2017vade}.

Two design choices keep training tractable. First, gradients with respect to \(\phi\) flow through \(\mathbf{Z}\) via the standard Gaussian reparametrization \(\mathbf{z}=\mu_\phi(I)+\sigma_\phi(I)\odot\epsilon\) with \(\epsilon\sim\mathcal N(0,\mathbf I)\). Second, we marginalize the discrete cluster \(C\) analytically rather than sampling, following VaDE \citep{jiang2017vade}: the categorical posterior \(q_\phi(c\mid\mathbf{z})\) is bounded (\(K\) values) and computable in closed form, so every loss term that involves \(C\) is written as a finite sum
\[
\mathbb E_{q_\phi(C\mid\mathbf{z})}[g(C)]\;=\;\sum_{c=1}^{K}q_\phi(c\mid\mathbf{z})\,g(c),
\]
and gradients propagate through the probabilities \(q_\phi(c\mid\mathbf{z})\) directly. For the reconstruction, KL, and treatment terms, the reconstruction term becomes \(\sum_c q_\phi(c\mid\mathbf{z})\log p_\theta(I\mid X,\mathbf{z},c)\), and the KL term decomposes into expressions involving \(q_\phi(c\mid\mathbf{z})\), \(P(c)\), and the cluster-conditional Gaussian parameters \((\mu_c,\sigma_c^2)\). For the entropy $H$ defined in \cref{sec:elbo-aug}, average $q_\phi(c\mid\mathbf z)$ over reparametrized latent draws before taking the categorical entropy; averaging the entropies instead gives a different regularizer. Hard categorical samples of \(C\) are used at generation time.

Pretraining initializes \((\mu_c,\sigma_c^2)\) by fitting a Gaussian mixture model to the embeddings produced by a vanilla VAE pretrained for a few epochs on \(\mathcal D\) with \(-\mathcal L_{\mathrm{rec}}+\mathcal L_{\mathrm{KL}}\) alone (i.e., \(\lambda=\gamma=0\), \(K=1\) in the pretraining stage). This avoids the well-known cluster-collapse issue of jointly training a Gaussian-mixture VAE from scratch and provides a common initialization across the evaluated values of the entropy weight \(\gamma\), so that differences in the learned cluster posterior reflect the regularizer rather than the initialization.

\begin{algorithm}[t]
\caption{\cauvade{} training at fixed $\gamma$}
\label{alg:cauvade-train}
\begin{algorithmic}[1]
\Require Dataset $\mathcal D=\{(x_i,I_i)\}_{i=1}^N$; cluster count $K$; weights $(\alpha,\beta,\lambda,\gamma)$; learning rate $\eta_{\mathrm{lr}}$; number of epochs $T$; batch size $B$; pretrained GMM parameters $\{\mu_c^{(0)},\sigma_c^{2,(0)},P^{(0)}(c)\}_{c=1}^K$.
\Ensure Trained parameters $(\theta,\phi,\psi)$.
\State Initialize encoder $\phi$, decoders $\theta\supseteq\{\mu_c,\sigma_c^2,P(c)\}_{c=1}^K$, classifier head $\psi$; load pretrained GMM parameters into $\theta$.
\For{epoch $=1,\dots,T$}
  \For{minibatch $\mathcal B\subset\mathcal D$ of size $B$}
    \State Encode $(\mu_\phi(I),\sigma_\phi(I))\gets\mathrm{Enc}_\phi(I)$ for each $I\in\mathcal B$.
    \State Sample $\mathbf{z}=\mu_\phi(I)+\sigma_\phi(I)\odot\epsilon$, $\epsilon\sim\mathcal N(0,\mathbf I)$. \Comment{Gaussian reparam.}
    \State Compute cluster posterior $q_\phi(c\mid\mathbf{z})\in\Delta^{K-1}$ for $c=1,\dots,K$. \Comment{Closed form.}
    \State Compute per-cluster decoder log-likelihoods $\ell_c\gets\log p_\theta(I\mid X,\mathbf{z},c)$ for $c=1,\dots,K$. \Comment{$K$ decoder passes.}
    \State Compute $\mathcal L_{\mathrm{rec}},\mathcal L_{\mathrm{KL}}$ from \cref{eq:elbo} and $\mathcal L_X,H$ as defined in \cref{sec:cauvade,sec:elbo-aug}, respectively, on $\mathcal B$, with $\mathbb E_{q_\phi(C\mid\mathbf{z})}[\,\cdot\,]=\sum_c q_\phi(c\mid\mathbf{z})\,[\,\cdot\,]$. \Comment{Finite sums over $C$; marginalize $\mathbf Z$ for $H$.}
    \State $\mathcal L\gets-\alpha\,\mathcal L_{\mathrm{rec}}+\beta\,\mathcal L_{\mathrm{KL}}+\lambda\,\mathcal L_X-\gamma\,H$. \Comment{\cref{eq:objective}.}
    \State $(\theta,\phi,\psi)\gets(\theta,\phi,\psi)-\eta_{\mathrm{lr}}\nabla_{(\theta,\phi,\psi)}\mathcal L$.
  \EndFor
\EndFor
\State \Return $(\theta,\phi,\psi)$.
\end{algorithmic}
\end{algorithm}

\begin{algorithm}[t]
\caption{\cauvade{} training across entropy weights}
\label{alg:cauvade-sweep}
\begin{algorithmic}[1]
\Require Dataset $\mathcal D$; cluster count $K$; entropy-weight grid $\Gamma=\{\gamma_1,\dots,\gamma_M\}$; weights $(\alpha,\beta,\lambda)$; target intervention values $\{x^{(1)},\dots,x^{(L)}\}\subset\mathcal D_X$; sample size $S$.
\Ensure Family of interventional sample sets $\bigl\{\mathcal S_x^{(m)}\bigr\}_{m=1,\dots,M;\,x\in\{x^{(1)},\dots,x^{(L)}\}}$.
\State Pretrain GMM parameters $\{\mu_c^{(0)},\sigma_c^{2,(0)},P^{(0)}(c)\}_{c=1}^K$ on $\mathcal D$ as described above.
\For{$m=1,\dots,M$}
  \State $(\theta_m,\phi_m,\psi_m)\gets$ \Call{TrainCauVaDE}{$\mathcal D,K,\alpha,\beta,\lambda,\gamma_m,\{\mu_c^{(0)},\sigma_c^{2,(0)},P^{(0)}(c)\}$}. \Comment{\cref{alg:cauvade-train}.}
  \For{$x\in\{x^{(1)},\dots,x^{(L)}\}$}
    \State $\mathcal S_x^{(m)}\gets\emptyset$.
    \For{$s=1,\dots,S$}
      \State Sample $c\sim P_{\theta_m}(C)$.
      \State Sample $\mathbf{z}\sim p_{\theta_m}(\mathbf{Z}\mid c)=\mathcal N(\mu_c,\sigma_c^2\mathbf I)$.
      \State Sample $y\sim p_{\theta_m}(Y\mid x,\mathbf{z},c)$.
      \State Sample $I\sim p_{\theta_m}(I\mid x,y,\mathbf{z},c)$.
      \State $\mathcal S_x^{(m)}\gets\mathcal S_x^{(m)}\cup\{I\}$.
    \EndFor
  \EndFor
\EndFor
\State \Return $\bigl\{\mathcal S_x^{(m)}\bigr\}$.
\end{algorithmic}
\end{algorithm}

The sample sets $\mathcal S_x^{(m)}$ are empirical realizations of $P_x^{\theta_m}(I)$. \cref{lem:price,thm:move} describe local surrogate cost and interventional response under (A4)--(A5); independent training runs need not follow such a branch. The returned samples are candidates: observationally accepted models are selected using the observational constraints discussed before \cref{cor:admissible-move}, with the local guarantee in \cref{cor:admissible-move} requiring strict baseline slack.


\section{Proofs}\label{app:proofs}

We first show how to reduce an acyclic semantic model to the canonical diagram in \cref{app:diagram-reduction}. We then prove the approximation result in \cref{app:proof} and the entropy-weight results in \crefrange{app:price}{app:admissible-move}. The approximation proof uses (A1)--(A3) and the finite-dimensional exogenous domain specified in \cref{thm:approx}. It establishes approximation within the canonical structural class and does not assert universal approximation by a particular neural architecture. The entropy-weight proofs use (A4)--(A5); the admissibility result additionally requires strict baseline slack.

\subsection{Reduction to the Canonical Diagram}\label{app:diagram-reduction}

Consider an ASCM whose finitely many semantic variables form an acyclic directed graph, with a designated treatment node $X$, joint semantic exogenous vector $\*U$, and an image renderer with independent noise $\1E_I$. Let $Y$ collect the semantic descendants of $X$ and let $Z$ collect all other semantic variables except $X$. Either block may be empty, in which case its domain is a singleton. Assume the original structural assignments are well-defined for the intervention values under consideration.

\begin{proof}
We first construct the grouped semantic mechanisms. No variable in $Z$ can have a parent in $\{X\}\cup Y$, since such an edge would make that variable a descendant of $X$. We may therefore evaluate the equations for $Z$ in topological order and substitute them into one another. This gives a joint assignment
\[
Z=f_Z(\*U).
\]
Every semantic parent of $X$ belongs to $Z$. Hence the original treatment assignment can be written as
\[
X=f_X(Z,\*U),
\]
where $f_X$ ignores any coordinates of $Z$ that it does not use. Finally, evaluating the descendants of $X$ in topological order gives
\[
Y=f_Y(X,Z,\*U).
\]
We retain the original joint distribution of $\*U$. The resulting assignments have the block order and shared exogenous dependence allowed by \cref{fig:ascm_obs}. Some displayed edges may be unused, since the canonical diagram permits all dependencies needed by this construction.

We next verify preservation of the observational and interventional distributions. Reindex the original renderer as $f_I(X,Y,Z,\1E_I)$ and fix a realization of $(\*U,\1E_I)$. By construction, evaluating the grouped equations gives the same semantic values as evaluating the original equations. Applying the unchanged renderer therefore gives the same image. Since the exogenous distribution is unchanged, the grouped model has the same $P(X,I)$.

Under $\dox$, the treatment assignment is replaced by $X\gets x$. The equations for $Z$ remain unchanged, and substituting $x$ into the grouped descendant mechanism gives exactly the original descendant values under this intervention. The rendered image is therefore also unchanged for each exogenous realization. It follows that the grouped model preserves every $P_x(I)$ for the designated treatment.

Finally, each grouped mechanism is a finite composition of the original measurable mechanisms and is therefore measurable. The grouping preserves $Y_x(u)$ for every $x,u$ and only reindexes the semantic inputs of the renderer, so it preserves (A2)--(A3) whenever the original model satisfies them. The exogenous vector is unchanged, so its finite-dimensionality is also preserved.

This construction groups variables and composes their mechanisms; it does not marginalize semantic nodes or discretize the exogenous variables. The canonical class may allow additional dependencies, so the reduction need not preserve the smaller feasible set imposed by a detailed graph. It also does not identify the latent blocks from data or address cyclic models or arbitrary collections of treatment nodes. Subject to the stated domain conditions, the construction proves the claimed reduction; applying \cref{thm:approx} additionally requires all the assumptions of that theorem.
\end{proof}

\subsection{Proof of \cref{thm:approx}}\label{app:proof}

We first represent the original model by a joint distribution of treatment assignments, pre-treatment covariates, and response functions. We then approximate this distribution by finitely many response types while keeping the image mechanism unchanged.

Let $\rho_I$ and $\rho_X$ denote the metrics on the image domain $\D_I$ and treatment domain $\D_X$, respectively. To pass from the coupled images to their distributions, let $\rho$ denote the metric on the space supporting $P$ and $Q$, and let $W_1$ be the Wasserstein distance induced by $\rho$. For any jointly distributed random variables $A\sim P$ and $B\sim Q$, their joint law is a coupling of $P$ and $Q$, so
\[
W_1(P,Q)\le\mathbb E[\rho(A,B)].
\]
Here $\rho(A,B)$ measures the distance between the two coupled samples. For image distributions we use the image-domain metric $\rho_I$; for joint image--treatment distributions we use the sum metric $\rho_I(i,i')+\rho_X(x,x')$~\citep{villani2009optimal}.

\paragraph{Joint response distribution.}
Fix the ASCM $\1M$ in \cref{thm:approx}, and write $S\subseteq\mathbb R^m$ for the Borel domain of its shared exogenous vector $\*U$. For $u\in S$, define
\[
z(u)=f_Z(u),\qquad x^\natural(u)=f_X(z(u),u),\qquad
r_u(x)=Y_x(u)=f_Y(x,z(u),u),\quad x\in\D_X.
\]
Here $x^\natural(u)$ is the naturally assigned treatment, and $r_u(x)$ is the outcome obtained by setting the treatment to $x$ while keeping the exogenous realization fixed. Define
\[
T(u)=(x^\natural(u),z(u),r_u),\qquad
(X^\natural,Z,R)=T(\*U),\qquad Q=T_\#P_{\*U}.
\]
Thus $Q$ is the distribution of the treatment assignment, covariates, and entire response function induced by one exogenous draw. It can be written as
\[
Q(dx^\natural,dz,dr)=P_{X,Z}(dx^\natural,dz)\,P_{R\mid X,Z}(dr\mid x^\natural,z),
\]
where $P_{X,Z}$ is the original observational joint distribution of treatment and covariates. The standard Borel property needed for this conditional distribution is verified in the lemma below. For each fixed $x$, evaluating the response at $x$ gives the joint distribution $P_{X,Z}(dx^\natural,dz)P(Y_x\in dy\mid X=x^\natural,Z=z)$. Keeping the entire response $R$ ensures that the same draw specifies outcomes for every treatment. The candidate SCM supplies this joint response distribution; the construction requires neither its identification from observational data nor conditional independence among potential outcomes.

To see why this representation preserves the original model, sample $(x^\natural,z,r)\sim Q$ as a shared exogenous variable and set
\[
Z\gets z,\qquad X\gets x^\natural,\qquad Y\gets r(X).
\]
By the definitions of $x^\natural$ and $r_u$, the observational outcome $r_{\*U}(x^\natural(\*U))$ equals the original $Y$. Under $\doX{x}$, the same draw gives
\[
(X,Y,Z)=(x,r_{\*U}(x),z(\*U)),
\]
which is the original interventional semantic vector. Keeping $f_I$ and the independent rendering noise therefore preserves $P(X,I)$ and every $P_x(I)$ exactly.

The shared exogenous variable in this representation need not have finite support. For a finite treatment domain $\{x_1,\ldots,x_k\}$, its response component is the vector $(Y_{x_1},\ldots,Y_{x_k})$; for continuous treatments it is a continuous function. The following lemma constructs finite approximations of this representation.

\begin{lemma}[Finite approximation of the joint response distribution]\label{lm:cascm_approx_semantic}
Under the assumptions of \cref{thm:approx}, there is a sequence of CASCMs $\1N_n$ with finite confounders, the same image mechanism $f_I$, and the same image-noise distribution as $\1M$, admitting a coupling such that
\[
\*V^{\1N_n}\longrightarrow\*V^{\1M},
\qquad
\*V_x^{\1N_n}\longrightarrow\*V_x^{\1M}
\quad\text{for every }x\in\D_X,
\]
almost surely, where $\*V=(X,Y,Z)$. The same probability-one event works for all $x$.
\end{lemma}

\begin{proof}
We first specify a metric in which response types can be approximated. For two continuous responses, define
\[
d_\infty(r,r')=\sup_{x\in\D_X}\rho_Y(r(x),r'(x)).
\]
For response types, use the sum metric
\[
d\bigl((x,z,r),(x',z',r')\bigr)
=\rho_X(x,x')+\rho_Z(z,z')+d_\infty(r,r').
\]
Here $\rho_X,\rho_Y,\rho_Z$ denote the metrics on the semantic domains. For example, continuous scalar or vector attributes may use Euclidean distance, $\rho(a,b)=\|a-b\|_2$, while finite categorical attributes may use the discrete metric, $\rho(a,b)=\mathbf{1}\{a\ne b\}$. For attributes containing both continuous and categorical components, the component metrics may be summed. Since each response is continuous on the compact domain $\D_X$, the supremum in $d_\infty$ is finite.

We next verify measurability of $T$ in this metric. By (A2), each $r_u$ belongs to $C(\D_X,\D_Y)$, which is separable under $d_\infty$ because $\D_X$ is compact metric and $\D_Y$ is separable metric. Choose a countable dense subset $\{t_k\}_{k\ge1}$ of $\D_X$. For each fixed $r\in C(\D_X,\D_Y)$, continuity of the responses gives
\[
d_\infty(r_u,r)=\sup_{k\ge1}\rho_Y(r_u(t_k),r(t_k)).
\]
Each evaluation $u\mapsto r_u(t_k)$ is Borel measurable by measurability of the structural assignments. The displayed supremum is therefore measurable. A countable base of balls in the response space shows that $u\mapsto r_u$ is Borel measurable. The maps $z(u)$ and $x^\natural(u)$ are also Borel measurable, so $T$ is a Borel map into the separable metric space $\D_X\times\D_Z\times C(\D_X,\D_Y)$.

We now construct the finite approximation. Since $S$ is a Borel subset of a finite-dimensional Euclidean space, $P_{\*U}$ is a Radon probability measure on $S$. Lusin's theorem applied to the measurable map $T$ gives, for every $n$, a compact set $K_n\subseteq S$ such that $T|_{K_n}$ is continuous and
\[
P(\*U\notin K_n)\le2^{-n}.
\]
Continuity of the restriction $T|_{K_n}$ implies that $L_n=T(K_n)$ is compact. Moreover,
\[
Q(L_n^c)=P(T(\*U)\notin L_n)
\le P(\*U\notin K_n)\le2^{-n}.
\]
It follows that $Q$ is concentrated on $\bigcup_nL_n$. This countable union of compact metric spaces is a standard Borel space. The response coordinate has the same property after projection, which justifies the conditional distribution used above on a set of full probability. Only the compact sets $L_n$ are needed; the entire response-function space need not be compact.

By compactness, $L_n$ admits a finite Borel partition into nonempty cells $B_{n,1},\ldots,B_{n,d_n}$ of diameter at most $1/n$. Choose one representative from each cell,
\[
a_{n,j}=(x_{n,j},z_{n,j},r_{n,j})\in B_{n,j}.
\]
For the complement, set $B_{n,0}=L_n^c$ and choose $a_{n,0}=T(u_*)$ for a fixed $u_*\in S$. Define the cell index and the approximation map by
\[
J_n(a)=j\quad\text{when }a\in B_{n,j},\qquad
q_n(a)=a_{n,J_n(a)}.
\]
These maps are Borel measurable. Let the finite confounder $U_n\in\{0,\ldots,d_n\}$ have probabilities
\[
P(U_n=j)=Q(B_{n,j}),\qquad
Q_n=\sum_{j=0}^{d_n}Q(B_{n,j})\,\delta_{a_{n,j}}.
\]
Thus $Q_n$ is the distribution obtained by replacing each response type by its cell representative. Each mixture component retains the treatment assignment, covariates, and entire response together. Taking the inverse images of the cells under $T$ also gives a finite measurable partition of the original exogenous domain.

We use these representatives to define the semantic mechanisms of $\1N_n$:
\[
\begin{aligned}
f_Z^{(n)}(j,e_Z)&=z_{n,j},\\
f_X^{(n)}(z,j,e_X)&=x_{n,j},\\
f_Y^{(n)}(x,z,j,e_Y)&=r_{n,j}(x).
\end{aligned}
\]
Choose auxiliary noises that are mutually independent and independent of $U_n$ and $\mathcal E_I$. The displayed mechanisms ignore these auxiliary noises and the $z$ inputs. Since $U_n$ has the discrete topology and every $r_{n,j}$ is continuous, all three mechanisms are continuous. The canonical diagram allows unused parents and permits the dependence induced by the shared confounder.

Set $f_I^{(n)}=f_I$ and retain the original distribution of $\mathcal E_I$, independently of $U_n$. The observed domains are unchanged, so (A1) remains satisfied. Every constructed response $r_{n,j}$ is continuous, giving (A2), and the unchanged mixing function preserves (A3). After relabeling the finite confounder states, $\1N_n$ is therefore a CASCM. This construction is designed for interventions on $X$; it need not preserve interventions on $Z$ or restrictions imposed by a smaller causal diagram.

It remains to prove convergence under a common coupling. Draw $\*U$ from its original distribution and set $U_n=J_n(T(\*U))$. By construction, $P(U_n=j)=Q(B_{n,j})$, so this coupling has the required marginal distributions. Write
\[
(X_n^\natural,Z_n,R_n)=q_n(T(\*U)).
\]
Since $Q(L_n^c)\le2^{-n}$, we have $\sum_nQ(L_n^c)<\infty$. The Borel--Cantelli lemma therefore implies that $T(\*U)\in L_n$ for all sufficiently large $n$, almost surely. On this event, the response type and its representative belong to the same cell, so
\[
d\bigl((X_n^\natural,Z_n,R_n),(X^\natural,Z,R)\bigr)\le1/n
\quad\text{eventually}.
\]
Each term in the sum metric is nonnegative. Hence
\[
X_n^\natural\to X^\natural,\qquad Z_n\to Z,\qquad
d_\infty(R_n,R)\to0
\quad\text{almost surely}.
\]
For the observational outcome, the triangle inequality gives
\[
\rho_Y\bigl(R_n(X_n^\natural),R(X^\natural)\bigr)
\le d_\infty(R_n,R)
+\rho_Y\bigl(R(X_n^\natural),R(X^\natural)\bigr)
\longrightarrow0.
\]
The first term tends to zero by uniform convergence of the responses, and the second tends to zero by continuity of the sampled response $R$. Together with convergence of $X_n^\natural$ and $Z_n$, this proves $\*V^{\1N_n}\to\*V^{\1M}$ almost surely.

Under $\doX{x}$, the coupled semantic vectors are $(x,R_n(x),Z_n)$ and $(x,R(x),Z)$. For every $x\in\D_X$,
\[
\rho_Y(R_n(x),R(x))\le d_\infty(R_n,R)\longrightarrow0.
\]
This bound holds on the same probability-one event for all $x$. Combining it with $Z_n\to Z$ proves the interventional convergence and completes the proof.
\end{proof}

\begin{proof}[Proof of \cref{thm:approx}]
We first prove convergence of the observational distributions. Use the coupling from \cref{lm:cascm_approx_semantic}, and give both models the same draw of $\mathcal E_I$, independent of their semantic exogenous variables. By (A3), $f_I(\cdot,e)$ is continuous in its endogenous inputs for each fixed $e$. On the probability-one convergence event of the lemma, applying the unchanged mixing function with the shared rendering-noise draw gives
\[
I^{\1N_n}=f_I(\*V^{\1N_n},\mathcal E_I)
\longrightarrow f_I(\*V^{\1M},\mathcal E_I)=I^{\1M}
\quad\text{almost surely}.
\]
Equip $\D_I\times\D_X$ with the sum metric. Applying the coupling bound to the observational pairs yields
\[
W_1\bigl(P(I,X;\1N_n),P(I,X;\1M)\bigr)
\le\mathbb E\!\left[\rho_I(I^{\1N_n},I^{\1M})+
\rho_X(X_n^\natural,X^\natural)\right].
\]
Both distances inside the expectation tend to zero almost surely. By (A1), they are bounded by the finite diameters $D_I$ and $D_X$. Dominated convergence therefore gives
\[
W_1\bigl(P(I,X;\1N_n),P(I,X;\1M)\bigr)\longrightarrow0.
\]
The bound depends only on the image and treatment domains, so neither compactness of the rendering-noise domain nor continuity in the rendering noise is required.

We next prove convergence of the interventional distributions. Fix $x\in\D_X$. On the same probability-one event, (A3) and the semantic convergence give
\[
I_x^{\1N_n}=f_I(\*V_x^{\1N_n},\mathcal E_I)
\longrightarrow f_I(\*V_x^{\1M},\mathcal E_I)=I_x^{\1M}
\quad\text{almost surely}.
\]
Since the image distance is bounded by $D_I$, the coupling bound and dominated convergence imply
\[
W_1\bigl(P_x(I;\1N_n),P_x(I;\1M)\bigr)
\le\mathbb E\!\left[\rho_I(I_x^{\1N_n},I_x^{\1M})\right]\longrightarrow0.
\]
It remains to integrate this estimate over the treatment domain. The evaluation map $(x,r)\mapsto r(x)$ is continuous for the uniform response metric, and $T$, $q_n$, and $f_I$ are Borel measurable. Thus the coupled images are jointly measurable in the treatment and the exogenous draws. Integrating the coupling bound gives
\[
\int_{\D_X}W_1\bigl(P_x(I;\1N_n),P_x(I;\1M)\bigr)\,dx
\le\int_{\D_X}\mathbb E\!\left[\rho_I(I_x^{\1N_n},I_x^{\1M})\right]\,dx.
\]
The integrand on the right tends to zero and is bounded by $D_I$. For continuous treatments, the compact Euclidean treatment domain has finite Lebesgue measure, so this constant is integrable. For finite categorical treatments, the integral is a finite sum. Dominated convergence therefore yields
\[
\int_{\D_X}W_1\bigl(P_x(I;\1N_n),P_x(I;\1M)\bigr)\,dx
\longrightarrow0.
\]

Finally, let $\epsilon>0$. Observational convergence gives an index $n_{\mathrm{obs}}$ after which the first approximation error is below $\epsilon$, and integrated interventional convergence gives an index $n_{\mathrm{int}}$ after which the second is below $\epsilon$. Choose $n\ge\max\{n_{\mathrm{obs}},n_{\mathrm{int}}\}$ and set $\1N=\1N_n$. Then the same CASCM satisfies both inequalities in \cref{eq:approx-int}. This completes the proof.
\end{proof}

\subsection{Proof of \cref{lem:price}}\label{app:price}
Throughout these proofs, we use the notation $\eta=(\theta,\phi,\psi)$ and $D,H,A,b,v,\kappa_0$ from \cref{sec:elbo-aug}. The losses are averaged over a fixed data distribution or training sample. We study the local branch of optima in \cref{lem:price}; an arbitrary sequence of SGD outputs need not follow this branch.

\begin{proof}
We first construct the local branch of minimizers. Define
\[
F(\eta,\gamma)=\nabla_\eta D(\eta)-\gamma\nabla_\eta H(\eta).
\]
Thus $F(\eta,\gamma)=0$ is the stationarity condition for $D-\gamma H$. By (A4), $F$ is $C^1$ and
\[
F(\eta_0,0)=0,\qquad
\partial_\eta F(\eta_0,0)=A\succ0.
\]
Since $A$ is invertible, the implicit function theorem gives a unique nearby $C^1$ branch $\eta_\gamma$ satisfying $F(\eta_\gamma,\gamma)=0$ and passing through $\eta_0$.

It remains to verify that these stationary points are strict local minimizers. The Hessian of the regularized objective along the branch is
\[
\1H_\gamma=\nabla_\eta^2D(\eta_\gamma)-\gamma\nabla_\eta^2H(\eta_\gamma).
\]
At $\gamma=0$, this matrix equals $A\succ0$. By continuity, there is a sufficiently small $\bar\gamma>0$ such that $\1H_\gamma\succ0$ for $0\le\gamma<\bar\gamma$. Hence every stationary point on this interval is a strict local minimizer.

We next derive the three expansions. Differentiating $F(\eta_\gamma,\gamma)=0$ with respect to $\gamma$ gives
\[
\1H_\gamma\eta'_\gamma-\nabla_\eta H(\eta_\gamma)=0.
\]
Since $\1H_\gamma$ is invertible, it follows that
\begin{equation}\label{eq:path}
\eta'_\gamma=\1H_\gamma^{-1}\nabla_\eta H(\eta_\gamma),\qquad \eta'_0=A^{-1}b=v.
\end{equation}
The differentiability of the branch therefore yields
\[
\eta_\gamma=\eta_0+\gamma v+o(\gamma).
\]

Since $\nabla_\eta D(\eta_0)=0$, Taylor expansion of $D$ at $\eta_0$ gives
\[
D(\eta_\gamma)-D(\eta_0)
=\tfrac12(\eta_\gamma-\eta_0)^\top A(\eta_\gamma-\eta_0)
+o(\|\eta_\gamma-\eta_0\|^2).
\]
Substituting the parameter expansion and using $v=A^{-1}b$, we obtain
\[
D(\eta_\gamma)-D(\eta_0)
=\tfrac12\gamma^2 v^\top Av+o(\gamma^2)
=\tfrac12\kappa_0\gamma^2+o(\gamma^2).
\]

Finally, $H$ depends only on $\phi$, and its gradient with respect to $\eta$ at the baseline is $b$. Its first-order expansion is therefore
\[
H(\phi_\gamma)-H(\phi_0)
=b^\top(\eta_\gamma-\eta_0)+o(\|\eta_\gamma-\eta_0\|).
\]
Using $b^\top v=b^\top A^{-1}b=\kappa_0$, we conclude that
\[
H(\phi_\gamma)-H(\phi_0)=\kappa_0\gamma+o(\gamma).
\]
All three expansions remain valid when $b=0$, in which case $v=0$ and $\kappa_0=0$. Since $A\succ0$, we also have $\kappa_0>0$ whenever $b\ne0$. The $C^2$ assumption gives the stated little-$o$ remainders; a cubic big-$O$ remainder for the cost would require additional regularity. This proves the lemma.
\end{proof}

\subsection{Proof of \cref{thm:move}}\label{app:move}
\begin{proof}
We first express the distributional response in the norm from (A5). For a finite signed measure $\mu$ of total mass zero on the compact image domain from (A1), define
\[
\|\mu\|_{\mathrm{KR}}=\sup_{\operatorname{Lip}(u)\le1}\left|\int u\,d\mu\right|.
\]
By the Kantorovich--Rubinstein identity, $\|P-Q\|_{\mathrm{KR}}=W_1(P,Q)$ for probability measures $P,Q$. Thus it is enough to estimate the difference of the interventional distributions in this norm.

Let $\mathrm d\mathsf P_x(\eta_0)$ be the bounded linear derivative supplied by (A5). Its values are finite signed measures of total mass zero, since the derivative describes a change in a probability distribution. By differentiability,
\[
\|\mathsf P_x(\eta_0+h)-\mathsf P_x(\eta_0)-\mathrm d\mathsf P_x(\eta_0)[h]\|_{\mathrm{KR}}
=o(\|h\|),\qquad h\to0.
\]
Assumption (A5) also gives $\dot P_x=\mathrm d\mathsf P_x(\eta_0)[v]\ne0$. Set
\[
\rho_x=\|\dot P_x\|_{\mathrm{KR}}>0.
\]
By \cref{lem:price}, the parameter displacement satisfies $\eta_\gamma-\eta_0=\gamma v+o(\gamma)$. Since the derivative is bounded and linear,
\[
\mathrm d\mathsf P_x(\eta_0)[\eta_\gamma-\eta_0]
=\gamma\dot P_x+o(\gamma)
\quad\text{in the KR norm}.
\]
Substituting $h=\eta_\gamma-\eta_0$ into the differentiability expansion therefore gives
\[
\|P_x^{\theta_\gamma}-P_x^{\theta_0}-\gamma\dot P_x\|_{\mathrm{KR}}=o(\gamma).
\]

We now convert this estimate into a Wasserstein distance. By the reverse triangle inequality,
\[
\left|W_1(P_x^{\theta_\gamma},P_x^{\theta_0})-\gamma\rho_x\right|
\le\|P_x^{\theta_\gamma}-P_x^{\theta_0}-\gamma\dot P_x\|_{\mathrm{KR}}
=o(\gamma).
\]
Hence
\[
W_1(P_x^{\theta_\gamma},P_x^{\theta_0})=\gamma\rho_x+o(\gamma),
\]
which proves the first limit. Since $\rho_x>0$, this distance is positive for every sufficiently small $\gamma>0$.

It remains to compare this movement with the surrogate cost. Using the cost expansion in \cref{lem:price}, we obtain
\[
\begin{aligned}
\frac{D(\eta_\gamma)-D(\eta_0)}{W_1(P_x^{\theta_\gamma},P_x^{\theta_0})}
&=\frac{\tfrac12\kappa_0\gamma^2+o(\gamma^2)}{\rho_x\gamma+o(\gamma)}\\
&=\frac{\kappa_0}{2\rho_x}\gamma+o(\gamma)\longrightarrow0.
\end{aligned}
\]
The argument uses differentiability of the distribution in the KR norm as required by (A5); differentiability of a single expectation alone would not give the norm remainder above. This proves the second limit and completes the proof.
\end{proof}

\subsection{Proof of \cref{cor:admissible-move}}\label{app:admissible-move}
\begin{proof}
We first show that all observational constraints remain satisfied on a common interval. For each $j\in\{1,\ldots,J\}$, define the baseline margin
\[
m_j=\varepsilon_j-\Delta_j(\theta_0)>0.
\]
The branch $\theta_\gamma$ is continuous at $\gamma=0$, and $\Delta_j$ is continuous at $\theta_0$. Hence there is a $\delta_j>0$ such that
\[
|\Delta_j(\theta_\gamma)-\Delta_j(\theta_0)|<m_j
\qquad\text{for }0\le\gamma<\delta_j.
\]
On this interval, it follows that
\[
\Delta_j(\theta_\gamma)
<\Delta_j(\theta_0)+m_j
=\varepsilon_j.
\]
Since there are finitely many constraints, $\delta_{\mathrm{obs}}=\min_{1\le j\le J}\delta_j$ is positive. Every $0\le\gamma<\delta_{\mathrm{obs}}$ therefore satisfies all the observational constraints.

We next ensure that the interventional distribution changes on the same interval. By \cref{thm:move}, there is a $\delta_{\mathrm{move}}>0$ such that
\[
W_1(P_x^{\theta_\gamma},P_x^{\theta_0})>0
\qquad\text{for }0<\gamma<\delta_{\mathrm{move}}.
\]
Choose $\delta=\min\{\delta_{\mathrm{obs}},\delta_{\mathrm{move}},\bar\gamma\}>0$. Then every $0\le\gamma<\delta$ satisfies the observational constraints, and every $0<\gamma<\delta$ gives $P_x^{\theta_\gamma}\ne P_x^{\theta_0}$.

This argument uses the strict baseline margins and continuity. If the discrepancies are also differentiable, the parameter expansion gives
\[
\Delta_j(\theta_\gamma)-\Delta_j(\theta_0)
=\gamma\nabla_\theta\Delta_j(\theta_0)^\top v_\theta+o(\gamma),
\]
which need not be second order. The conclusion therefore concerns acceptance under the fixed observational tolerances and does not assert equality of the full observational distributions. This completes the proof.
\end{proof}

\section{Experimental Setup and Additional Results}\label{app:setup}

We evaluate \cauvade{} on synthetic Color-MNIST at $28{\times}28$ resolution, CelebA face attributes at $64{\times}64$, and MIMIC-CXR-JPG chest radiographs at $224{\times}224$. We train on confounded data and use an unconfounded reference only for evaluation. Following \cref{sec:model}, only the treatment $X$ is used as supervision; outcome $Y$, covariate $Z$, and confounder $U$ labels are not supplied to the models during training.

We first describe dataset construction (\crefrange{app:cmnist}{app:mimic}) and training configurations (\cref{app:hyperparameters}), then report reconstruction diagnostics and qualitative samples. We next examine mixed discrete--continuous confounding and non-binary attributes, followed by an entropy ablation, cluster-count sensitivity, and exploration of plausible confounder distributions.

\subsection{Dataset Construction}

For each dataset, we construct confounded and unconfounded variants through synthetic generation or resampling existing images. In the confounded variant, treatment assignment depends on $U$, creating the backdoor path $X \leftarrow U \rightarrow Y$. The unconfounded variant samples $X$ independently of $U$ while retaining the remaining construction, providing a reference for evaluating interventional samples from $P_x(I)$. The dataset-specific mechanisms and assignment probabilities are given below.

\subsubsection{Confounded Color-MNIST}\label{app:cmnist}

Color-MNIST is built from three binary variables: treatment $X$, outcome $Y$, and unobserved confounder $U$, with $P(U{=}0)=0.6$ and $P(U{=}1)=0.4$. The outcome follows the XOR mechanism $Y = U \oplus X$, taking value $1$ whenever $U$ and $X$ disagree. A Conditional GAN, conditioned on the $(X, Y)$ stratum, then generates $28{\times}28$ coloured digit images. We use $N=1{,}000$ samples. The mechanism specifies the population reference exactly; empirical frequencies remain subject to sampling error. Treatment assignment is specified as follows:
\begin{itemize}
    \item \textbf{Confounded}: $P(X{=}1 \mid U{=}0)=0.7$ and $P(X{=}1 \mid U{=}1)=0.3$, inducing the spurious correlations $P(Y{=}1 \mid X{=}0) = 0.609$ and $P(Y{=}1 \mid X{=}1) = 0.778$.
    \item \textbf{Unconfounded}: $X \perp\!\!\!\perp U$, yielding the ground-truth interventional probabilities $P(Y{=}1\mid\doX{0})=0.40$ and $P(Y{=}1\mid\doX{1})=0.60$.
\end{itemize}

\subsubsection{Confounded CelebA}\label{app:celeba}

CelebA \citep{liu2015celeba} is built by mapping five attributes to causal variables: $Z$ (\texttt{Male}), $U$ (\texttt{Attractive}, unobserved), $X$ (\texttt{Young}, treatment), $W$ (\texttt{Wearing\_Earrings}), and $Y$ (\texttt{Heavy\_Makeup}, outcome). The joint distribution is specified by $P(Z{=}1)=0.4$, $P(U{=}1 \mid Z{=}0)=0.60$, $P(U{=}1 \mid Z{=}1)=0.40$, and $P(W{=}1 \mid X, U)$ rising monotonically from $0.1$ at $(X,U)=(0,0)$ to $0.9$ at $(X,U)=(1,1)$. Images are then resampled from the original CelebA corpus to match the target joint distribution over $(Z, X, W)$ with $U$ marginalized out. Treatment assignment is specified as follows:
\begin{itemize}
    \item \textbf{Confounded} ($N=10{,}000$): $P(X{=}1 \mid U{=}0)=0.1$ and $P(X{=}1 \mid U{=}1)=0.9$.
    \item \textbf{Unconfounded} ($N=10{,}000$): $P(X{=}1)=0.5$ regardless of $U$.
\end{itemize}

\subsubsection{Confounded MIMIC-CXR-JPG}\label{app:mimic}

MIMIC-CXR-JPG \citep{johnson2024mimiccxr,johnson2019mimic,PhysioNet} is built from three causal variables: $X$ (\texttt{Pneumonia}, treatment), $Y$ (subsequent outcome variables, e.g., \texttt{Lung Opacity}, unlabeled), and $U$ (unobserved confounder), grouping the patient's age $A$ in years and whether a radiograph was acquired in the anteroposterior (AP) view $V$. These attributes define the hidden grouping used in our resampling experiment; the stipulated causal mechanism is part of the experimental construction. The dataset is partitioned into four atomic pools by $(X, Y)$ stratum, from which samples are drawn to match the target distribution. Treatment is assigned through a logistic model,
\begin{equation}
    P(X{=}1 \mid A{=}a, V{=}v) = \sigma\!\left(\beta_0 + \beta_A\,\frac{a - 60}{10} + \beta_V\, v\right), \qquad \sigma(z) = \frac{1}{1+e^{-z}},
\end{equation}
with the following coefficients:
\begin{itemize}
    \item \textbf{Confounded} ($N=20{,}000$): $\beta_0 = 0$, $\beta_A = \tfrac{1}{2}\ln 9 \approx 1.10$, and $\beta_V = \ln 9 \approx 2.20$.
    \item \textbf{Unconfounded} ($N=20{,}000$): $P(X{=}1)=0.5$ regardless of $U = (A, V)$.
\end{itemize}

\subsection{Training Details and Hyperparameters}\label{app:hyperparameters}

\textbf{Optimization.} All models are trained with the AdamW optimizer~\citep{loshchilov2019decoupledweightdecayregularization} and weight decay $1{\times}10^{-5}$. Color-MNIST and CelebA experiments run on a single NVIDIA RTX 3080 Ti (16GB VRAM); MIMIC-CXR-JPG experiments run on a single NVIDIA A100 (80GB VRAM) on an internal cluster.

\textbf{Architectural scaling.} The architecture scales with image resolution through the base channel capacity $h$ and latent dimension $d_z$. Convolutional stages use channel widths $h, 2h, 4h, \ldots$, while $d_z$ sets the dimension of the global latent $\mathbf{z}$ and increases for higher-resolution datasets.

\textbf{Dataset-specific configurations.} \cref{tab:hyperparameters} lists the configurations for all three datasets. The reconstruction, KL, and SCM loss weights ($\alpha$, $\beta$, $\lambda$) are calibrated against the input resolution to maintain stable convergence and avoid posterior collapse; the cluster count $K=8$ is held fixed across datasets, with extra clusters absorbed by the prior $P(C)$ rather than overfit. The entropy weight $\gamma$ is varied over a per-dataset grid, producing the family of generators reported in the main text.

\begin{table}[h]
    \centering
    \caption{Dataset-specific hyperparameters for \cauvade{}.}
    \label{tab:hyperparameters}
    \begin{tabular}{lccc}
        \toprule
        \textbf{Hyperparameter} & \textbf{Colored MNIST} & \textbf{CelebA} & \textbf{MIMIC-CXR} \\
        \midrule
        Input Resolution & $28 \times 28$ & $64 \times 64$ & $224 \times 224$ \\
        Total Epochs & 400 & 800 & 800 \\
        Batch Size & 128 & 64 & 128 \\
        Learning Rate & $1 \times 10^{-4}$ & $1 \times 10^{-4}$ & $1 \times 10^{-4}$ \\
        Base Channel Capacity ($h$) & 8 & 64 & 32 \\
        Latent Dimension ($d_z$) & 64 & 512 & 512 \\
        Number of Clusters ($K$) & 8 & 8 & 8 \\
        \midrule
        Recon. Weight ($\alpha$) & 10.0 & 15.0 & 15.0 \\
        KL Weight ($\beta$) & 0.1 & 0.05 & 0.05 \\
        SCM Weight ($\lambda$) & 5.0 & 3.0 & 3.0 \\
        Entropy Weight ($\gamma$) & $\{0, 1, 2, 10, 20, 50, 100\}$ & \multicolumn{2}{c}{$\{0, 1, 10, 100\}$} \\
        \bottomrule
    \end{tabular}
\end{table}

\subsection{Reconstruction Quality and Additional Samples}

For Color-MNIST and CelebA, we assess observational admissibility through agreement with the labeled joint distribution $P(X,Y)$ and the stability of image reconstruction NLL and MSE across entropy weights. For MIMIC-CXR-JPG, we use $\gamma=0$ as the baseline and assess changes in observational fit alongside the corresponding change in FID against the unconfounded reference. This comparison checks whether improved interventional image-distribution agreement is accompanied by deterioration in observational fit. We also inspect generated images across entropy weights and interventions.

\subsubsection{Reconstruction Diagnostics}\label{app:reconstruction}

\Cref{tab:recon_gamma,tab:recon_gamma_celeba} report per-pixel reconstruction MSE and NLL, together with Label RMSE for the observational joint distribution $P(X,Y)$, across the entropy weights evaluated for this diagnostic on Confounded Color-MNIST and CelebA, respectively. These grids differ from the generation settings in \cref{tab:hyperparameters}. On Confounded Color-MNIST, MSE ranges from $0.0303$ to $0.0366$, and NLL from $0.1153$ to $0.1362$. On CelebA, the variation is smaller: MSE ranges from $0.0216$ to $0.0237$, and NLL from $0.5374$ to $0.5426$. Label RMSE measures agreement between the model-implied observational joint probabilities and those estimated from labeled $(X,Y)$ data. It ranges from $0.0067$ to $0.0519$ on Color-MNIST and from $0.0055$ to $0.0387$ on CelebA. These small errors directly support fit to the labeled semantic distribution. Together with the stable reconstruction metrics, they provide strong empirical evidence for observational admissibility at the measured image and semantic levels as interventional predictions vary. They do not establish equality of the full observational distributions; the formal local guarantee remains subject to the assumptions and discrepancy tolerances of \cref{cor:admissible-move}.

\begin{table}[t]
\centering
\caption{Reconstruction quality on Confounded Color-MNIST across the entropy weights $\gamma$ evaluated for this diagnostic. Mean squared error (MSE) and reconstruction negative log-likelihood (NLL) are reported per pixel; lower is better. Label RMSE is the root-mean-square error between the model-implied observational joint probabilities $P(X,Y)$ and the empirical joint probabilities estimated from labeled data.}
\label{tab:recon_gamma}
\resizebox{\linewidth}{!}{%
\begin{tabular}{lccccccccccccc}
\toprule
Metric & $\gamma{=}0$ & $\gamma{=}1$ & $\gamma{=}2$ & $\gamma{=}10$ & $\gamma{=}20$ & $\gamma{=}30$ & $\gamma{=}40$ & $\gamma{=}50$ & $\gamma{=}60$ & $\gamma{=}70$ & $\gamma{=}80$ &
$\gamma{=}90$ & $\gamma{=}100$ \\
\midrule
MSE (per pixel) & 0.0307 & 0.0303 & 0.0313 & 0.0315 & 0.0323 & 0.0366 & 0.0330 & 0.0336 & 0.0329 & 0.0321 & 0.0319 & 0.0312 & 0.0318 \\
NLL (per pixel) & 0.1161 & 0.1153 & 0.1186 & 0.1190 & 0.1225 & 0.1362 & 0.1243 & 0.1260 & 0.1238 & 0.1215 & 0.1202 & 0.1182 & 0.1199 \\
Label RMSE & 0.0517 & 0.0519 & 0.0516 & 0.0411 & 0.0441 & 0.0309 & 0.0344 & 0.0232 & 0.0102 & 0.0143 & 0.0067 & 0.0114 & 0.0079 \\
\bottomrule
\end{tabular}}
\end{table}

\begin{table}[t]
\centering
\caption{Reconstruction quality on CelebA across the entropy weights $\gamma$ evaluated for this diagnostic. Mean squared error (MSE) and reconstruction negative log-likelihood (NLL) are reported per pixel; lower is better. Label RMSE is the root-mean-square error between the model-implied observational joint probabilities $P(X,Y)$ and the empirical joint probabilities estimated from labeled data.}
\label{tab:recon_gamma_celeba}
\resizebox{\linewidth}{!}{%
\begin{tabular}{lcccccccccccc}
\toprule
Metric & $\gamma{=}0$ & $\gamma{=}1$ & $\gamma{=}10$ & $\gamma{=}20$ & $\gamma{=}30$ & $\gamma{=}40$ & $\gamma{=}50$ & $\gamma{=}60$ & $\gamma{=}70$ & $\gamma{=}80$ & $\gamma{=}90$ & $\gamma{=}100$\\
\midrule
MSE (per pixel) & 0.0216 & 0.0221 & 0.0224 & 0.0216 & 0.0228 & 0.0233 & 0.0228 & 0.0237 & 0.0227 & 0.0233 & 0.0237 & 0.0233\\
NLL (per pixel) & 0.5374 & 0.5392 & 0.5397 & 0.5378 & 0.5409 & 0.5421 & 0.5402 & 0.5426 & 0.5399 & 0.5413 & 0.5424 & 0.5410\\
Label RMSE & 0.0101 & 0.0252 & 0.0132 & 0.0387 & 0.0255 & 0.0383 & 0.0210 & 0.0322 & 0.0055 & 0.0343 & 0.0157 & 0.0069\\
\bottomrule
\end{tabular}}
\end{table}

\subsubsection{Additional Qualitative Samples}

\textbf{Color-MNIST.}
\cref{fig:_app_cmnist} shows representative \cauvade{} samples on Confounded Color-MNIST as the entropy weight $\gamma$ varies, illustrating the visual fidelity of the model at $28{\times}28$ resolution and how the entropy regularizer reshapes the cluster posterior at each setting.

\begin{figure*}[t]
\hfill
    \begin{subfigure}{0.33\linewidth}\centering
        \setlength{\abovecaptionskip}{0pt}
        \includegraphics[width=\linewidth]{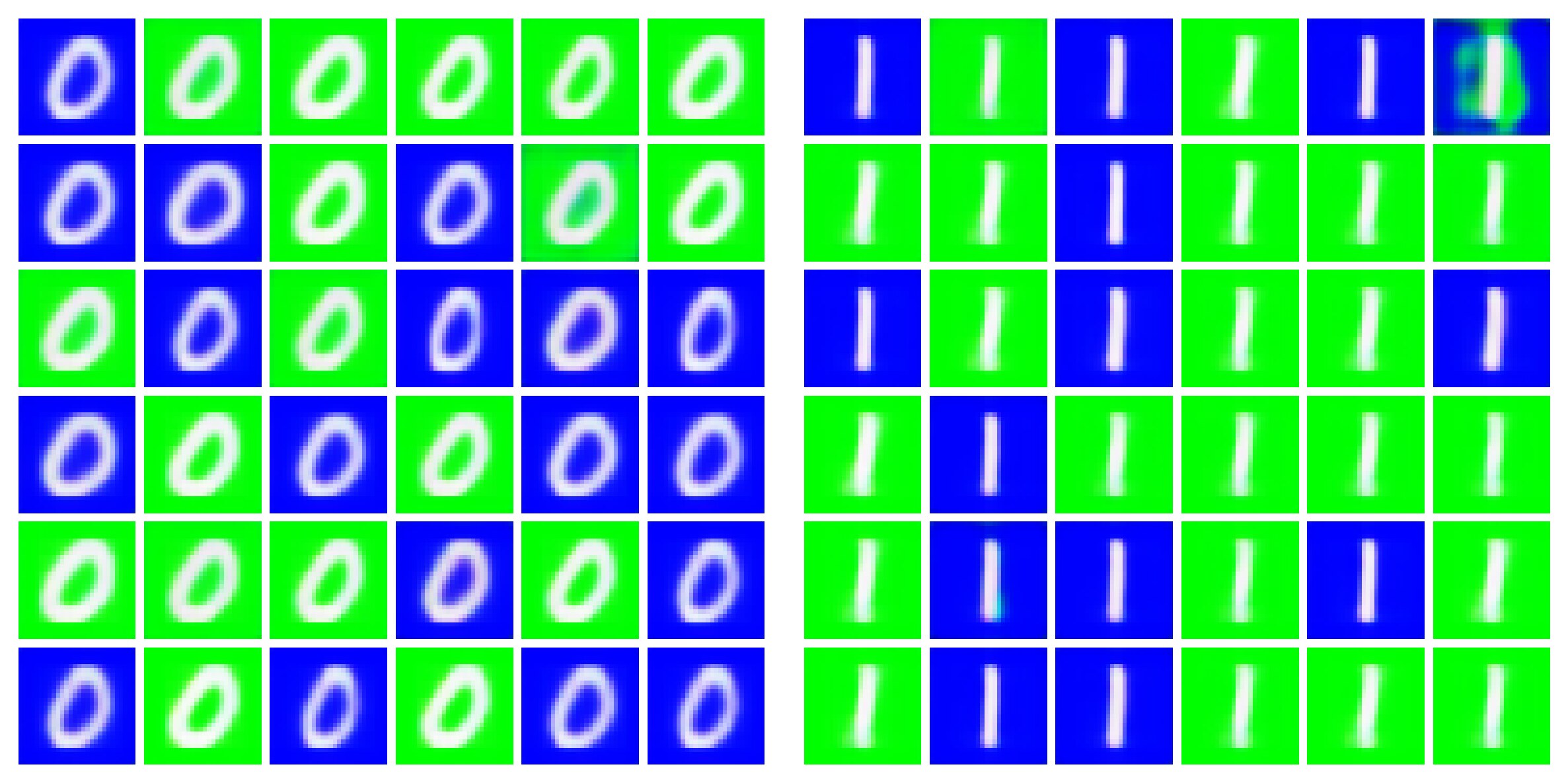}
        \caption{}
        \label{fig:_app_cmnist_a}
    \end{subfigure}\hfill
    \begin{subfigure}{0.33\linewidth}\centering
        \setlength{\abovecaptionskip}{0pt}
        \includegraphics[width=\linewidth]{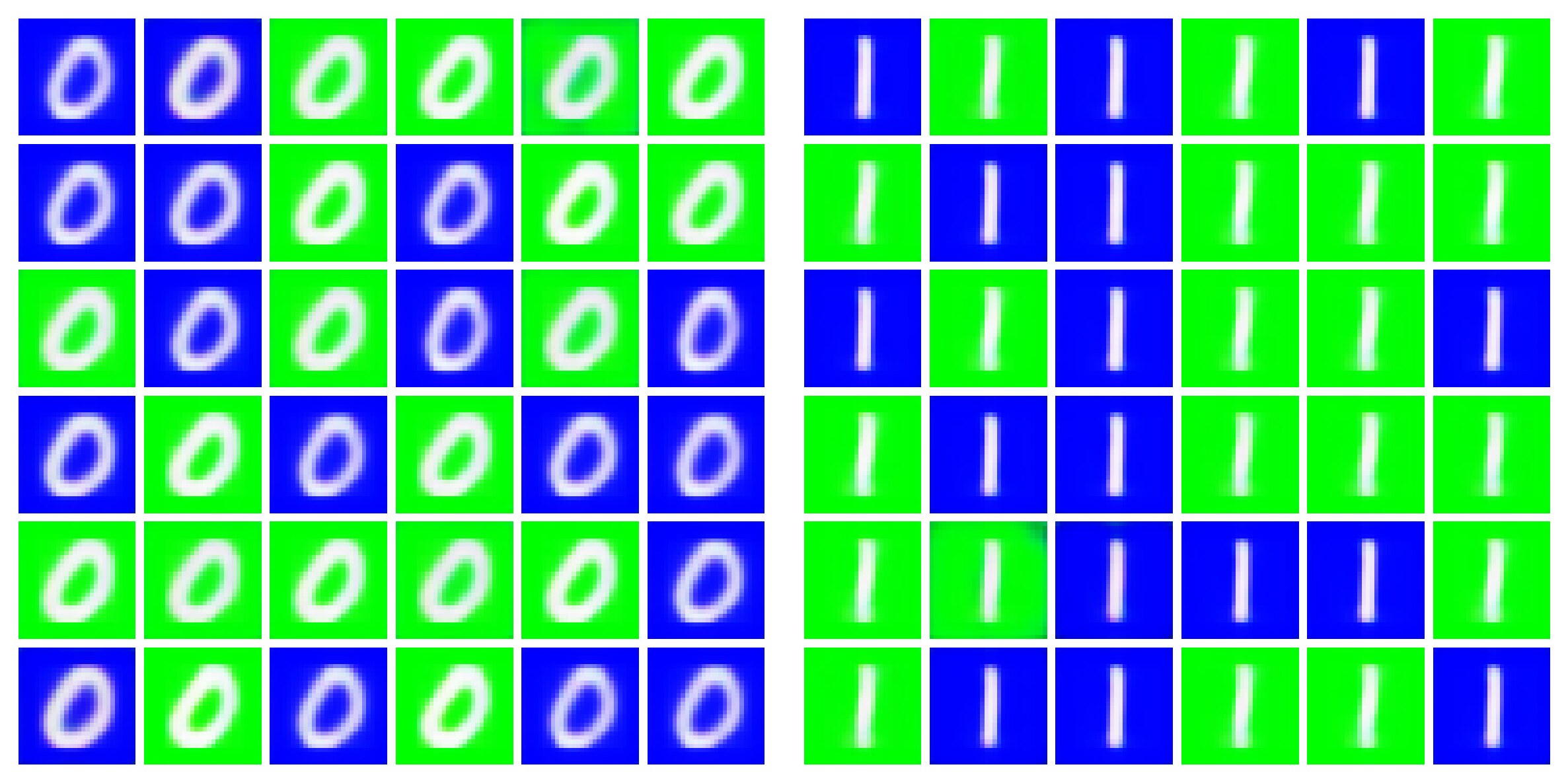}
        \caption{}
        \label{fig:_app_cmnist_b}
    \end{subfigure}\hfill
    \begin{subfigure}{0.33\linewidth}\centering
        \setlength{\abovecaptionskip}{0pt}
        \includegraphics[width=\linewidth]{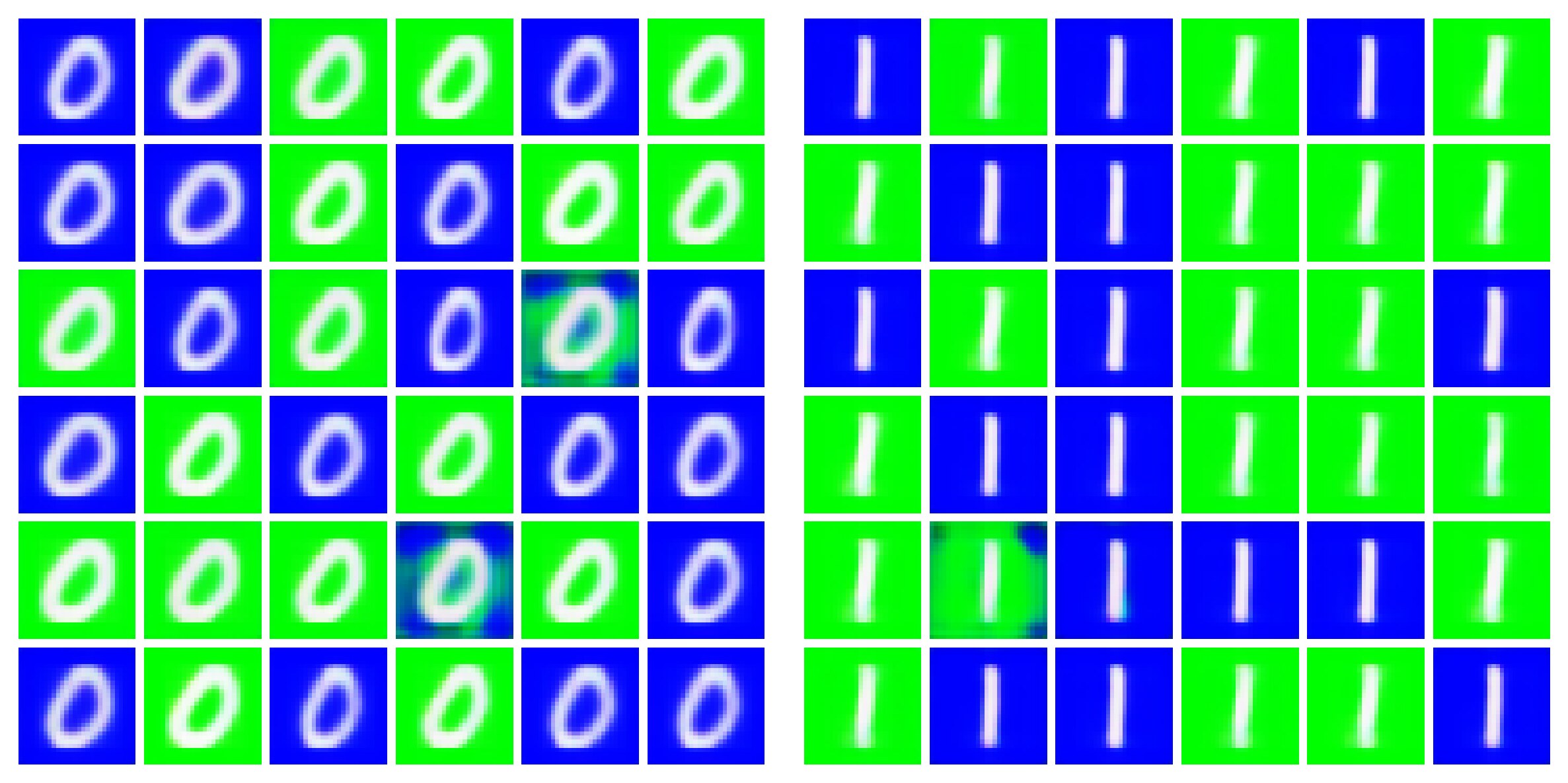}
        \caption{}
        \label{fig:_app_cmnist_c}
    \end{subfigure}\hfill

    \vspace{0.5em}

    \hfill
    \begin{subfigure}{0.33\linewidth}\centering
        \setlength{\abovecaptionskip}{0pt}
        \includegraphics[width=\linewidth]{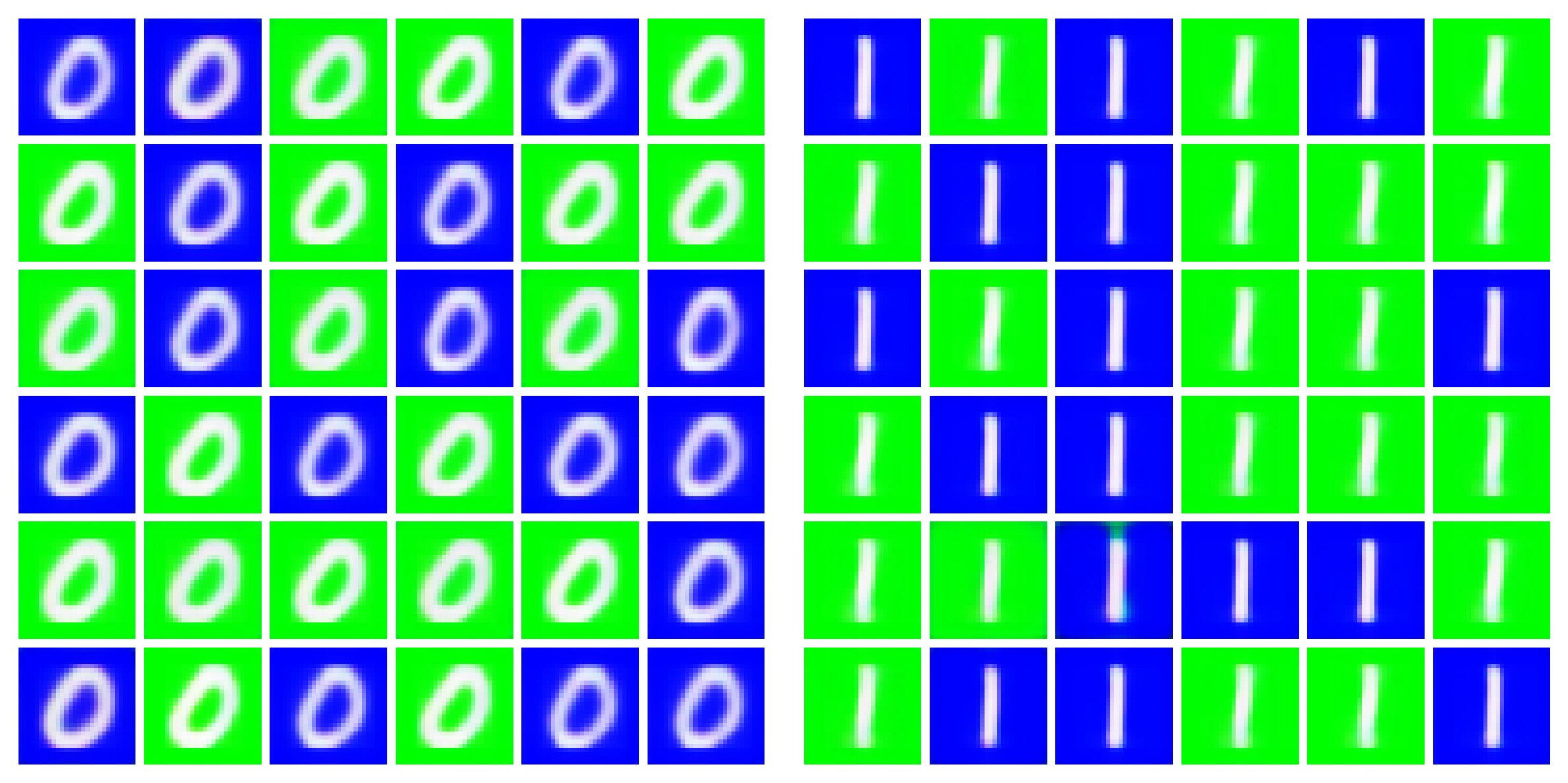}
        \caption{}
        \label{fig:_app_cmnist_d}
    \end{subfigure}\hfill
    \begin{subfigure}{0.33\linewidth}\centering
        \setlength{\abovecaptionskip}{0pt}
        \includegraphics[width=\linewidth]{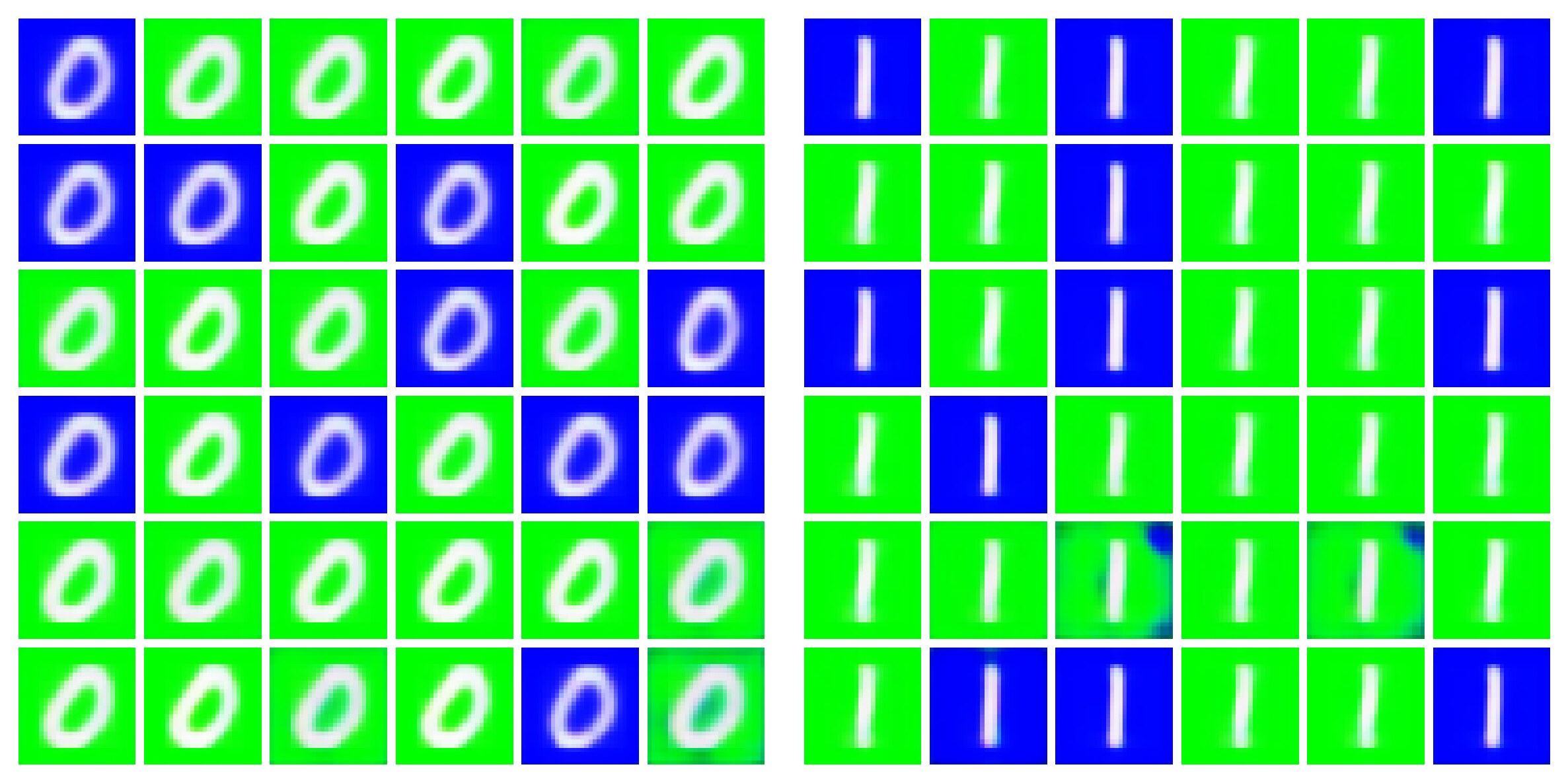}
        \caption{}
        \label{fig:_app_cmnist_e}
    \end{subfigure}\hfill
    \begin{subfigure}{0.33\linewidth}\centering
        \setlength{\abovecaptionskip}{0pt}
        \includegraphics[width=\linewidth]{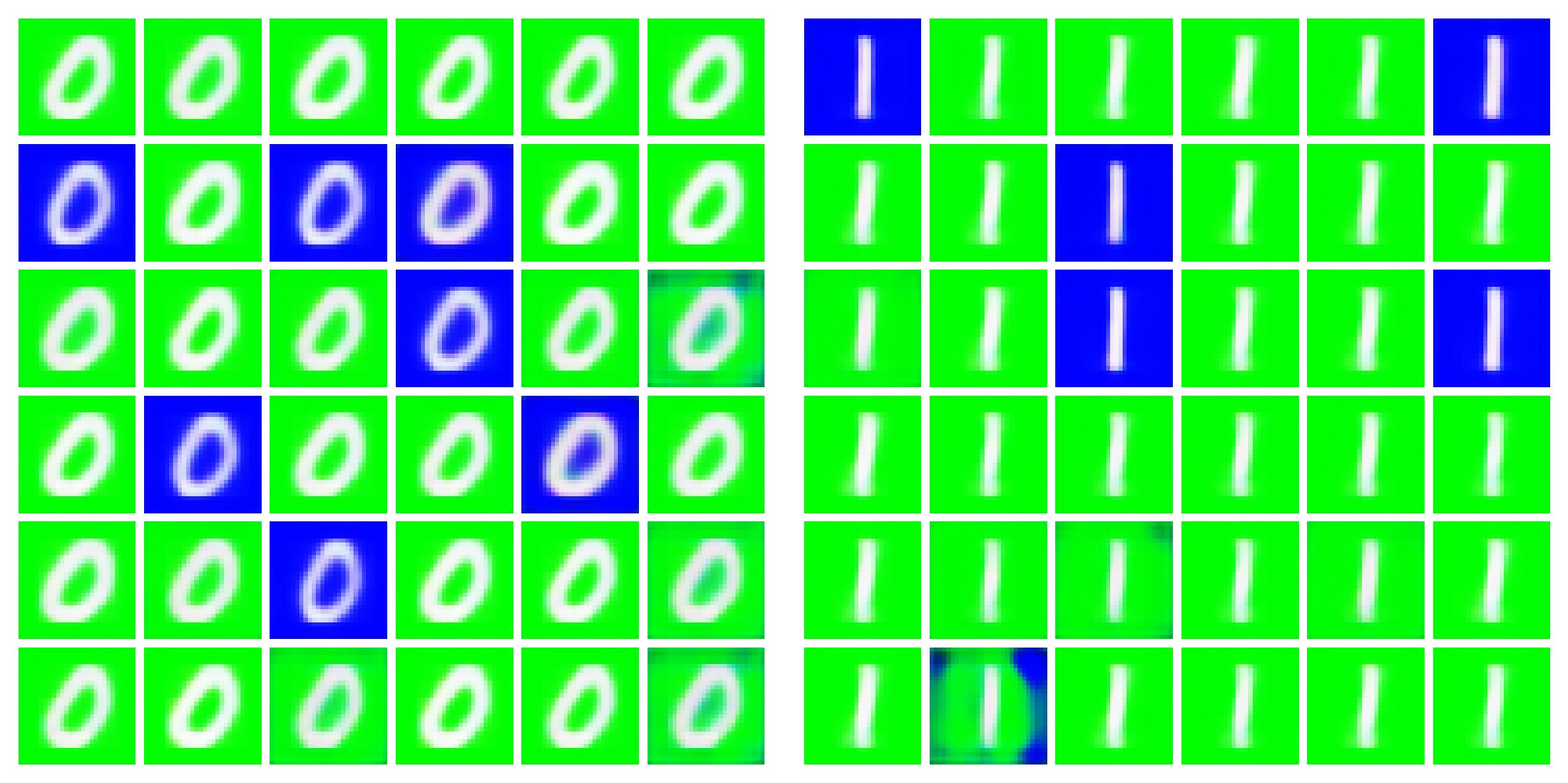}
        \caption{}
        \label{fig:_app_cmnist_f}
    \end{subfigure}\hfill

    \caption{Generated samples from \cauvade{} on Colored MNIST across varying constraint strengths:
    (\subref{fig:_app_cmnist_a}) $\gamma=1$;
    (\subref{fig:_app_cmnist_b}) $\gamma=2$;
    (\subref{fig:_app_cmnist_c}) $\gamma=10$;
    (\subref{fig:_app_cmnist_d}) $\gamma=20$;
    (\subref{fig:_app_cmnist_e}) $\gamma=50$;
    and (\subref{fig:_app_cmnist_f}) $\gamma=100$.}
    \label{fig:_app_cmnist}
\end{figure*}

\textbf{CelebA.}
\cref{fig:_app_celeba} compares samples from the confounded training set, the unconfounded ground truth, two baselines (VAE and ANCM), and \cauvade{} at four values of $\gamma$, illustrating the visual fidelity of the model at $64{\times}64$ resolution and the diversity of generators \cauvade{} produces across entropy-weight settings.

\begin{figure*}[t]
\centering
\hfill
    \begin{subfigure}{0.48\linewidth}\centering
        \setlength{\abovecaptionskip}{0pt}
        \includegraphics[width=\linewidth]{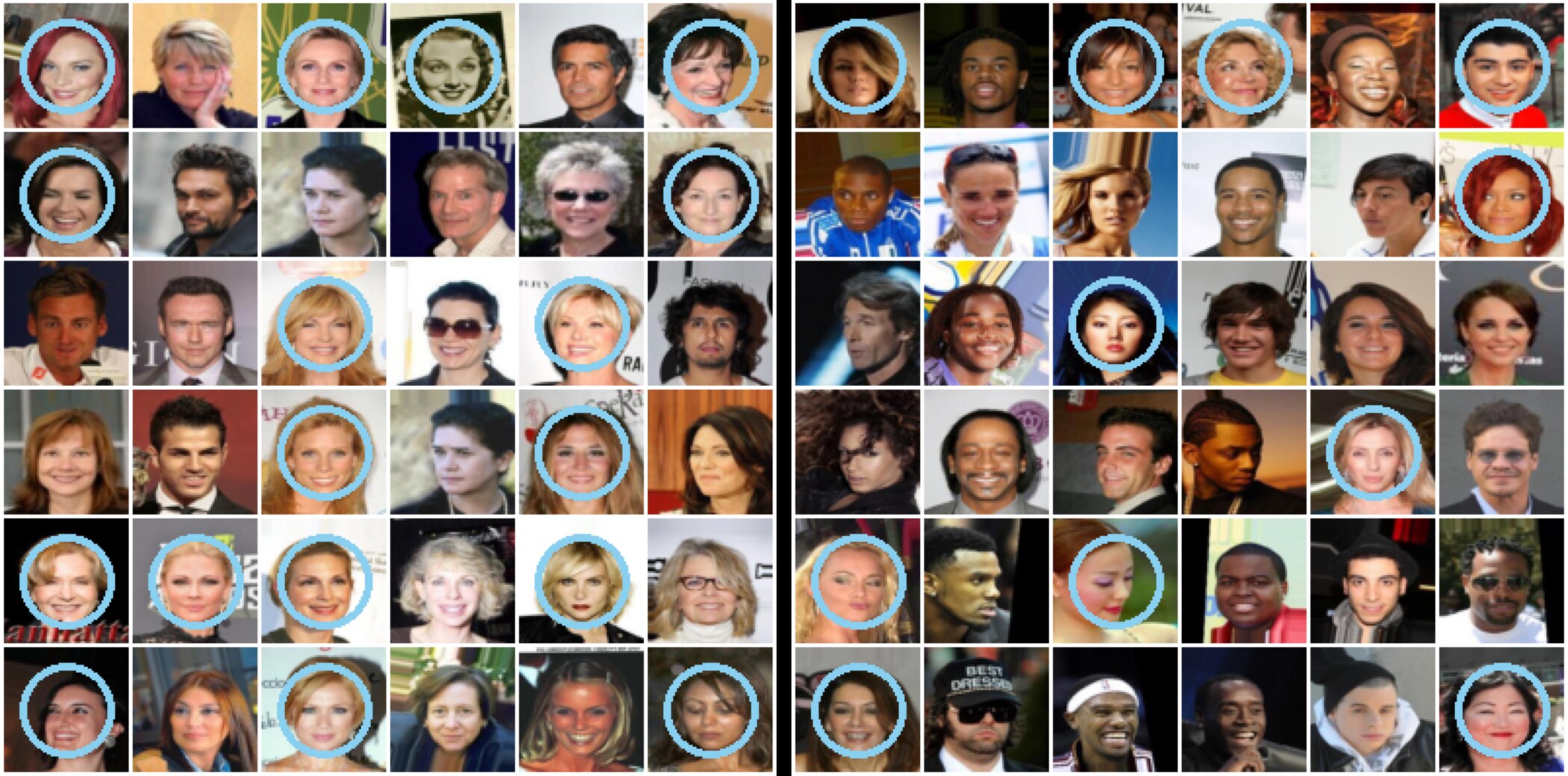}
        \caption{}
        \label{fig:_app_celeb_confound}
    \end{subfigure}\hfill
    \begin{subfigure}{0.48\linewidth}\centering
        \setlength{\abovecaptionskip}{0pt}
        \includegraphics[width=\linewidth]{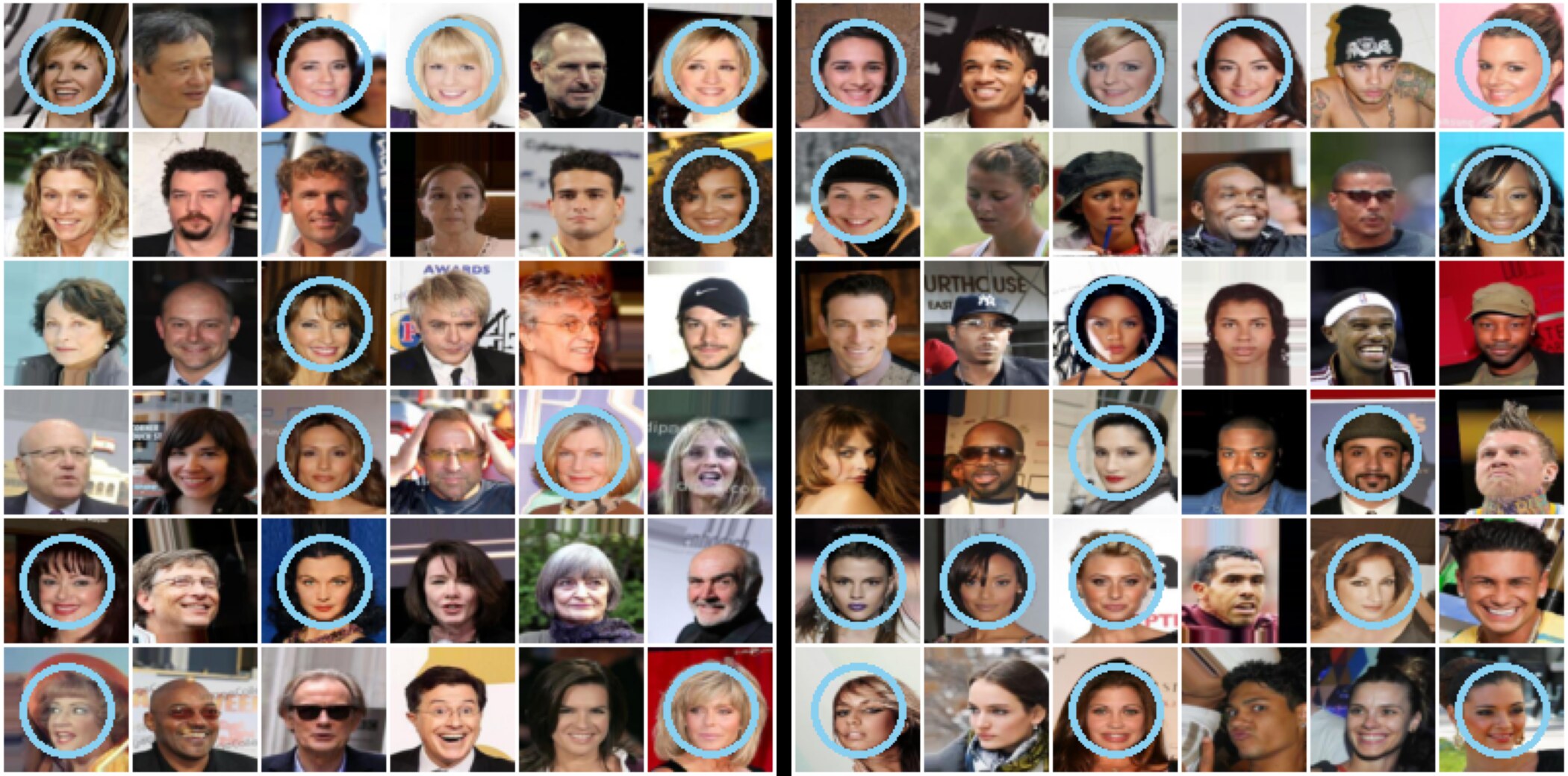}
        \caption{}
        \label{fig:_app_celeb_unconfound}
    \end{subfigure}\hfill\null

    \hfill
    \begin{subfigure}{0.48\linewidth}\centering
        \setlength{\abovecaptionskip}{0pt}
        \includegraphics[width=\linewidth]{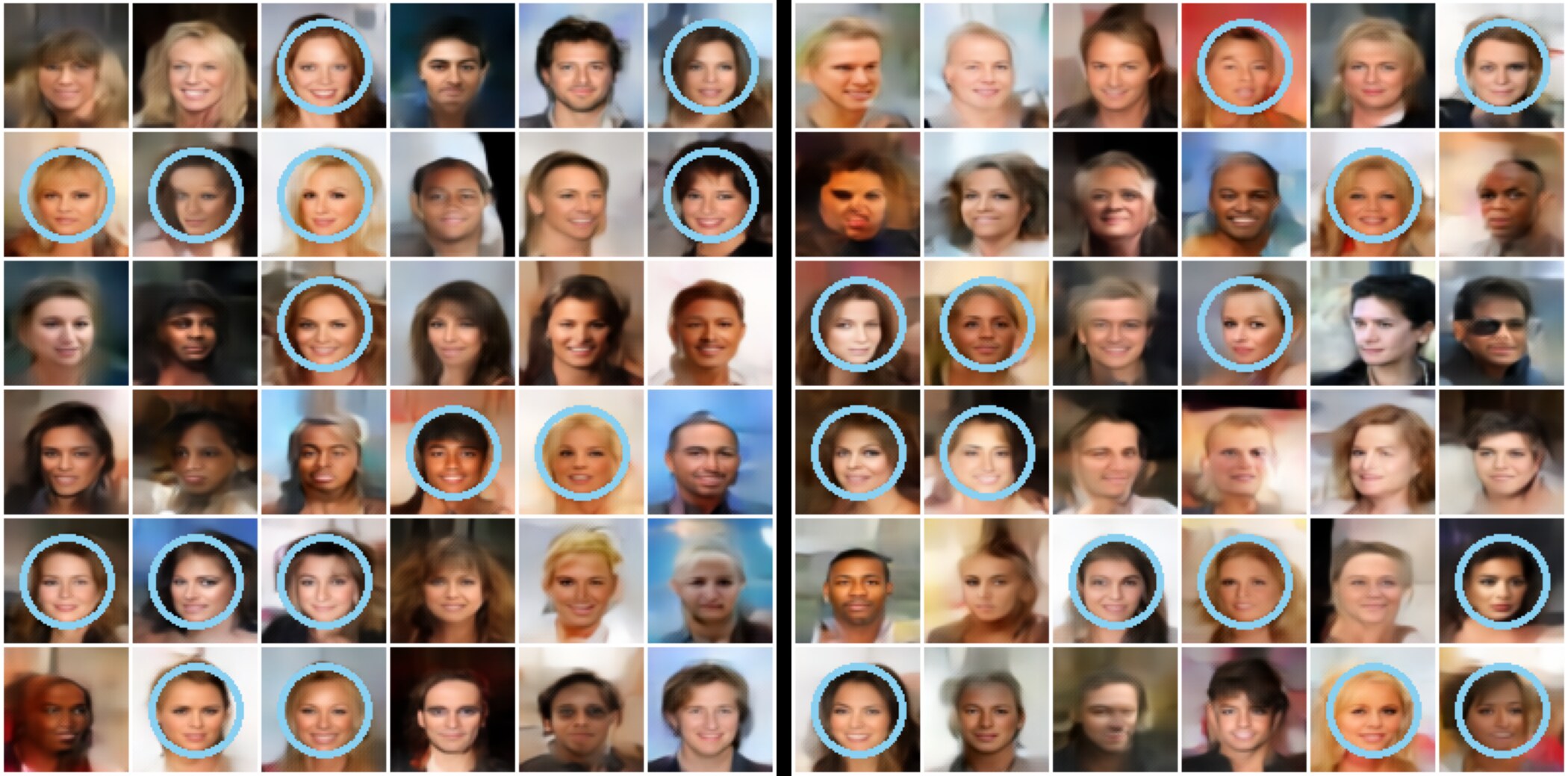}
        \caption{}
        \label{fig:_app_celeb_vae}
    \end{subfigure}\hfill
    \begin{subfigure}{0.48\linewidth}\centering
        \setlength{\abovecaptionskip}{0pt}
        \includegraphics[width=\linewidth]{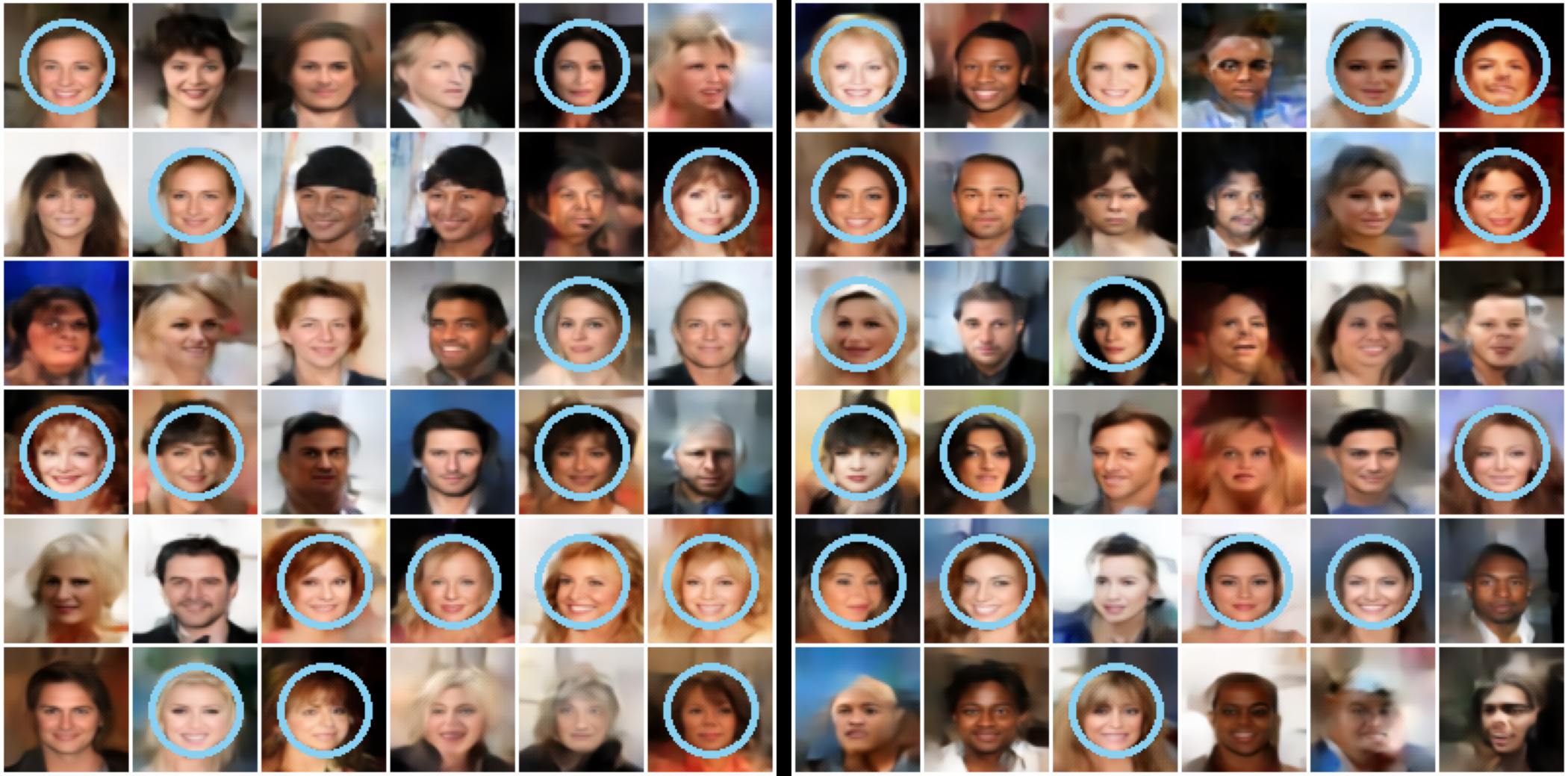}
        \caption{}
        \label{fig:_app_celeb_ancm}
    \end{subfigure}\hfill\null

    \hfill
    \begin{subfigure}{0.48\linewidth}\centering
        \setlength{\abovecaptionskip}{0pt}
        \includegraphics[width=\linewidth]{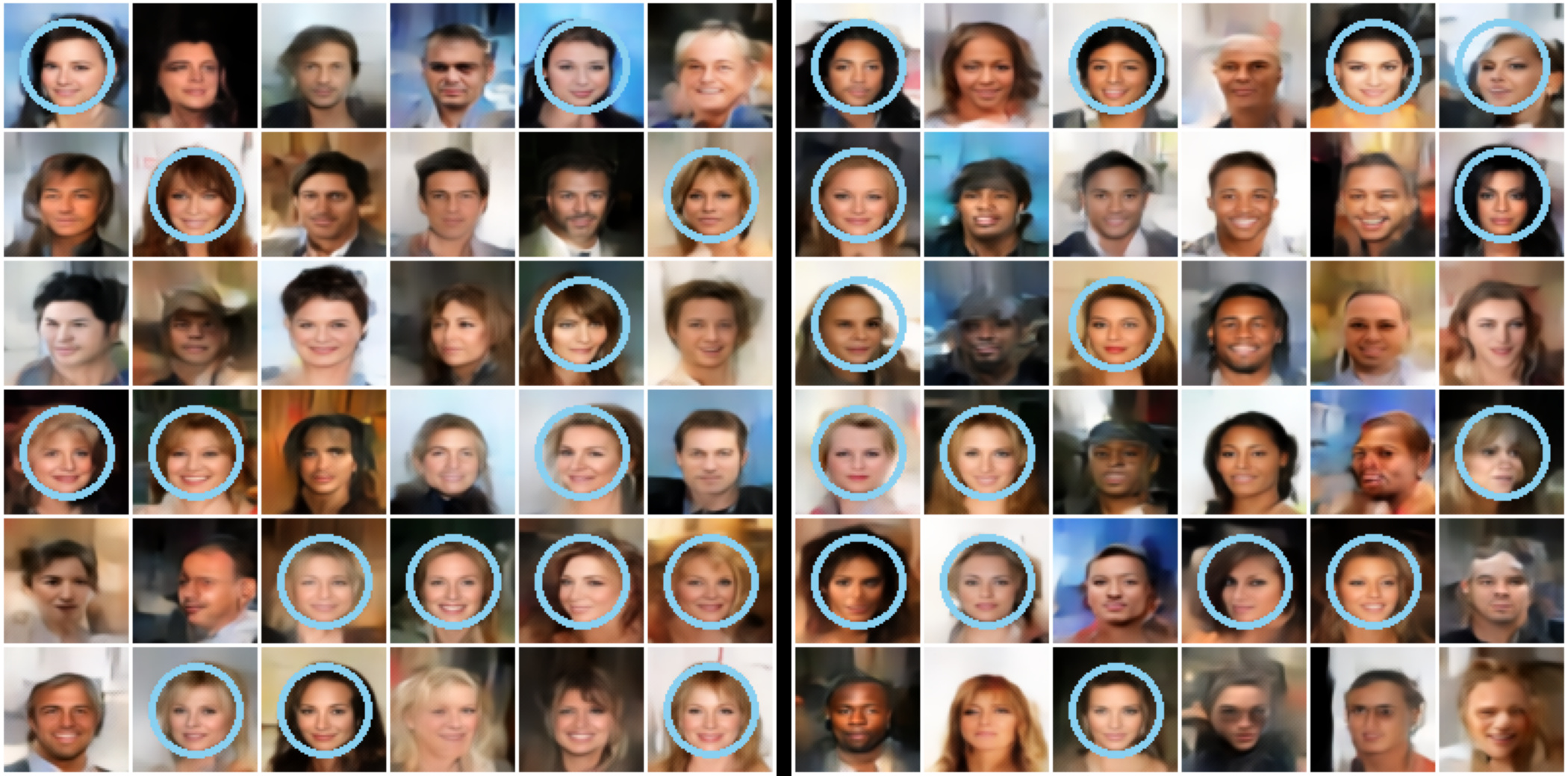}
        \caption{}
        \label{fig:_app_celeb_gamma0}
    \end{subfigure}\hfill
    \begin{subfigure}{0.48\linewidth}\centering
        \setlength{\abovecaptionskip}{0pt}
        \includegraphics[width=\linewidth]{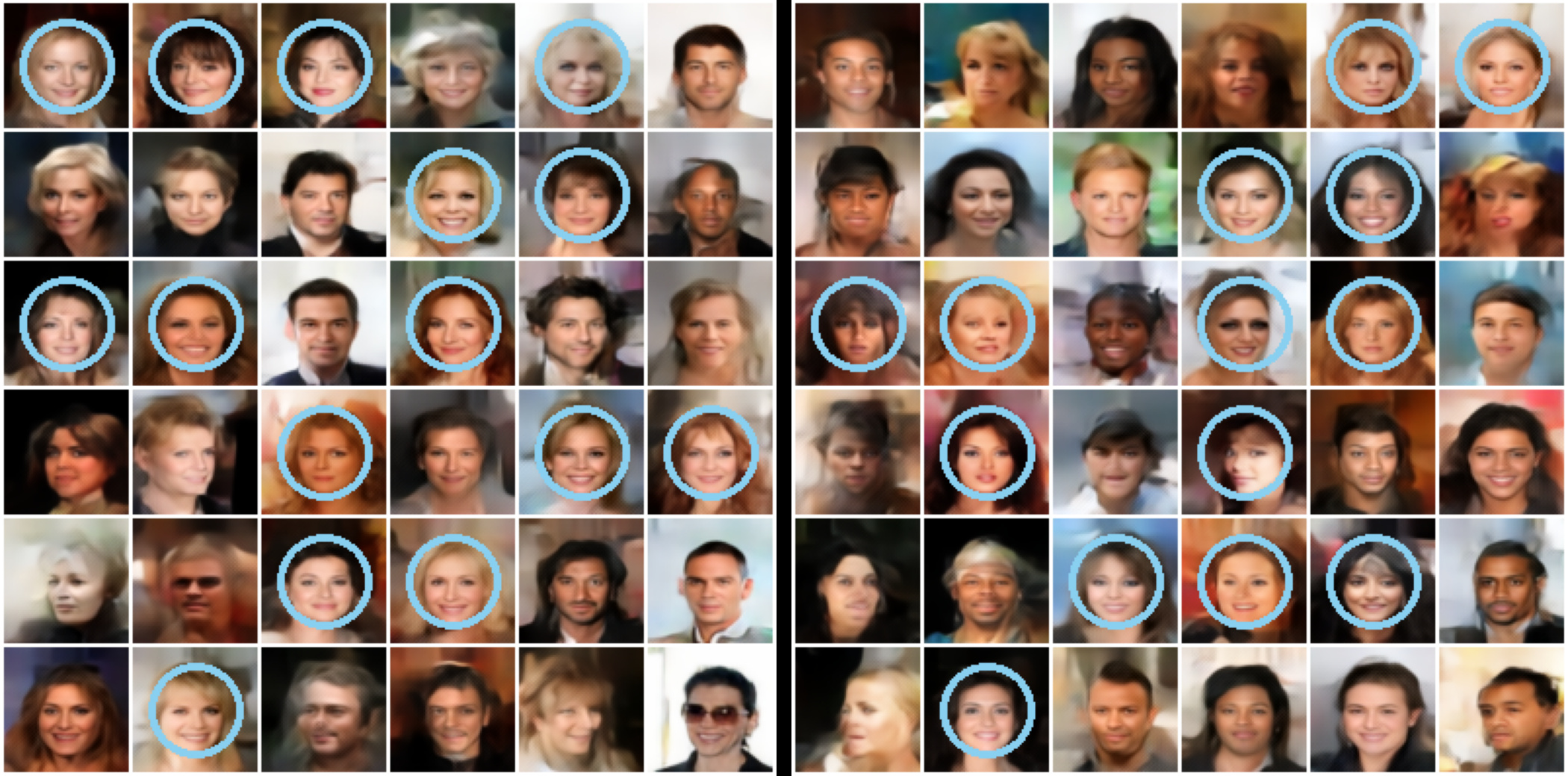}
        \caption{}
        \label{fig:_app_celeb_gamma1}
    \end{subfigure}\hfill\null

    \hfill
    \begin{subfigure}{0.48\linewidth}\centering
        \setlength{\abovecaptionskip}{0pt}
        \includegraphics[width=\linewidth]{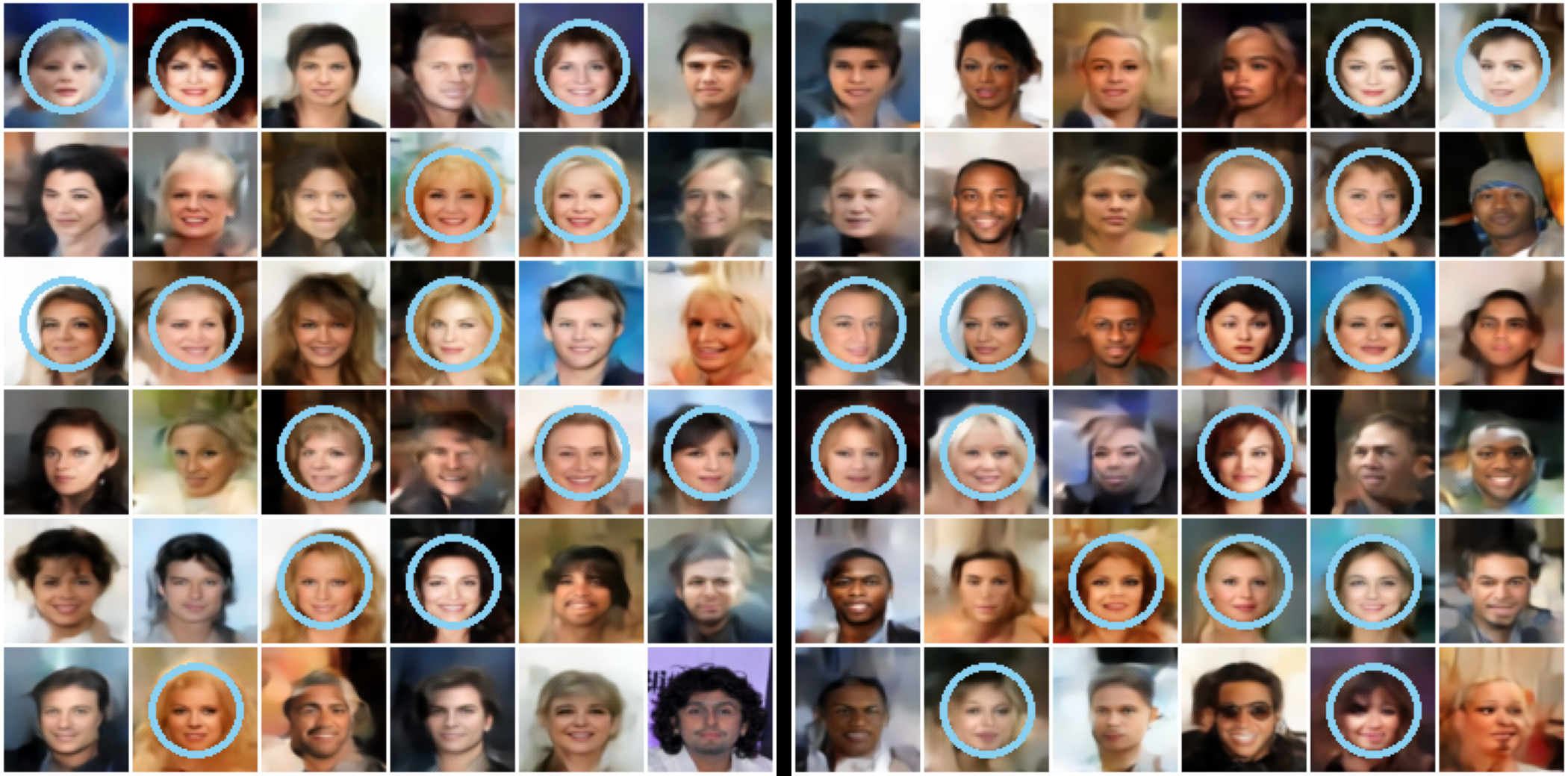}
        \caption{}
        \label{fig:_app_celeb_gamma10}
    \end{subfigure}\hfill
    \begin{subfigure}{0.48\linewidth}\centering
        \setlength{\abovecaptionskip}{0pt}
        \includegraphics[width=\linewidth]{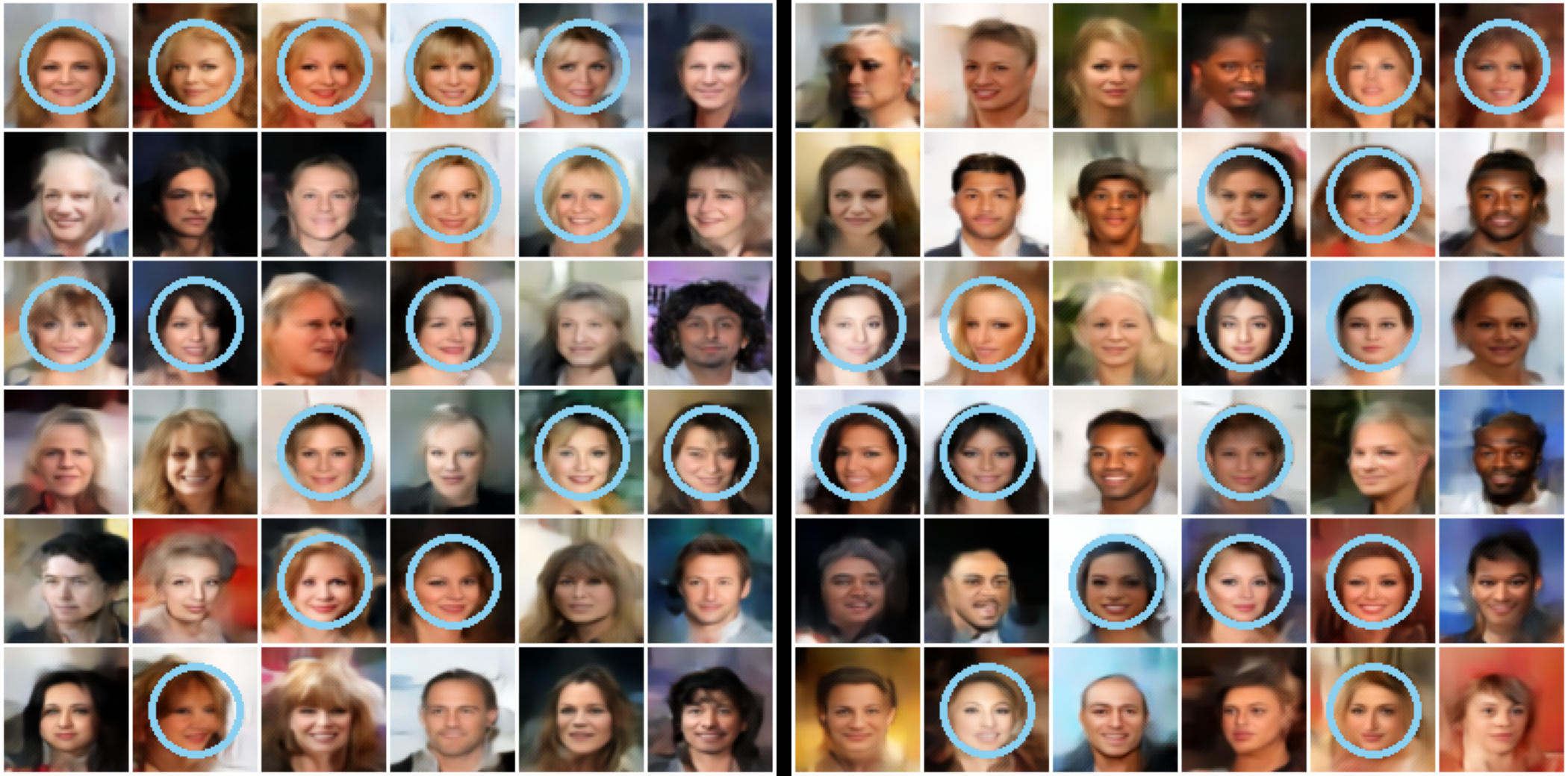}
        \caption{}
        \label{fig:_app_celeb_gamma100}
    \end{subfigure}\hfill\null

    \caption{Visualizing data distributions on CelebA. Samples are drawn from:
    (\subref{fig:_app_celeb_confound}) the confounded training set;
    (\subref{fig:_app_celeb_unconfound}) the unconfounded ground truth;
    (\subref{fig:_app_celeb_vae}) synthetic outputs from VAE;
    (\subref{fig:_app_celeb_ancm}) synthetic outputs from ANCM;
    (\subref{fig:_app_celeb_gamma0}) synthetic outputs from \cauvade{} with $\gamma{=}0$;
    (\subref{fig:_app_celeb_gamma1}) synthetic outputs from \cauvade{} with $\gamma{=}1$;
    (\subref{fig:_app_celeb_gamma10}) synthetic outputs from \cauvade{} with $\gamma{=}10$;
    and (\subref{fig:_app_celeb_gamma100}) synthetic outputs from \cauvade{} with $\gamma{=}100$.}
    \label{fig:_app_celeba}
\end{figure*}

\textbf{MIMIC-CXR-JPG.}
\cref{fig:_app_mimic_samples} shows \cauvade{} samples under both interventions, $P_{X=0}(I)$ and $P_{X=1}(I)$, illustrating the visual fidelity of the model at $224{\times}224$ resolution.

\begin{figure}[t]
    \centering
    \includegraphics[width=\linewidth]{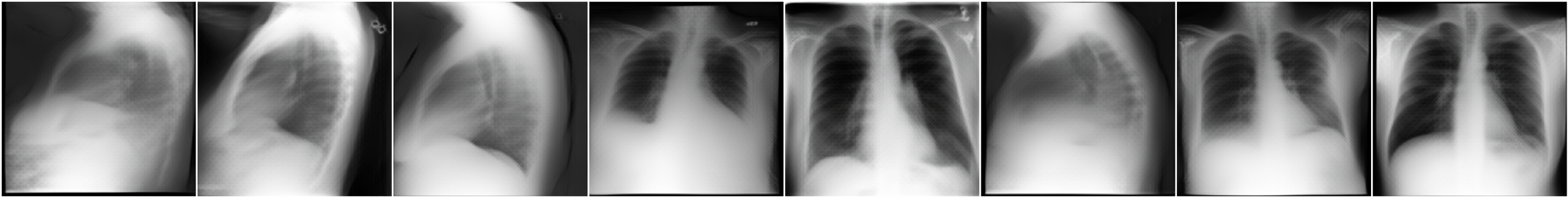}
    \includegraphics[width=\linewidth]{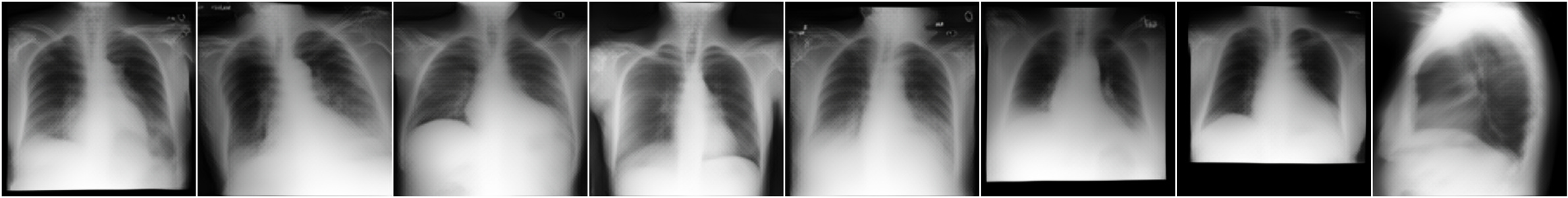}
    \caption{Generated samples from \cauvade{} under interventions $P_{X=0}(I)$ (top) and $P_{X=1}(I)$ (bottom).}
    \label{fig:_app_mimic_samples}
\end{figure}

\subsection{Extended Experimental Settings}

We extend the binary Color-MNIST construction in two directions: mixed discrete--continuous confounding and non-binary treatment and outcome attributes.

\subsubsection{Mixed Discrete--Continuous Confounding}
\label{app:mixed-confounder}

We evaluate on a Confounded Color-MNIST variant whose confounder mixes
discrete and continuous sources: $U=(U_0,U_1,U_2,U_3)$ with
$U_0,U_1\sim\mathrm{Bernoulli}(0.4)$ and $U_2,U_3\sim\mathrm{Uniform}(0,1)$,
combined into a single continuous score $Z=\sum_i U_i$. The treatment
propensity is a smooth logistic function of the full continuous $Z$,
\begin{equation}
P(X{=}1\mid Z)=\sigma\bigl(k(z^{*}-Z)\bigr).
\label{eq:propensity}
\end{equation}
A quantile-calibrated threshold on $Z$ keeps the marginal $P(b{=}1)=0.4$,
matching the binarized confounder of \cref{app:cmnist}. The resulting Manski
intervals are $[0.288,0.837]$ for $do(X{=}0)$ and $[0.423,0.874]$ for
$do(X{=}1)$. \Cref{tab:gamma-sweep-mixed} reports the $\gamma$-sweep; all
twelve swept estimates fall within these intervals.

\begin{table}[htbp]
\centering
\caption{$\gamma$-sweep on the mixed discrete--continuous confounder. All six
swept estimates on both interventions fall within the Manski intervals
(12/12).}
\label{tab:gamma-sweep-mixed}
\begin{tabular}{ccccc}
\toprule
$\gamma$ & $P(do\,X{=}0)$ & in $[0.288,0.837]$ & $P(do\,X{=}1)$ & in $[0.423,0.874]$\\
\midrule
0  & 0.619 & \checkmark & 0.656 & \checkmark\\
1  & 0.663 & \checkmark & 0.703 & \checkmark\\
2  & 0.666 & \checkmark & 0.706 & \checkmark\\
10 & 0.682 & \checkmark & 0.720 & \checkmark\\
20 & 0.660 & \checkmark & 0.704 & \checkmark\\
50 & 0.707 & \checkmark & 0.737 & \checkmark\\
\bottomrule
\end{tabular}
\end{table}

\subsubsection{Non-Binary Attributes}
\label{app:nonbinary}

We evaluate on a 3-value Confounded Color-MNIST: the treatment $X$ takes four
intervention values and the outcome is a 3-way color distribution over the
B/G/R channels. \Cref{tab:nonbinary-gamma-sweep} reports the
per-intervention B/G/R distributions across the $\gamma$-sweep, together with
the ground truth. The in-bounds column counts swept predictions falling inside
the Manski feasible region; all twelve do so at both $\gamma{=}1$ and
$\gamma{=}10$.

\begin{table}[htbp]
\centering
\caption{$\gamma$-sweep on 3-value Confounded Color-MNIST. Per-intervention
B/G/R outcome distributions for the four intervention values of $X$, with the
ground truth (GT) in the first row. The in-bounds column counts swept
predictions (4 interventions $\times$ 3 color channels) falling inside the
Manski feasible region.}
\label{tab:nonbinary-gamma-sweep}
\small
\begin{tabular}{cccccc}
\toprule
$\gamma$ & $do(X{=}0)$ B/G/R & $do(X{=}1)$ B/G/R & $do(X{=}2)$ B/G/R & $do(X{=}3)$ B/G/R & in-bounds\\
\midrule
GT & 0.54/0.24/0.22 & 0.28/0.54/0.18 & 0.18/0.28/0.54 & 0.32/0.36/0.32 & ---\\
1  & 0.248/0.380/0.372 & 0.259/0.371/0.370 & 0.251/0.379/0.371 & 0.250/0.366/0.384 & 12/12\\
10 & 0.258/0.350/0.392 & 0.268/0.343/0.389 & 0.251/0.355/0.394 & 0.255/0.344/0.401 & 12/12\\
\bottomrule
\end{tabular}
\end{table}

\subsection{Ablations and Confounder Diagnostics}

We isolate the role of the entropy term, examine sensitivity to the number of latent clusters, and explore variation in fitted confounder distributions using the injected distribution as a diagnostic reference.

\subsubsection{Entropy Ablation}\label{sec:ablation-gamma-zero}

To isolate the role of the entropy term, we set $\gamma = 0$. The term
$H(\phi)$ then drops out of \cref{eq:objective}, and \cauvade{} fits only
the data term $D$. In this setting, the estimate of $P_x(Y{=}1)$ produced by
\cauvade{} nearly coincides with that of ANCM and concentrates tightly across
random seeds (\cref{fig:_vade_gamma_zero,tab:minmax}). Both estimates lie inside the
Manski bounds, so they are consistent with the observed data. However, each
recovers only a single point of the feasible region and gives no indication
of its extent.

This ablation shows that fitting the data alone does not trace the
partially identified set. The entropy term moves the solution within the region, and sweeping $\gamma$ is how \cauvade{} produces multiple distinct candidates rather than one point inside it, without guaranteeing coverage of the region.

\begin{table}[t]
\centering
\caption{Min and max of $P(Y=1 \mid do(X{=}0))$ and $P(Y=1 \mid do(X{=}1))$ across \cauvade{} $\gamma=0$ and ANCM.}
\label{tab:minmax}
\begin{tabular}{llcccc}
\toprule
\multirow{2}{*}{Dataset} & \multirow{2}{*}{Model} & \multicolumn{2}{c}{$P(Y=1 \mid do(X{=}0))$} & \multicolumn{2}{c}{$P(Y=1 \mid do(X{=}1))$} \\
\cmidrule(lr){3-4} \cmidrule(lr){5-6}
 & & Min & Max & Min & Max \\
\midrule
\multirow{2}{*}{CMNIST} & ANCM & 0.5700 & 0.6000 & 0.6308 & 0.6426 \\
                        & \cauvade{} & 0.5700 & 0.5891 & 0.5765 & 0.6108 \\
\midrule
\multirow{2}{*}{CelebA} & ANCM & 0.4885 & 0.5011 & 0.4755 & 0.5027 \\
                        & \cauvade{} & 0.3899 & 0.4373 & 0.3779 & 0.4539 \\
\bottomrule
\end{tabular}
\end{table}

\begin{center}
    \begin{minipage}{0.49\linewidth}\centering
        \includegraphics[width=\linewidth]{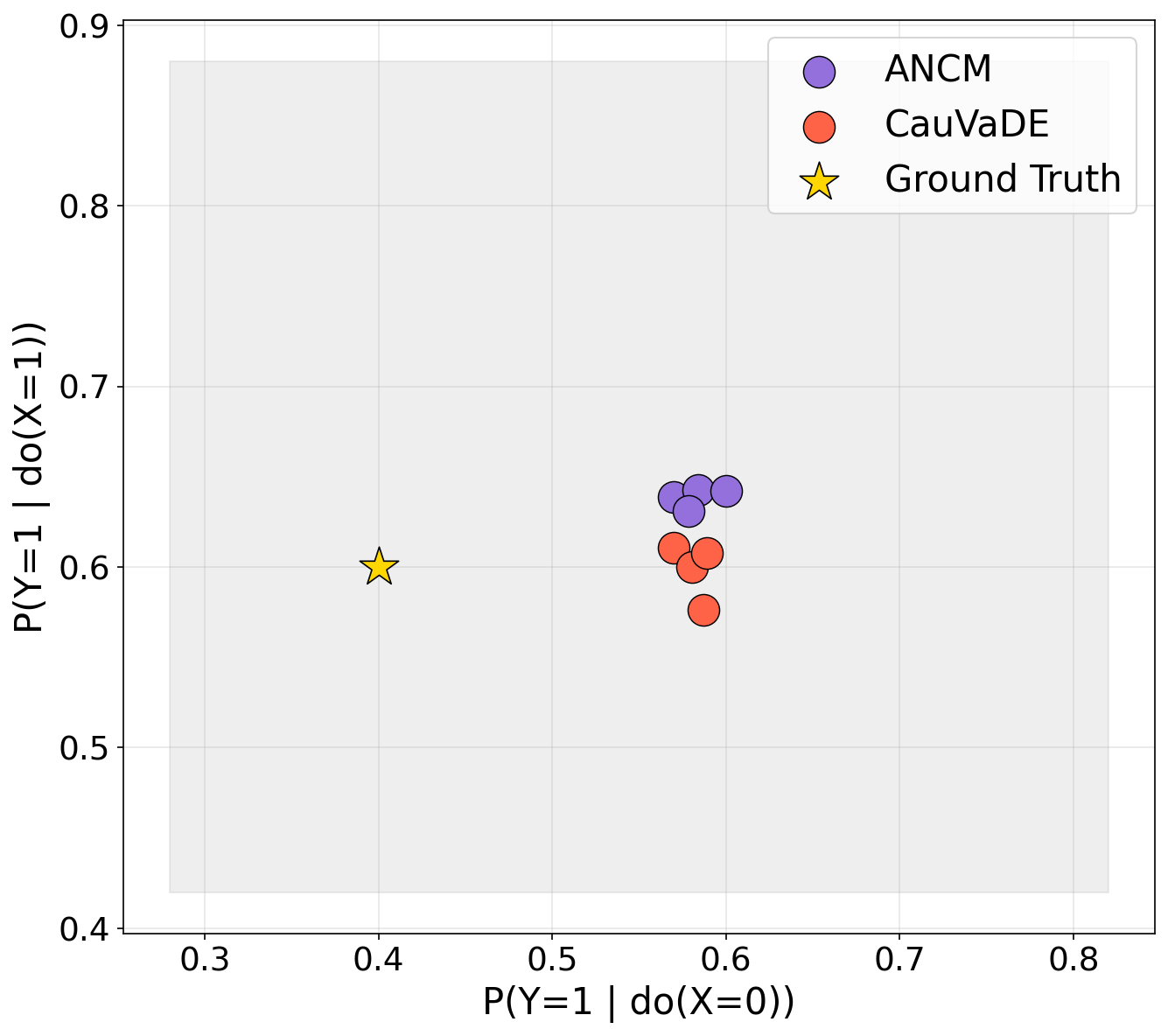}\\
        \small (a) Color-MNIST
    \end{minipage}\hfill
    \begin{minipage}{0.49\linewidth}\centering
        \includegraphics[width=\linewidth]{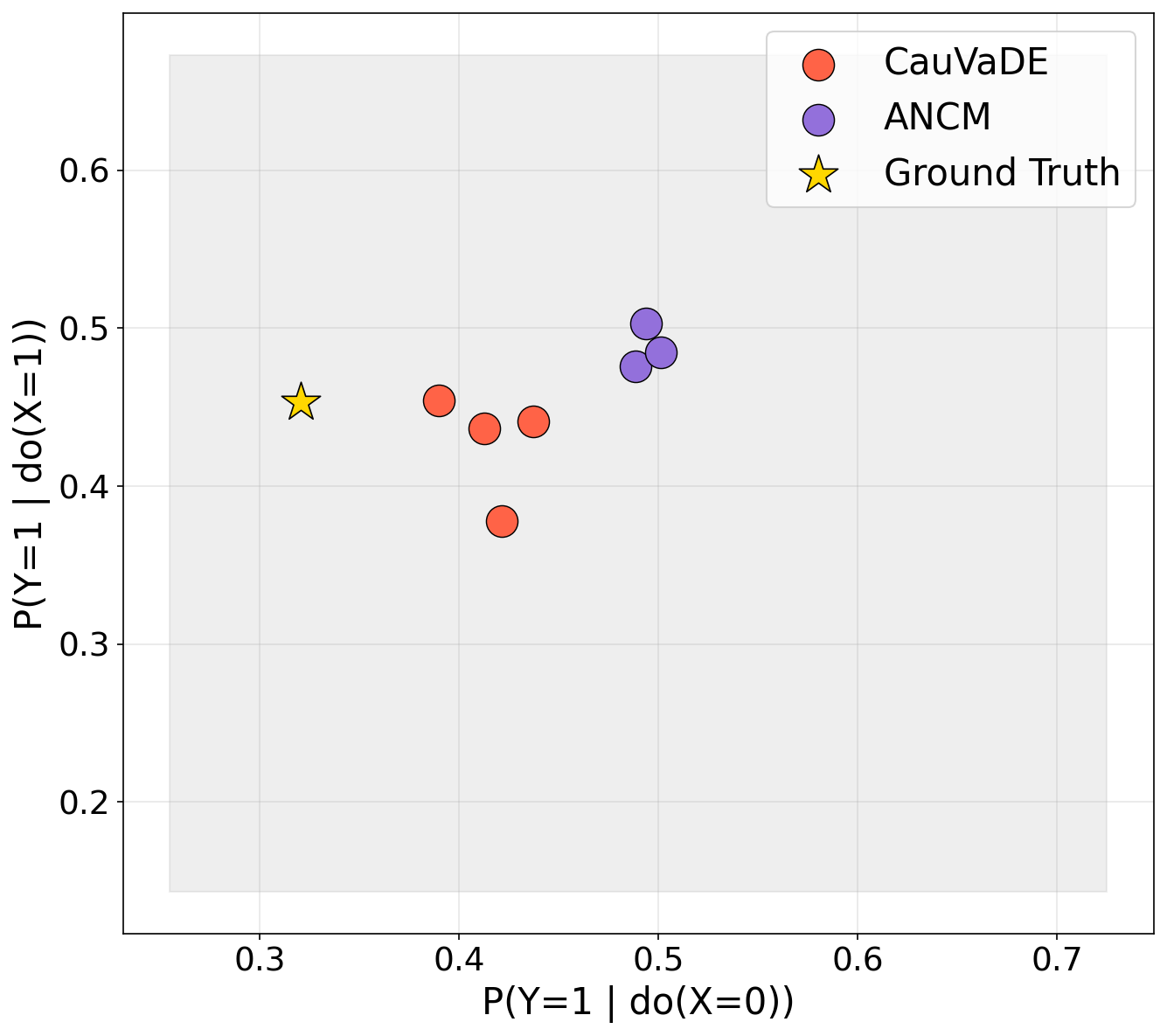}\\
        \small (b) Confounded CelebA
    \end{minipage}
    \captionof{figure}{Estimated $P_x(Y=1)$ on Confounded Color-MNIST and CelebA under $\gamma=0$, which closely matches ANCM. The shaded region denotes Manski's causal bounds \citep{manski1990nonparametric}.}
    \label{fig:_vade_gamma_zero}
\end{center}

\subsubsection{Sensitivity to the Cluster Count}
\label{app:K-sensitivity}

We first vary $K$ by changing $\gamma$ at $K \in \{2,4,6\}$ on Confounded
Color-MNIST, and we measure the width of the recovered interventional range for
each intervention (\cref{tab:K-sensitivity}). The in-bounds column counts how
many of the $28$ swept estimates per intervention fall inside the ground-truth
Manski region. Going from $K{=}2$ to $K{=}4$, the width increases from $0.167$
to $0.352$ for $do(X{=}0)$ and from $0.171$ to $0.283$ for $do(X{=}1)$, and
in-bounds coverage rises from $14/28$ to $26/28$. At $K{=}6$, the width is
$0.216$ for $do(X{=}0)$ and $0.178$ for $do(X{=}1)$, and in-bounds coverage
reaches $28/28$.

We next consider a large $K$. \Cref{tab:K100-occupancy} reports single fitted
instances at $K{=}100$ on Color-MNIST. Occupancy is the number of components
with non-negligible aggregated responsibility, and $I_q(C;Z)$ is the
categorical KL term of \cref{eq:elbo}. Occupancy stays between $12$ and $21$ of
$100$ components, and the holdout BCE is flat across $\gamma$.

\begin{table}[htbp]
\centering
\caption{Recovered interventional range on Confounded Color-MNIST at $K \in \{2,4,6\}$, aggregated over the $\gamma$-sweep. Ranges span the swept
estimates inside the region; the in-bounds column counts how many of the $28$
swept estimates per intervention those are.}
\label{tab:K-sensitivity}
\begin{tabular}{cccc}
\toprule
$K$ & in-bounds & $P(Y\mid \doo(X{=}0))$ range (width) & $P(Y\mid \doo(X{=}1))$ range (width)\\
\midrule
2 & 14/28 & $[0.500,\,0.667]$ \; ($0.167$) & $[0.494,\,0.665]$ \; ($0.171$)\\
4 & 26/28 & $[0.392,\,0.744]$ \; ($0.352$) & $[0.462,\,0.745]$ \; ($0.283$)\\
6 & 28/28 & $[0.593,\,0.808]$ \; ($0.216$) & $[0.504,\,0.682]$ \; ($0.178$)\\
\bottomrule
\end{tabular}
\end{table}

\begin{table}[htbp]
\centering
\caption{Single fitted instances at $K{=}100$ on Color-MNIST. These are
individual model evaluations, not a recovered bound.}
\label{tab:K100-occupancy}
\begin{tabular}{cccccc}
\toprule
$\gamma$ & Occupancy & $I_q(C;Z)$ & Holdout BCE & $P(Y \mid \doo(X{=}0))$ & $P(Y=1 \mid \doo(X{=}1))$\\
\midrule
1   & 21 / 100 & 1.290 & 0.762 & 0.574 & 0.596\\
10  & 20 / 100 & 1.289 & 0.760 & 0.583 & 0.610\\
100 & 12 / 100 & 1.274 & 0.760 & 0.625 & 0.696\\
\bottomrule
\end{tabular}
\end{table}

\subsubsection{Exploring Plausible Confounder Distributions}
\label{app:confounder-recovery}

Under hidden confounding, neither the latent representation $U$ nor its
distribution $P(U)$ is generally identifiable from observational data.
This structural ambiguity motivates our approach: drawing inspiration from partial identification, \cauvade{} explores plausible confounder distributions
compatible with the observations and uses them to generate diverse
interventional image predictions. Recovering the true latent representation
or confounder distribution is therefore not our objective.

\begin{table}[htbp]
\centering
\caption{Average $L_2$ distance between the learned confounder distribution
$P(C)$ and the injected $P(U)$ on Confounded Color-MNIST, used as a diagnostic
reference across the $\gamma$-sweep and for ANCM. Cluster indices are matched
to confounder states before the distance is computed. This comparison does
not establish identification of the latent representation or its distribution.}
\label{tab:confounder-l2}
\begin{tabular}{lcccc}
\toprule
Model & ANCM & \multicolumn{3}{c}{\cauvade{}}\\
\cmidrule(lr){3-5}
$\gamma$ & --- & 1 & 10 & 100\\
\midrule
$L_2$ & 0.516 & 0.194 & 0.368 & 0.566\\
\bottomrule
\end{tabular}
\end{table}

Confounded Color-MNIST (\cref{app:cmnist}) provides an injected $P(U)$ that
is known by construction and can serve as a diagnostic reference.
\Cref{tab:confounder-l2} reports the average $L_2$ distance between the
learned $P(C)$ and this reference, after matching cluster indices to
confounder states. Across $\gamma\in\{1,10,100\}$, the distances are
$0.194$, $0.368$, and $0.566$, respectively, compared with $0.516$ for ANCM.
Both \cauvade{} and the ANCM baseline learn confounder distributions that
differ from the injected distribution. This comparison illustrates the
ambiguity associated with non-identifiability, which also affects the
baseline; recovering the injected confounder is not the goal of this
experiment or our method.
Instead, \cauvade{} explores plausible exogenous distributions $P(U)$ compatible with the confounded data and uses them to generate diverse predictions of images under intervention. The varying distances across the $\gamma$-variation reflect this exploration: different values of $\gamma$ select different observationally compatible confounder distributions, not converging to a specific solution.

\section{Common Questions}\label{app:qa}

\begin{enumerate}
\item \textbf{Does varying $\gamma$ compromise observational fit?}\label{app:qa-obs-fit}\\
In \cref{eq:objective}, $\gamma$ weights the additive cluster-posterior entropy
term $H(\phi)$; the reconstruction, KL, and treatment-consistency terms defining
$D$ remain unchanged. Changing $\gamma$ therefore preserves the definition of
data fit, although it can change the fitted solution. Under their local
assumptions, \cref{lem:price,thm:move} show second-order surrogate cost and
first-order interventional movement near $\gamma=0$. With strict slack in the
baseline observational checks, \cref{cor:admissible-move} guarantees that
sufficiently small weights preserve admissibility. Larger weights still
require checking the fitted candidates, as prescribed in \cref{sec:elbo-aug}.

The Color-MNIST and CelebA experiments provide strong empirical evidence for observational
admissibility at the measured image and semantic levels.
\Cref{tab:recon_gamma,tab:recon_gamma_celeba} report Label RMSE between the
model-implied observational joint distribution of $(X,Y)$ and the empirical
joint distribution estimated from labeled data. It remains small across the
evaluated weights: $0.0067$--$0.0519$ on Color-MNIST and
$0.0055$--$0.0387$ on CelebA. This directly assesses agreement with labeled
$P(X,Y)$, complementing the stable reconstruction MSE and NLL. Together,
these diagnostics support observational fit as interventional predictions
vary, without establishing equality of the full observational distributions
or guaranteeing admissibility for arbitrary $\gamma$.

For MIMIC-CXR-JPG, we compare observational fit with the $\gamma=0$ baseline
alongside the change in FID against the unconfounded reference. This checks
whether gains in interventional image-distribution agreement come with a loss
of observational fit; FID against that reference evaluates the interventional
samples and does not itself measure observational admissibility.

\item \textbf{What happens when $\gamma=0$?}\label{app:qa-gamma-zero}\\
The entropy term drops out, and \cauvade{} optimizes only $D$
(\cref{eq:objective}). In the ablation of \cref{sec:ablation-gamma-zero},
its interventional estimates concentrate tightly across random seeds,
exhibiting behavior similar to ANCM (\cref{fig:_vade_gamma_zero,tab:minmax}).
These estimates lie within the scalar Manski bounds but explore little of the
variation compatible with the observations. Sweeping $\gamma$ supplies an
additional means of obtaining diverse candidate interventional predictions;
it does not guarantee exhaustive coverage or sharp bounds.

\item \textbf{How should $K$ be chosen to obtain diverse hypotheses?}\label{app:qa-K-selection}\\
$K$ controls the capacity of the latent confounder mixture. Increasing it can
support more diverse causal hypotheses: in \cref{tab:K-sensitivity}, increasing
$K$ from $2$ to $4$ expands the sampled interventional range from $0.167$ to
$0.352$ for $do(X{=}0)$ and from $0.171$ to $0.283$ for $do(X{=}1)$.
Realized diversity need not increase monotonically with $K$, however; the
reported ranges narrow again at $K=6$.

Our goal is to generate diverse hypotheses that fit the observations, without
requiring an exhaustive characterization of the compatible set. Thus, $K$
need not recover the true confounder cardinality or be large enough to
represent every compatible model. Choosing fewer clusters is acceptable when
the fitted family remains sufficiently diverse and passes observational
checks. Too little capacity can still restrict fit and diversity: at $K=2$,
only $14/28$ estimates per intervention pass the scalar bounds check
(\cref{tab:K-sensitivity}). In practice, compare a small set of cluster counts
using observational fit and variation in the resulting interventional
predictions; increase $K$ if capacity limits either. We use $K=8$ in the main
experiments (\cref{tab:hyperparameters}).

\item \textbf{Are $U$ and $P(U)$ identifiable from observational data?}
\label{app:qa-nonidentifiability}\\
In general, neither the latent confounder representation $U$ nor its
distribution $P(U)$ is identifiable from observational data under our
assumptions. We explicitly acknowledge this structural uncertainty; it is the starting
point of our work. Inspired by partial identification, \cauvade{} explores
plausible confounder distributions within models that pass observational
checks and uses them to generate diverse interventional image predictions.
Our objective is not to recover the true latent representation $U$ or exogenous distribution $P(U)$.
The synthetic experiment in \cref{app:confounder-recovery} illustrates this
non-identifiability: even though the injected confounder distribution is known
by construction, it is not identified by the observational data, and both
\cauvade{} and the ANCM baseline learn distributions that differ from it.
The comparison highlights this shared non-identifiable nature; recovering the ground-truth confounder has never been the goal of our method.
Instead, our method explores plausible exogenous distributions $P(U)$ compatible with the confounded data and uses them to generate diverse predictions of images under intervention.

\end{enumerate}

\end{document}